\documentclass{article} % For LaTeX2e
\usepackage{iclr2027_conference,times}
\usepackage{asymptote}
\usepackage{booktabs}
\usepackage{colortbl}
\definecolor{lightblue}{RGB}{230,245,255}
\usepackage{dsfont}
\usepackage{fontawesome}
\usepackage{wrapfig}
\usepackage{algorithm}
\usepackage{multirow}

\usepackage{graphicx}
\usepackage{booktabs}
\usepackage{caption}
\usepackage{adjustbox}
\usepackage{pifont}
\usepackage{epstopdf}
\usepackage{hyperref}
\hypersetup{  
    colorlinks=true,  
    citecolor=cyan,
}

\usepackage{algorithm}
\usepackage{algorithmicx}
\usepackage{algpseudocode}  % 更现代的伪代码环境，支持 For/If 等

\usepackage[most]{tcolorbox} % 导言区引入
\usepackage{tcolorbox} % 导言区引入

\usepackage[a4paper,margin=1in]{geometry}

\usepackage{amsmath,amssymb,amsthm,bm}
\usepackage{mathtools}          % \DeclarePairedDelimiter, split, etc.
\allowdisplaybreaks             % let long display blocks break across pages

\usepackage{booktabs}
\usepackage{tabularx}
\usepackage{microtype}          % subtle spacing; helps avoid overfull lines
\usepackage{enumitem}
\newtheorem{assumption}{Assumption}[section]
\newtheorem{definition}{Definition}[section]
\newtheorem{lemma}{Lemma}[section]
\newtheorem{proposition}{Proposition}[section]
\newtheorem{theorem}{Theorem}[section]
\newtheorem{corollary}{Corollary}[section]
\newtheorem{remark}{Remark}[section]

\DeclareMathOperator{\KL}{KL}
\DeclareMathOperator{\SNR}{SNR}
\DeclareMathOperator{\CV}{CV}
\DeclareMathOperator{\tr}{tr}
\DeclareMathOperator{\Var}{Var}

\DeclareMathOperator{\softmax}{softmax}
\DeclareMathOperator{\diag}{diag}
\DeclareMathOperator*{\argmin}{arg\,min}
\newcommand{\R}{\mathbb{R}}
\newcommand{\E}{\mathbb{E}}
\newcommand{\bh}{\mathbf{h}}
\newcommand{\bz}{\mathbf{z}}
\newcommand{\bt}{\mathbf{t}}
\newcommand{\bs}{\mathbf{s}}
\newcommand{\bu}{\mathbf{u}}
\newcommand{\bv}{\mathbf{v}}
\newcommand{\bw}{\mathbf{w}}
\newcommand{\br}{\mathbf{r}}
\newcommand{\bp}{\mathbf{p}}
\newcommand{\bq}{\mathbf{q}}
\newcommand{\bc}{\mathbf{c}}

\newcommand{\bdelta}{\bm{\delta}}
\newcommand{\bDelta}{\bm{\Delta}}
\newcommand{\bmu}{\bm{\mu}}
\newcommand{\bxi}{\bm{\xi}}
\newcommand{\Dh}{\bDelta_{\bh}}      % first-order activation shift from a weight update
\newcommand{\PW}{P_W}                 % projector onto the weight principals
\newcommand{\Gnorm}[1]{\lVert #1\rVert_G}

\usepackage{amsmath,amsfonts,bm}

\def\eqref#1{equation~\ref{#1}}
\def\1{\bm{1}}

\DeclareMathAlphabet{\mathsfit}{\encodingdefault}{\sfdefault}{m}{sl}
\SetMathAlphabet{\mathsfit}{bold}{\encodingdefault}{\sfdefault}{bx}{n}

\providecommand{\E}{\mathbb{E}}

\providecommand{\R}{\mathbb{R}}

\providecommand{\softmax}{\mathrm{softmax}}

\providecommand{\KL}{D_{\mathrm{KL}}}
\providecommand{\Var}{\mathrm{Var}}

\ifdefined\argmin\else
\DeclareMathOperator*{\argmin}{arg\,min}
\fi

\tcbset{
    takeaway/.style={
        colback=blue!5,      % 背景色（浅蓝）
        colframe=black,      % 边框颜色
        sharp corners,       % 直角边框（或用rounded corners圆角）
        boxrule=0.5pt,       % 边框粗细
        coltitle=white,      % 标题字体颜色
        fonttitle=\bfseries, % 标题字体加粗
        colbacktitle=black,  % 标题背景色
        enhanced,            % 开启高级功能
        attach boxed title to top left={yshift=-2mm, xshift=2mm}, % 标题位置
        boxed title style={
            sharp corners,
            boxrule=0pt,
            colframe=black,
        },
    }
}

\usepackage{url}

\title{Learning to Steer, Steering to See: Unveiling the Geometry of RLVR in Large Language Models via Trainable Vectors}

\author{
\textbf{Yuchen Cai}$^{1,2}$\thanks{This work was done during an internship at Tencent.},
\textbf{Ding Cao}$^{1}$,
\textbf{Qixiang Yin}$^{3}$,
\textbf{Xin Xu}$^{2}$,
\textbf{Kai Yang}$^{2}$,
\textbf{Siye Wu}$^{2}$,
\textbf{Pengyuan Wang}$^{2}$, \\
\textbf{Jiaxuan Wang}$^{2}$,
\textbf{Weijie Liu}$^{2}$,
\textbf{Saiyong Yang}$^{2}$,
\textbf{Guangzhong Sun}$^{1}$,
\textbf{Guiquan Liu}$^{1}$\thanks{Corresponding author: \texttt{gqliu@ustc.edu.cn, fjf@mail.ustc.edu.cn}},
\textbf{Junfeng Fang}$^{1}$\footnotemark[3]
\\
$^{1}$USTC, $^{2}$Tencent Hunyuan, $^{3}$BUPT \\
\texttt{\{caiyuchen,caoding\}@mail.ustc.edu.cn}
}

\iclrfinalcopy % Uncomment for camera-ready version, but NOT for submission.
\begin{document}

\maketitle

\begin{abstract}
Reinforcement learning (RL) has emerged as a key paradigm for enhancing the
reasoning capabilities of large language models. However, the high
dimensionality of parameter updates makes RL training dynamics difficult to
analyze, obscuring the mechanisms underlying these gains. In this work, we study reinforcement learning with verifiable rewards (RLVR) and use vector steering as an analytical tool to identify a low-dimensional effective manifold in activation space that is closely associated with RL-induced performance gains. We further uncover two key geometric properties of this
manifold. \textbf{(1) Effective Manifold Capacity:} The capacity required to reproduce RL-induced gains can be very small, but it is not infinitely compressible. When the capacity is compressed to an extremely low level, the intervention dimensionality and the ability to express input-dependent corrections become
key constraints, and this requirement further varies with injection depth.
\textbf{(2) Control Manifold Separation:} Effective control directions lie
predominantly in the low-variance complement of the activation principal
subspace. Within the same task and base model, the learned geometry remains
largely consistent across training configurations; across tasks, geometric
alignment correlates with capability transfer. We conduct experiments on 5 LLMs and 6 tasks with verifiable rewards, and the results support the above properties. Based on these findings, we propose \textbf{Alpha-Stabler}, a plug-and-play training framework
comprising two modules: a \textbf{Predictor} that monitors principal-subspace
intrusion to provide early warnings of training collapse, and a
\textbf{Controller} that removes the principal-subspace component of activation gradients during backpropagation while preserving their components in the orthogonal complement of the principal subspace.
Experiments show that Alpha-Stabler can stabilize training for 2,000 steps and consistently enhance the gains brought by reinforcement learning. This work advances the understanding of RL from an activation-manifold perspective and offers practical insights for more robust post-training. Our code is available at: \href{https://github.com/caiyuchen-ustc/On_Policy_Vector_Training}{https://github.com/caiyuchen-ustc/On\_Policy\_Vector\_Training}.
\end{abstract}

\begin{flushright}
\small\itshape
``Mille viae ducunt homines per saecula Romam.''\\
--- Alain de Lille
\end{flushright}

\section{Introduction}
Reinforcement learning (RL) has become central to post-training for large
language models (LLMs), yielding substantial gains in reasoning performance
\citep{openai2025, xiao2026mimo, xu2026deepseek}. Yet how RL produces these
improvements internally remains incompletely understood. High-dimensional
parameter updates, optimization noise, and sensitivity to training
configurations make it difficult to distinguish behaviorally meaningful changes
from incidental variation. Prior work has examined sampling efficiency
\citep{yue2025doesreinforcementlearningreally}, entropy dynamics
\citep{cui2025entropymechanismreinforcementlearning}, scaling behavior
\citep{tan2026scalingbehaviorsllmreinforcement}, and update geometry
\citep{cai2025predictability, cai2026learning, zhu2025path}. A complementary
question remains unanswered: what structure do RL-induced behavioral changes
exhibit in activation space, and how does this structure relate to capability
transfer and training stability?

We view RL post-training as an intervention problem in activation space. Concretely, we freeze the base model and learn input-invariant shared vectors at selected layers to distill its RL-tuned counterpart, thereby testing whether RL-induced gains can be recovered by low-dimensional activation interventions. This directly probes whether RL gains admit a low-dimensional description and when this description breaks down. We thereby identify a low-dimensional effective manifold in activation space tied to RL-induced gains, and characterize two geometric properties of this manifold: \textbf{Effective Manifold Capacity} and \textbf{Control Manifold Separation}.

Our first finding concerns the effective capacity of this manifold. With an RL-trained model as the teacher, a single input-invariant vector per controlled layer recovers over 85\% of the gain achieved by full fine-tuning on most evaluated tasks, as shown in Fig.~\ref{fig1} (a) left. However, this compressibility has clear limits. The first arises from the teacher--student distribution gap: when the two distributions are close, a single-vector intervention suffices, whereas a widening gap requires more flexible parameterizations. The second arises from insertion depth: fixed-vector interventions work well at low-to-middle layers but degrade near the output, as shown in Fig.~\ref{fig1}(a) right. Through nonlinear analysis, we further find that this degradation does not stem from weak optimization signals at deeper layers, but from the expressiveness bottleneck of low-dimensional interventions at high layers. We summarize these findings as \textbf{Property 1 (Effective Manifold Capacity)}: the capacity required to reproduce RL-induced gains can be very small, but it is not infinitely compressible; this capacity depends on how well the intervention form matches the required correction.

\begin{figure}[t]
    \includegraphics[width=1\textwidth]{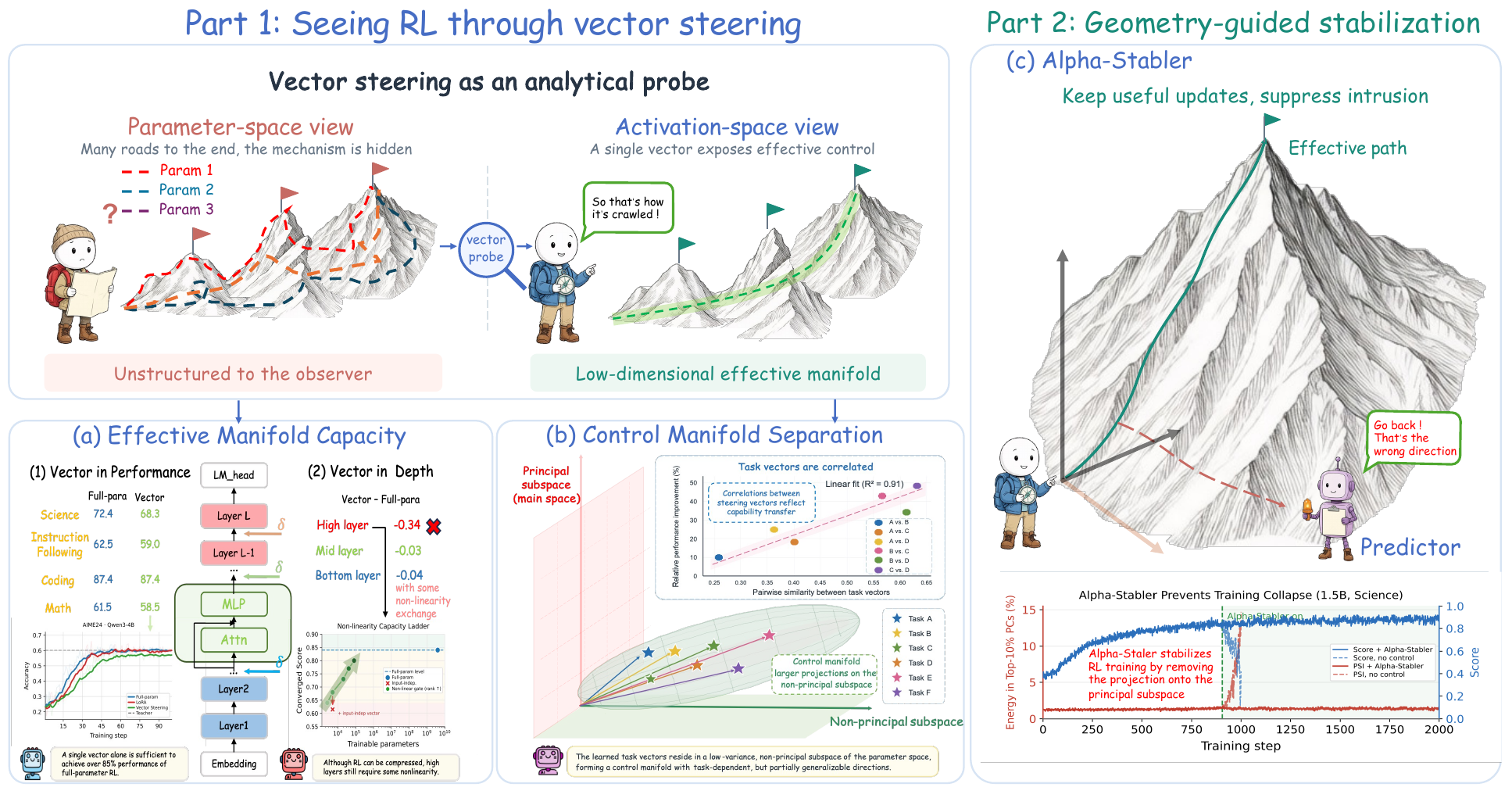}
    \caption{The single steering vector reveals that RLVR-induced changes form a low-dimensional, depth-sensitive control geometry, characterized by \textbf{(a)} effective manifold capacity and \textbf{(b)} control-manifold separation. Building on these findings, \textbf{(c)} Alpha-Stabler stabilizes training via principal-subspace intrusion monitoring and backward gradient projection.}
    \label{fig1}
    \vspace{-0.2cm}
\end{figure}

We next examine the geometry of the effective manifold. By extending the single-vector probe to multiple mutually orthogonal vectors per layer, we find that although the standalone performance of later directions declines with extraction order, they can still independently support the target behavior, indicating that the leading direction is dominant but not uniquely effective. While RL's training parameters differ substantially across configurations in parameter space, they consistently converge, after compression, to a set of reproducible task-related directions. Across tasks, the alignment between direction vectors strongly reflects transfer gains, and these directions consistently concentrate in the low-variance complement of the activation principal subspace, as shown in Fig.~\ref{fig1} (b). We characterize this organization as \textbf{Property 2 (Control Manifold Separation)}: RL-induced control converges to a set of reproducible task-related directions that are largely separated from the variance-dominant representation subspace.

These two properties indicate that RL-induced behavioral changes admit a structured, low-dimensional description in activation space. During successful RL training, activation shifts remain predominantly in the low-variance complement of the principal subspace, whereas principal-subspace intrusion accompanies performance degradation and collapse. To this end, we propose \textbf{Alpha-Stabler}, a lightweight, plug-and-play framework for stabilizing RL training, as shown in Fig.~\ref{fig1} (c). It introduces no additional trainable parameters and requires no modification to the optimizer, reward function, or model architecture. Alpha-Stabler consists of a \textbf{Predictor} that monitors principal-subspace intrusion (PSI) against a fixed principal subspace from the frozen base model and warns before collapse, and a \textbf{Controller} that removes the principal-subspace component of activation gradients during backpropagation while preserving their orthogonal complements. Experiments show that, with no significant additional computation and zero additional trainable parameters, Alpha-Stabler achieves more stable reward growth over thousands of training steps than gradient clipping and other gradient projection methods.

In summary, this work identifies two geometric properties of RLVR-induced changes: \textbf{Effective Manifold Capacity} and \textbf{Control Manifold Separation}. These properties reveal that the seemingly complex dynamics of RL post-training are governed by a surprisingly simple structure: a low-dimensional manifold confined to the low-variance complement of the activation principal subspace. Building on these properties, we introduce \textbf{Alpha-Stabler}, a plug-and-play framework that operates with negligible overhead and provides a practical tool for stabilizing large-scale RL training.

\begin{figure}[t]
    \includegraphics[width=1\textwidth]{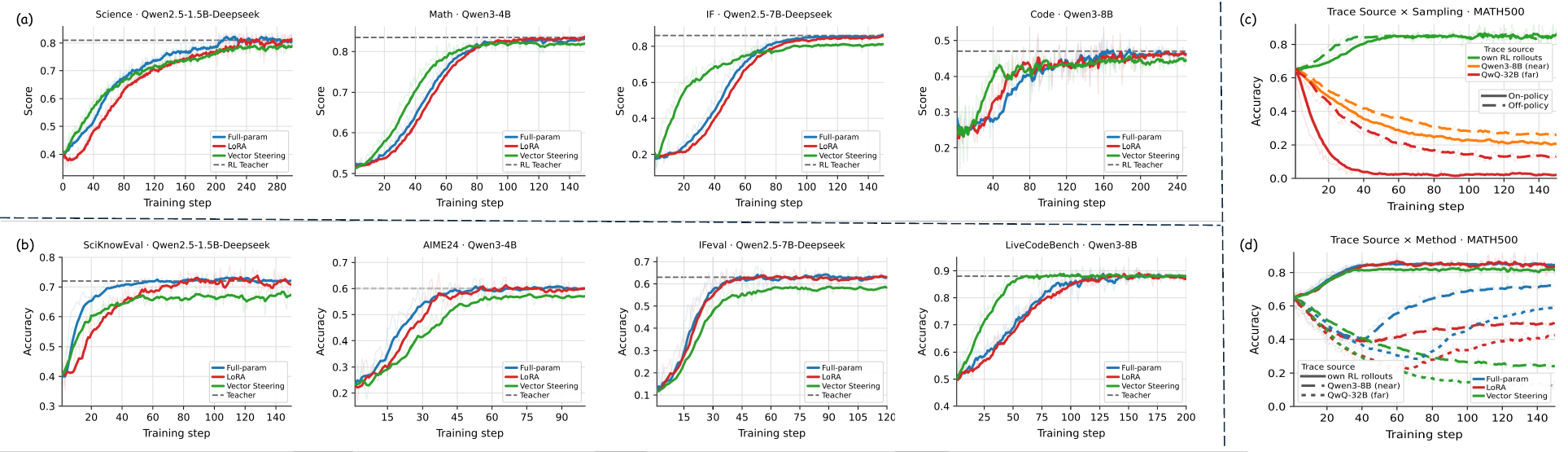}
    \caption{\textbf{(a)} On-policy distillation scores on the training set. \textbf{(b)} Off-policy distillation scores on the evaluation set. \textbf{(c)} Scores of vector steering under varying teacher types. \textbf{(d)} Scores of full-parameter training, LoRA, and vector
steering under off-policy distillation at three distribution gaps.}
    \label{fig2}
\end{figure}

\section{Effective Manifold Capacity}

\subsection{Experimental Setting}
\label{Experimental Setting}

\textbf{Models and Datasets.}
We evaluate DeepSeek-R1-Distill-Qwen-1.5B/7B \citep{guo2025deepseek} and Qwen3-4B/8B/14B \citep{qwen3} on four task families: mathematical reasoning, scientific reasoning, code generation, and instruction following. For each task, we train a teacher with GRPO \citep{guo2025deepseek} or DAPO \citep{yu2025dapoopensourcellmreinforcement} and perform on- or off-policy distillation \citep{li2026rethinking}, comparing three update parameterizations: full-parameter training, LoRA \citep{hu2021loralowrankadaptationlarge}, and vector steering. Details are provided in Appendix~\ref{Experimental Setup}.

\textbf{Vector Steering.}
Consider a frozen base model with $L$ decoder layers and hidden dimension $d$. For a set of controlled layers $\mathcal{S}\subseteq\{1,\ldots,L\}$, we add a trainable vector to each layer's output:
\begin{equation}
\mathbf{h}_\ell' = \mathbf{h}_\ell + \boldsymbol{\delta}_\ell, \qquad \ell\in\mathcal{S},
\end{equation}
where $\mathbf{h}_\ell\in\mathbb{R}^{d}$ is the hidden state at a given token position, and $\boldsymbol{\delta}_\ell\in\mathbb{R}^{d}$ is a layer-specific vector shared across all inputs and token positions. Only $\{\boldsymbol{\delta}_\ell\}_{\ell\in\mathcal{S}}$ are optimized under the corresponding distillation objective. We then examine the compressibility of RL updates in Section~\ref{Compressibility of RL Learning Signals} and further analyze its properties in Section~\ref{Layer-wise Compressibility of RL Updates}.

\subsection{Compressibility of RL Learning Signals}
\label{Compressibility of RL Learning Signals}

\textbf{Compressibility across Distillation Regimes.}
We examine how much of the capability gain obtained by RL can be retained after distillation under restricted update parameterizations. As shown in Fig.~\ref{fig2} (a), under on-policy distillation, LoRA and full-parameter training perform comparably, while a single steering vector per controlled layer still recovers over $85\%$ of the RL gain. To test whether this compressibility depends on student-generated trajectories \citep{YanHeLiuJin26}, we also evaluate off-policy distillation with teacher-generated trajectories under the same update parameterizations. As shown in Fig.~\ref{fig2} (b), steering retains a comparably high recovery rate on evaluation datasets. Together, these results indicate that RL-induced gains can be highly compressed, and that fixed layer-wise activation offsets alone suffice to retain most of them.

\textbf{Effect of Teacher Type and Update Capacity.}
These recovery rates, however, have limits. We fix the student and vary the teacher under both distillation regimes to examine the sensitivity of vector steering to the teacher--student distribution gap. As shown in Fig.~\ref{fig2} (c), vector steering effectively transfers the capabilities of RL-trained teachers, whereas steering performance degrades substantially with the tested non-RL teachers, with some settings collapsing as the gap widens. This indicates that teacher--student distribution matching is a key factor governing the effectiveness of fixed activation offsets.

A natural explanation is that the required corrections become more input-dependent as the gap widens, exceeding what an input-invariant fixed offset can express. If so, parameterizations with different expressive capacities should exhibit a consistent ordering. As shown in Fig.~\ref{fig2} (d), under off-policy distillation, the three parameterizations perform comparably when the gap is small, whereas a larger gap yields a clear ordering: full-parameter training performs best, followed by LoRA, and vector steering worst. This ordering matches their expressive capacities: full-parameter training allows model-wide adjustments, LoRA retains input-dependent corrections under low-rank constraints, and vector steering only applies fixed offsets. Therefore, a larger gap raises the input dependence required of the correction, and fixed additive corrections are the first to reach their capacity limit \citep{cai2026learning}.

\textbf{Implications for RL Compressibility.}
Taken together, these results show that, although vector steering cannot capture arbitrary teacher-induced corrections, it serves as a targeted probe for analyzing the mechanisms underlying RL-specific capability gains.
Building on this, we next examine how RL-induced gains vary with
injection depth and how effective the low-dimensional representation is
across layers.

\begin{figure}[t]
    \centering
    \includegraphics[width=\textwidth]{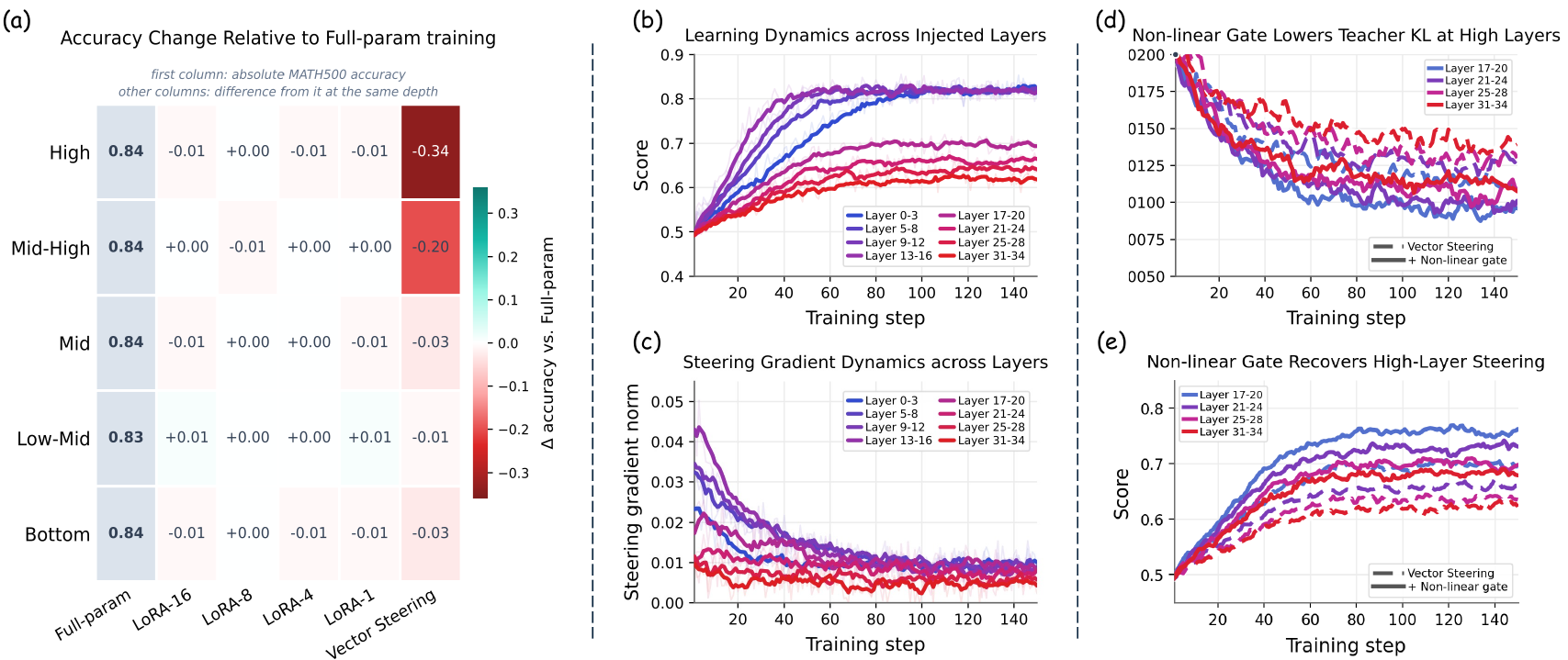}
    \caption{
    Layer-wise performance.
    \textbf{(a)} Converged performance of steering vectors across injection layers.
    \textbf{(b)} Task performance during training for steering vectors.
    \textbf{(c)} Steering gradient norms at different layers.
    \textbf{(d)} KL divergence with the teacher under fixed versus gated steering.
    \textbf{(e)} Converged performance of fixed versus gated injection.
    }
    \label{fig3}
\end{figure}

\subsection{Layer-wise Compressibility of RL Updates}
\label{Layer-wise Compressibility of RL Updates}

\textbf{Injection Depth and Optimization Dynamics.}
Section~\ref{Compressibility of RL Learning Signals} shows that RL-induced gains can be compressed into fixed activation offsets. This raises a complementary question: where should such a low-capacity intervention be applied? We therefore compare full-parameter updating, LoRA, and fixed-vector steering at representative network depths, restricting each intervention to the selected layer and freezing the remaining backbone. As shown in Fig.~\ref{fig3} (a), selected-layer full-parameter updating and LoRA remain effective across depths, whereas fixed-vector steering performs well at lower and intermediate layers but deteriorates near the output. This contrast indicates that low-capacity interventions depend on both the update parameterization and the injection depth.

We next examine the learning curves across injection sites. Fig.~\ref{fig3} (b) shows that shallow-layer steering gradually approaches intermediate-layer performance during training, whereas high-layer steering plateaus early, and extended training does not noticeably close the gap. To test whether this stagnation stems from weaker gradient signals, we inspect gradient norms. Fig.~\ref{fig3} (c) further shows that shallow and high-layer injections receive gradients of comparable magnitude yet yield substantially different final performance, indicating that gradient magnitude alone does not explain high-layer degradation. One explanation is that high-layer injections have fewer downstream nonlinear transformations, which may limit the ability of a fixed offset to produce context-dependent effects, whereas LoRA and selected-layer full-parameter updates can directly produce input-dependent activation corrections.

\textbf{Input-Dependent Corrections and Capacity Controls.}
To test whether explicit token-dependent modulation alleviates high-layer degradation, we replace the fixed offset with a low-rank nonlinear correction:
\begin{equation}
\mathbf{h}_\ell' = \mathbf{h}_\ell + \boldsymbol{\delta}_\ell(\mathbf{h}_\ell), \qquad
\boldsymbol{\delta}_\ell(\mathbf{h}_\ell)
= \mathbf{B}_\ell\sigma\!\left(\mathbf{A}_\ell\mathbf{h}_\ell\right),
\end{equation}
where $\mathbf{A}_\ell\in\mathbb{R}^{r\times d}$,
$\mathbf{B}_\ell\in\mathbb{R}^{d\times r}$, $r\ll d$, and $\sigma$ is an element-wise gating nonlinearity. Here $\mathbf{B}_\ell$ defines shared correction directions, while $\sigma(\mathbf{A}_\ell\mathbf{h}_\ell)$ assigns token-dependent coefficients. We zero-initialize $\mathbf{B}_\ell$ so that the intervention initially leaves the model unchanged. Fig.~\ref{fig3} (d,e) shows that Rank-1 gating reduces teacher KL divergence and improves converged performance at all four tested high-layer injection sites. At Rank-1, the correction still follows a single shared direction per layer, but its strength can vary across tokens. To separate this flexibility from parameter count, we compare against a parameter-matched input-independent control that replaces $\sigma(\mathbf{A}_\ell\mathbf{h}_\ell)$ with a constant coefficient. Fig.~\ref{fig4} (a) shows that the static control does not reproduce the improvement of Rank-1 gating, supporting the role of token-dependent modulation. Increasing the gating rank yields further gains toward the full-parameter reference at $r=16$. Together, these results suggest that high-layer recovery depends on both the available correction directions and their input-dependent modulation.

\begin{figure}[t]
\includegraphics[width=\textwidth]{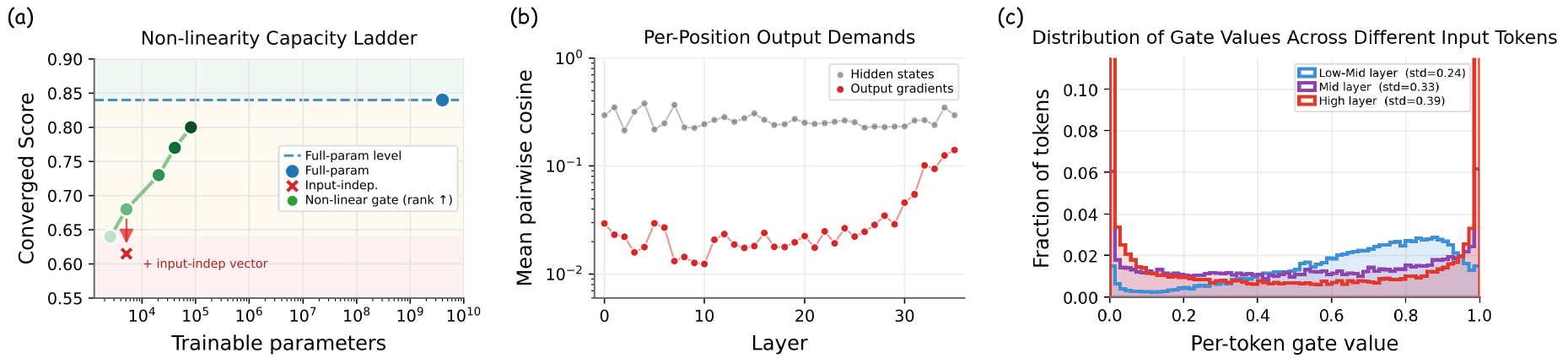}
\caption{Diagnosing and mitigating high-layer degradation.
\textbf{(a)} Converged performance of vector steering under different nonlinear parameters.
\textbf{(b)} Layer-wise mean pairwise cosine similarity for hidden states and output gradients.
\textbf{(c)} Distribution of gating activations at different layers.}
\label{fig4}
\end{figure}

\textbf{Directional Structure and Effective Manifold Capacity.}
We further examine how gradient alignment and learned modulation vary across depth. At each layer, we compute the mean pairwise cosine similarity between per-token gradient contributions to the steering vector,
$\mathrm{Sim}_{\ell}^{\mathrm{grad}} = \frac{1}{|\mathcal{P}|} \sum_{(i,j)\in\mathcal{P}} \cos(\mathbf{g}_{\ell,i},\mathbf{g}_{\ell,j})$,
where $\mathbf{g}_{\ell,i}$ denotes the normalized contribution of token $i$ to the gradient of the loss with respect to the steering vector at layer $\ell$, and $\mathcal{P}$ is the set of token pairs. We compute the same statistic for hidden states as a reference. Fig.~\ref{fig4} (b) shows that mean gradient similarity rises from about $0.019$ at lower and intermediate layers to about $0.14$ near the output. If directional alignment alone determined steering effectiveness, this trend would favor shared-direction interventions at higher layers. Yet fixed-vector steering degrades precisely there, showing that average directional alignment alone does not explain the depth-dependent performance gap. We further analyze the distribution of the Rank-1 gating coefficient $\sigma(\mathbf{A}_\ell\mathbf{h}_\ell)$ in Fig.~\ref{fig4} (c). At lower and intermediate layers, coefficients spread broadly over $[0,1]$, consistent with graded modulation; at higher layers, they concentrate near $0$ and $1$, consistent with more selective use of the shared direction. Together with the gated-versus-static comparison in Fig.~\ref{fig4} (a), this indicates that explicit token-dependent modulation improves the effectiveness of high-layer interventions: when the intervention can adjust its strength according to each token's hidden state, the limited downstream nonlinear transformations at high layers no longer constitute a major bottleneck.

\textbf{Summary.}
We summarize these findings as \textbf{Property 1 (Effective Manifold Capacity)}: reproducing RL-induced gains can require only a very low-dimensional intervention, but its effective capacity depends on how well the intervention form matches the required correction. Appendix~\ref{sec:recovery} provides a conditional account of fixed-vector recovery in terms of the balance between shared teacher-induced shifts and the context-dependent fluctuations around them.

\section{Control Manifold Separation}
\label{Control Manifold Separation}

Having established strong recovery with one steering vector per controlled layer, we now examine the diversity, reproducibility, and principal-subspace localization of effective steering directions.

\subsection{Probing Residual Effective Directions with Multi-Vector Steering}
\label{sec:multi_vector}
\paragraph{Sequential orthogonal steering.}
To test whether the leading direction is uniquely effective, we
sequentially learn $K$ vectors
$\mathcal V_\ell=\{\mathbf v_\ell^{(k)}\}_{k=0}^{K-1}$ at each controlled
layer $\ell\in\mathcal S$, enforcing pairwise orthogonality against
earlier vectors,
$\langle\mathbf v_\ell^{(k)},\mathbf v_\ell^{(j)}\rangle=0$ for $j<k$.
The first stage is trained to convergence and frozen; each subsequent
stage initializes new vectors in the remaining orthogonal complement,
optimizes them to convergence, and reprojects them after every optimizer
step:
\begin{equation}
\mathbf{v}_{\ell}^{(k)}
\leftarrow
\mathbf{v}_{\ell}^{(k)}
-
\sum_{j=0}^{k-1}
\frac{
\left\langle \mathbf{v}_{\ell}^{(k)}, \mathbf{v}_{\ell}^{(j)} \right\rangle
}{
\left\| \mathbf{v}_{\ell}^{(j)} \right\|_2^2
}
\mathbf{v}_{\ell}^{(j)}.
\end{equation}
At each stage, only the current stage’s vectors are injected and evaluated, jointly defining a multi-layer intervention; earlier vectors remain frozen and serve solely to define the orthogonality constraints.

\begin{figure}[t]
    \centering
    \includegraphics[width=\textwidth]{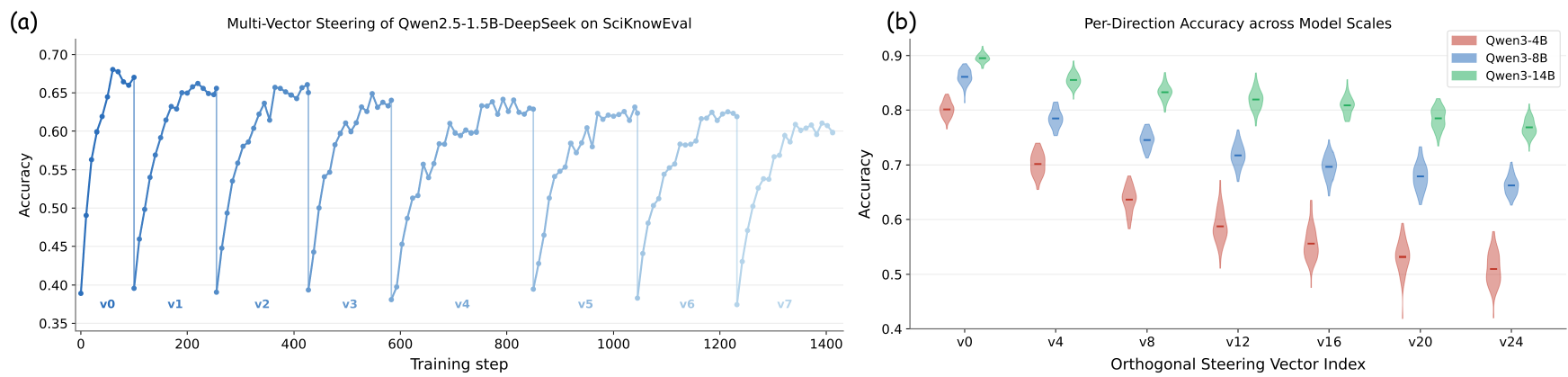}
    \caption{Multi-vector steering with sequential orthogonalization.
    \textbf{(a)} Training trajectories of sequentially learned orthogonal steering vectors on SciKnowEval with DeepSeek-R1-Distill-Qwen-1.5B.
    \textbf{(b)} Standalone accuracy of the extracted directions across model scales.
    Each direction is trained and evaluated independently, with preceding directions used only to define the orthogonality constraint.}
    \label{fig5}
\end{figure}

\paragraph{Residual effectiveness and model scale.}
As shown in Fig.~\ref{fig5} (a), the first direction achieves $67.5\%$
standalone accuracy on SciKnowEval with DeepSeek-R1-Distill-Qwen-1.5B, while the
eighth direction, $\mathbf v^{(7)}$, achieves $62.0\%$, only $5.5$
percentage points lower despite being orthogonal to the preceding seven
directions at every controlled layer.
This modest reduction indicates that excluding previously learned
directions does not exhaust the available steering solutions: the
remaining orthogonal space still admits interventions with substantial
standalone effectiveness.
Training results across all four task domains are shown in
Appendix Fig.~\ref{appendix3}.

We further examine how this phenomenon varies with model scale.
As shown in Fig.~\ref{fig5} (b), the leading direction consistently
performs best, while standalone accuracy generally declines with
extraction order; the decline is sharper in smaller models, whereas
larger models retain higher-order directions whose performance remains
closer to that of the leading direction.
This indicates that model scale affects not only the performance of the
leading direction, but also the extent to which higher-order directions
are retained relative to it.

\textbf{Summary.} These results show that the leading steering direction is effective but not unique: multiple layer-wise orthogonal interventions can independently support the target behavior, although their standalone effectiveness generally declines with extraction order. We next examine whether comparable directions are recovered across teacher parameterizations and distillation regimes.

\subsection{Task-Dependent Geometry and Principal-Subspace Separation}
\label{sec:task_geometry}

\paragraph{Consistency across training configurations.}
We next ask whether the extracted steering directions are reproducible across training configurations. We compare directions along two axes: teacher parameterizations (full-parameter fine-tuning vs.\ LoRA with different ranks) and distillation regimes (on- vs.\ off-policy). For each axis, we extract multi-vector directions using the same procedure and measure their layer-averaged cosine similarity,
$\mathrm{sim}(\{\mathbf{v}_a^{(i)}\},\{\mathbf{v}_b^{(j)}\}) = \frac{1}{|\mathcal{S}|}\sum_{\ell\in\mathcal{S}} \cos(\mathbf{v}_{a,\ell}^{(i)}, \mathbf{v}_{b,\ell}^{(j)})$,
where $\mathcal{S}$ is the controlled layer set and $i,j$ denote extraction indices. Fig.~\ref{fig6} (a) shows stronger same-index than cross-index alignment across the tested teacher parameterizations, while Fig.~\ref{fig6} (b) reports matching-index alignment between on- and off-policy directions. Together, these comparisons indicate a broadly consistent geometric structure across the evaluated configurations and distillation regimes within the same task.

\paragraph{Cross-topic alignment and capability transfer.}
Based on this, we further examine whether this geometric structure also reflects functional relationships across tasks. We partition Science into physics, chemistry, materials, and biology. Keeping the base model fixed, we train one RL teacher per topic and distill it using the same procedure to extract its multi-vector directions. For source topic $s$ and target topic $t$, we define directional alignment as the sum of the first four same-index cosine similarities, averaged across controlled layers:
$$
A(s,t)
=
\frac{1}{|\mathcal{S}|}
\sum_{\ell\in\mathcal{S}}\sum_{i=0}^{3}
\cos\!\left(\mathbf{v}_{s,\ell}^{(i)},\mathbf{v}_{t,\ell}^{(i)}\right).
$$
We normalize the source-to-target transfer gain by the gain obtained from training directly on the target topic:
\begin{equation}
\Delta_{\mathrm{rel}}\operatorname{Acc}(s\rightarrow t)
=
\frac{\operatorname{Acc}(M_s,t)-\operatorname{Acc}(M_{\mathrm{base}},t)}
{\operatorname{Acc}(M_t,t)-\operatorname{Acc}(M_{\mathrm{base}},t)}
\times 100\%,
\end{equation}
where $M_s$, $M_t$, and $M_{\mathrm{base}}$ denote the source-trained, target-trained, and base models, respectively. We evaluate six ordered source--target pairs among the four scientific topics, excluding self-pairs, and report the corresponding directional alignments and normalized transfer gains. As shown in Fig.~\ref{fig6} (c), stronger directional alignment is associated with larger normalized transfer gains across these pairs, with a linear fit yielding $R^2=0.91$. This indicates that the geometry of the learned steering directions is associated with capability transfer across the evaluated topic pairs, with stronger similarity corresponding to better transferability across tasks.

\paragraph{Effective steering directions avoid the principal subspace.}
We further examine where these directions are located in activation space. Let $U_{\ell,r}\in\mathbb R^{d\times r}$ contain the leading orthonormal eigenvectors of the frozen base model's token-level activation covariance at layer $\ell$, with $r$ corresponding to $10\%$ of the hidden dimension, and define $P_\ell=U_{\ell,r}U_{\ell,r}^{\top}$. For a layer-wise direction $\mathbf z_\ell$, its principal-subspace energy fraction is
\begin{equation}
e_{\mathrm{PC}}(\mathbf z_\ell)
=
\frac{\|P_\ell\mathbf z_\ell\|_2^2}
     {\|\mathbf z_\ell\|_2^2+\epsilon}.
\end{equation}
For multi-layer directions, we average $e_{\mathrm{PC}}$ across controlled layers. As shown in Fig.~\ref{fig7} (a), across four scientific topics, unconstrained steering places only about $1\%$--$2\%$ of its energy in the Top-$10\%$ principal subspace, well below the $10\%$ isotropic reference, with this pattern persisting across the displayed directions $\mathbf v^{(0)}$ through $\mathbf v^{(8)}$. Low principal-subspace energy is thus a shared geometric property of the extracted interventions.

\begin{figure}[t]
    \centering
    \includegraphics[width=\textwidth]{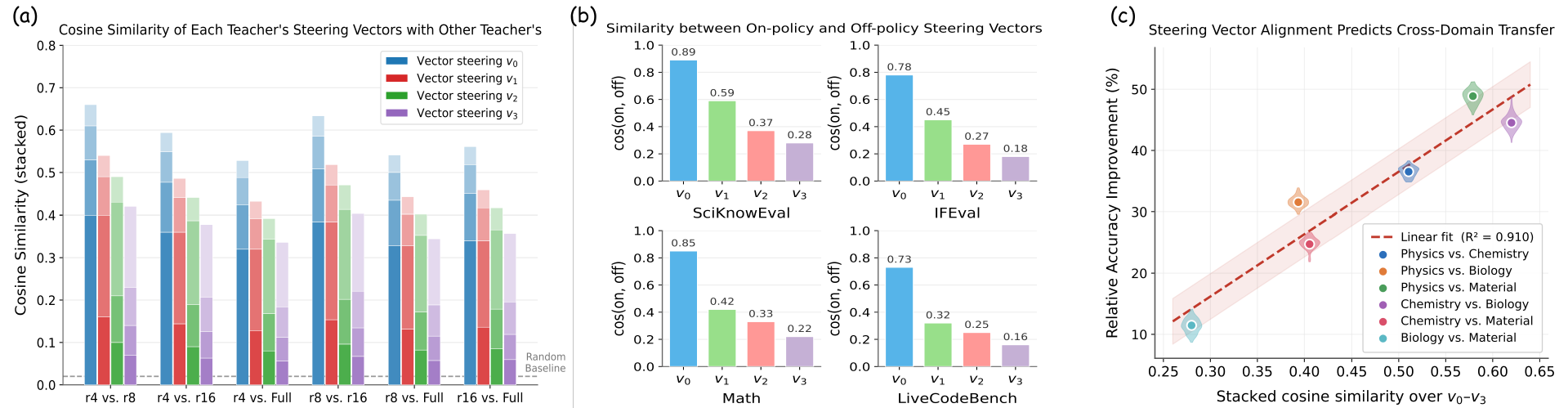}
    \caption{Geometric consistency and transfer of steering directions.
    \textbf{(a)} Layer-averaged cosine similarities between steering vectors under different teacher parameterizations, shown as stacked bars for each reference vector $\mathbf{v}_{a}^{(i)}$, with segments ordered from $\mathbf{v}^{(0)}$ to $\mathbf{v}^{(3)}$ indicating similarities with $\{\mathbf{v}_{b}^{(j)}\}_{j=0}^{3}$.
    \textbf{(b)} Layer-averaged cosine similarities between on- and off-policy steering vectors with matching extraction indices, per task.
    \textbf{(c)} Cross-topic directional alignment versus normalized source-to-target transfer gains, with a fitted regression line.}
    \label{fig6}
\end{figure}

\paragraph{Restricting steering to the principal subspace does not recover the gain.}
To test the functional importance of the low-variance complement, we remove each steering vector's projection onto the bottom-$70\%$ activation-PC subspace after every optimizer step, restricting it to the span of the Top-$30\%$ components: $\mathbf v_\ell^{(k)}\leftarrow P_{\ell,m}\mathbf v_\ell^{(k)}$, where \(m\) is the number of retained components and \(P_{\ell,m}=U_{\ell,m}U_{\ell,m}^{\top}\). As shown in Fig.~\ref{fig7} (b), the Top-$10\%$ energy fraction rises to about $15\%$, but accuracy does not improve meaningfully. The gain therefore does not come from alignment with the principal subspace.

\paragraph{Principal-subspace intrusion tracks training collapse.}
We next examine whether actual RL-induced activation shifts exhibit increased intrusion into the base model's Top-$10\%$ principal subspace as training becomes unstable. At step $t$, we run the current policy and the frozen base model on the same token sequences, defining the activation shift $\boldsymbol\Delta_{\ell,b,i}^{(t)}=\mathbf h_{\ell,b,i}^{(t)}-\mathbf h_{\ell,b,i}^{\mathrm{base}}(t)$ for sample $b$, position $i$, and layer $\ell$. Here \(\mathbf h_{\ell,b,i}^{\mathrm{base}}(t)\) is the frozen base model's activation on the token sequence used at update \(t\); \(t\) indexes the sequence, not a change in base-model parameters. We stack all $N$ valid token-level shifts row-wise into $X_\ell^{(t)}\in\mathbb R^{N\times d}$ to avoid cancellation before projection, and compute
\begin{equation}
\label{eq6}
\operatorname{PSI}_\ell^{(t)}
=
\frac{\|X_\ell^{(t)}U_{\ell,r}\|_F^2}
     {\|X_\ell^{(t)}\|_F^2+\epsilon},
\end{equation}
where \(U_{\ell,r}\) is the fixed Top-\(10\%\) principal-subspace basis of the base model. As shown in Fig.~\ref{fig7}(c), collapse-prone runs show a clear pattern: during stable training, PSI remains low, well below the $10\%$ isotropic reference; as training proceeds, PSI departs from this range and becomes more volatile before any visible score degradation; as collapse progresses, PSI keeps rising while the score falls more sharply.

\textbf{Summary.} These results support \textbf{Property 2 (Control Manifold Separation)}: in the evaluated settings, effective steering directions lie mainly in the low-variance complement of the activation principal subspace, their ordered geometry stays broadly consistent across teacher parameterizations and distillation regimes, and cross-topic alignment correlates with transfer gains. Appendices~\ref{sec:multi} and~\ref{sec:complement} further discuss constrained vector fitting and the link between principal-subspace energy and output sensitivity.

\begin{figure}[t]
    \centering
    \includegraphics[width=\textwidth]{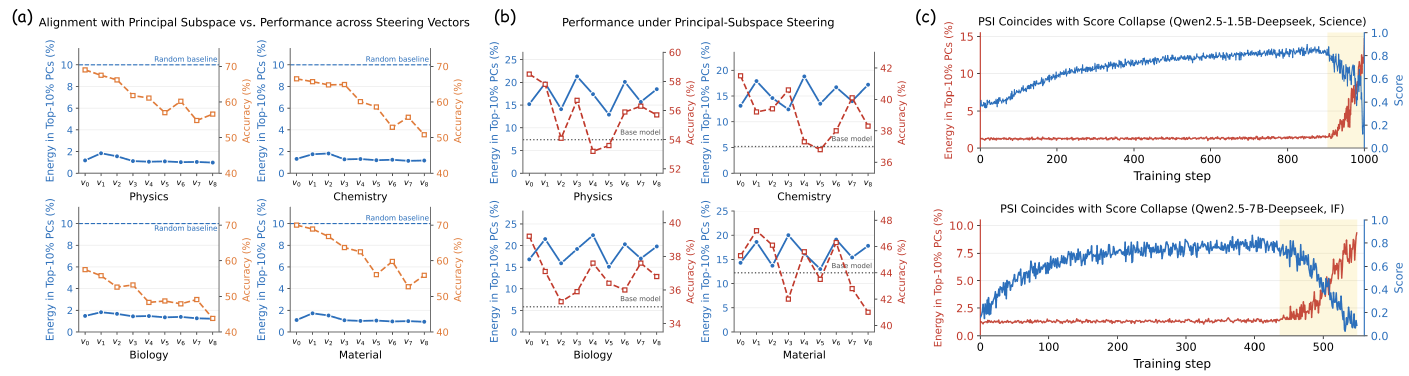}
    \caption{Steering geometry and principal-subspace intrusion. Energy fractions in \textbf{(a)} and \textbf{(b)} use the Top-$10\%$ principal subspace.
    \textbf{(a)} Standalone accuracy and principal-subspace energy fractions of sequentially extracted directions across four scientific topics.
    \textbf{(b)} The same quantities when steering vectors are restricted to the Top-$30\%$ activation-PC subspace.
    \textbf{(c)} Task score and PSI during RL training, comparing stable and collapse-prone runs.}
    \label{fig7}
\end{figure}

\section{Alpha-Stabler: Geometry-Guided RL Stabilization}
Experiments show that effective steering directions lie in the low-variance complement, while principal-subspace intrusion accompanies training collapse. We therefore propose Alpha-Stabler to monitor whether the RL-induced activation shift departs from this regime. It consists of a Predictor, which detects sustained increases in principal-subspace intrusion (PSI), and a Controller, which removes the principal-subspace component of activation gradients during backpropagation while preserving each incoming gradient’s component in the orthogonal complement of the principal subspace.

\subsection{Online Monitoring and Control}
At parameter-update step $t$, we evaluate the actor and the frozen initial base model on identical token sequences. We stack their valid token-level differences, $\Delta_{\ell,b,i}^{(t)} = \mathbf{h}_{\ell,b,i}^{(t)} - \mathbf{h}_{\ell,b,i}^{\mathrm{base}}(t)$, into $X_{\ell}^{(t)} \in \mathbb{R}^{N \times d}$. During the first $T_{\mathrm{warm}} = 50$ updates, actor-generated rollouts are evaluated by the frozen base model to estimate the base model's token-level activation covariance at each monitored layer. At the end of warm-up, we extract its leading orthonormal eigenvectors $U_{\ell,r}$, with $r = \lceil 0.10d \rceil$, and keep these bases fixed thereafter. Using the layer-wise PSI defined in Eqn. \ref{eq6}, we additionally aggregate across monitored layers:
\begin{equation}
\mathrm{PSI}^{(t)} = \frac{\sum_{\ell \in \mathcal{S}}\|X_{\ell}^{(t)}U_{\ell,r}\|_{F}^{2}}{\sum_{\ell \in \mathcal{S}}\|X_{\ell}^{(t)}\|_{F}^{2} + \epsilon}.
\end{equation}
Layer-wise PSI determines where control is activated; the aggregate summarizes intrusion across monitored layers according to their shift energies. Both sum token-level energies before normalization, avoiding cancellation between differently oriented shifts. PSI is computed from detached activations, so it does not affect gradients, while the actor's computation graph is retained for backpropagation.

\paragraph{Predictor: detecting sustained intrusion.}
During an initially stable warm-up, the actor trains without projection, and valid activation-shift matrices are retained every \(M\) updates. Once the principal bases are fixed, we compute PSI for these retained shifts and set fixed warning and lower release thresholds from layer-wise medians and scaled median absolute deviations. Subsequent valid observations update an exponential moving average initialized at the warm-up median. Control activates after \(p\) consecutive valid checks with the smoothed PSI above the warning threshold, and releases when it falls below the lower threshold. Algorithm~\ref{alpha_stabler} specifies the full procedure.

\paragraph{Controller: projecting activation gradients.}
The Controller modifies backward gradients without changing forward activations. Let \(G_{\ell}^{(t)} \in \mathbb{R}^{N \times d}\) denote the incoming gradient at layer \(\ell\), including any downstream control effects. With the Predictor's flag \(a_{\ell}^{(t)} \in \{0, 1\}\) fixed during backpropagation and the optimizer step, we apply
\begin{equation}
\tilde{G}_{\ell}^{(t)} = G_{\ell}^{(t)} - a_{\ell}^{(t)}\left(G_{\ell}^{(t)}U_{\ell,r}\right)U_{\ell,r}^{\top}.
\label{eq:controller-projection}
\end{equation}
When active, this removes the component of each token’s gradient in the principal subspace while preserving the component orthogonal to that subspace. It integrates directly into standard RL training: rollout generation, likelihood ratios, rewards, and the GRPO or DAPO loss remain unchanged, and the optimizer updates parameters using only the modified backward signal. Warm-up requires estimating reference statistics and temporarily storing calibration shifts; after calibration, reference activations are needed only at monitoring checks and can be reused from a compatible pass. Each active hook performs two low-rank matrix multiplications, and inference uses the trained actor alone, so the additional computation is minimal. Algorithm~\ref{alpha_stabler} summarizes the full training loop.

\begin{figure}[t]
    \centering
    \includegraphics[width=\textwidth]{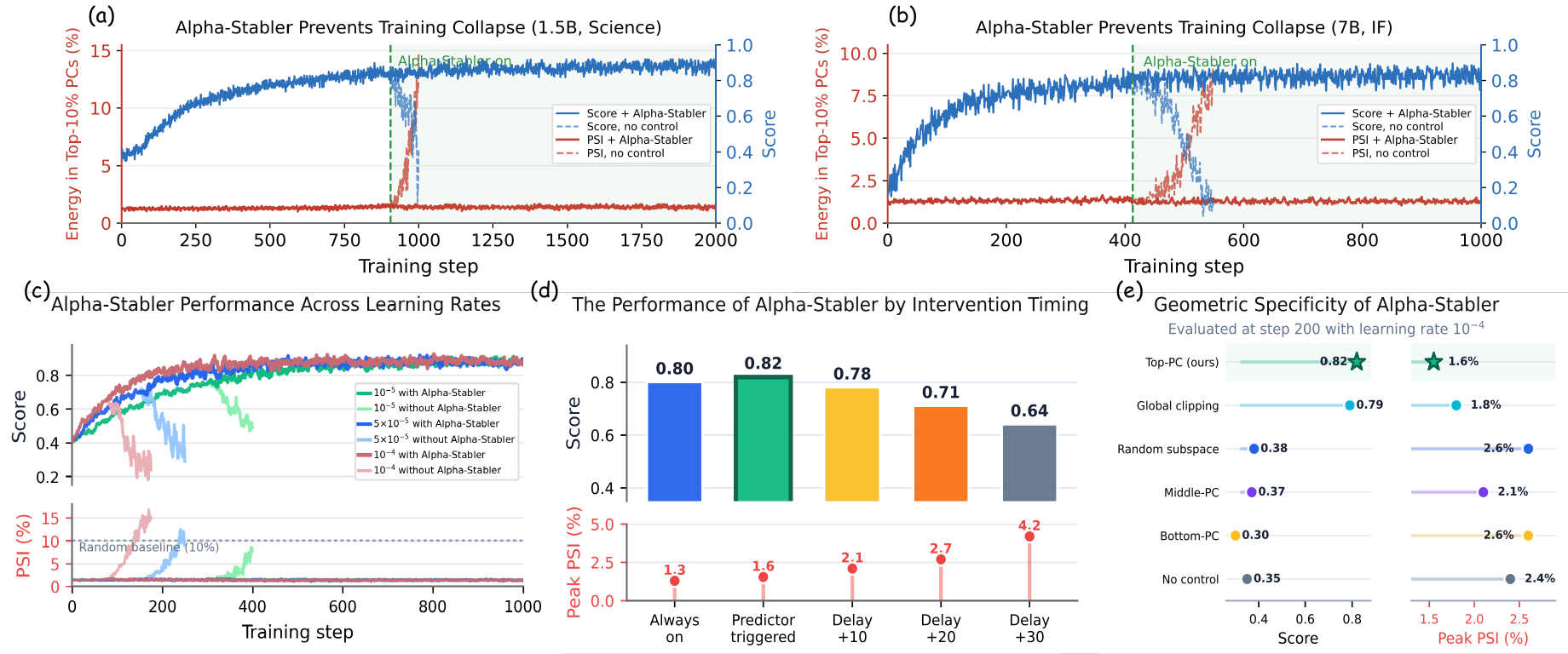}
    \caption{\textbf{(a, b)} Training dynamics with and without Alpha-Stabler.
    \textbf{(c)} Training dynamics across learning rates.
    \textbf{(d)} Task score and peak PSI under different intervention timings.
    \textbf{(e)} Performance under control of different subspaces.}
    \label{fig8}
\end{figure}

\subsection{Experimental Evaluation}

\paragraph{Training stabilization.}
Fig.~\ref{fig8} (a) and (b) compares training trajectories with and without Alpha-Stabler.. In uncontrolled runs, a pronounced increase in PSI coincides with sharp score deterioration, indicating that performance loss accompanies a departure from the earlier activation geometry. Under control, PSI remains near its early-training range while task scores continue to improve, yielding sustained learning alongside limited principal-subspace intrusion rather than reduced PSI at the expense of task performance.

\paragraph{Ablation studies.}
We then conduct ablation studies of Alpha-Stabler along three dimensions: learning rate, intervention timing, and controlled subspace. \textbf{First}, we find that stabilization holds across learning rates, as shown in Fig.~\ref{fig8} (c): at each tested learning rate, uncontrolled training eventually collapses, whereas controlled training keeps peak PSI at a low single-digit percentage and reaches comparable final scores across settings, indicating robustness to the choice of learning rate.  \textbf{Second}, triggered control is preferable to delayed control, as shown in Fig.~\ref{fig8} (d). Triggered control matches always-on control in both score and peak PSI, indicating that continuous intervention is unnecessary, whereas delaying intervention progressively degrades the score and raises peak PSI. \textbf{Third}, Top-PC projection outperforms alternative subspaces, as shown in Fig.~\ref{fig8} (e). Among the compared gradient-control methods, removal of Top-PC components achieves the highest score, while random-subspace, Middle-PC, and Bottom-PC control fail to match it and remain close to the uncontrolled baseline; Top-PC also compares favorably with global gradient clipping. This indicates that the improvement does not arise from an arbitrary gradient constraint, but from removing gradient components in a specific subspace, namely the principal subspace. These results support geometrically targeted control rather than indiscriminate gradient suppression. Appendix~\ref{Geometric Analysis of Alpha-Stabler} shows that gradient projection minimally removes the principal-subspace component of the incoming gradient and bounds principal-subspace activation leakage for a local, single-site projected step. Fig.~\ref{appendix4} further shows nearly identical score trajectories with and without Alpha-Stabler under the same measured wall-clock time, indicating almost no additional time overhead.

\section{Limitations and Future Work}
Despite revealing two geometric properties of RLVR's effect on LLMs, our study has several limitations. First, the theoretical framework rests on empirically motivated assumptions not yet derived from first principles. Second, our analysis and experiments are restricted to RLVR on tasks such as mathematical reasoning and code generation; we do not validate the proposed geometry on RLHF, multi-turn agentic decision-making, or open-ended generation, where whether the same properties hold remains open. Future work can relax these assumptions via neuron attribution and causal tracing, extend the framework to RLHF and agentic settings, and exploit off-principal localization as a broader design principle, e.g., by adding complement-space regularizers or using PSI as an adaptive signal for learning rate and intervention strength, enabling low-rank control without layer monitoring.

\section{Conclusion}
In this work, we identify two geometric properties of RLVR’s effect on LLMs. Effective Manifold Capacity characterizes how a very low-dimensional intervention can reproduce RLVR-induced gains, with the required capacity depending on how well the intervention form matches the required correction. Control Manifold Separation describes the concentration of reproducible, task-related control directions in the low-variance complement of the activation principal subspace. Based on these properties, Alpha-Stabler monitors principal-subspace intrusion as an early-warning signal and projects away the intrusive component of activation gradients during backpropagation, stabilizing training with almost no additional computational overhead. Overall, this work recasts RLVR from black-box parameter optimization into an interpretable geometric process, opening new avenues for theoretically grounded and practically efficient LLM post-training.

\section*{AI use statement}

Generative AI tools were used to assist with language editing, manuscript organization, and formatting, with the aim of improving clarity, readability, and presentation. All AI-assisted revisions were critically reviewed and revised by the authors to ensure technical accuracy, consistency, and fidelity to the underlying research. The authors take full responsibility for the final manuscript, including its methodology, theoretical arguments, experimental results, and conclusions.

\section*{Reproducibility Statement}

We document the experimental procedures and theoretical assumptions underlying our results. Appendix~B specifies the training objectives, model configurations, datasets, task-specific verifiers, evaluation protocols, prompts, optimization hyperparameters, and hardware, together with example commands for vector-steering experiments. Appendix~C provides additional analyses of the learned steering vectors. Appendix~D states the assumptions and proofs supporting our geometric analysis and details the principal-subspace diagnostic and gradient-projection mechanism of Alpha-Stabler. Algorithm~\ref{alpha_stabler} specifies its warm-up calibration, monitoring, and control procedures. The source code is provided in the supplementary materials.

\bibliography{iclr2027_conference}
\bibliographystyle{iclr2027_conference}

\appendix

\clearpage
\section{Related Work}

\textbf{Reinforcement Learning for Large Language Models.}
Prior to reasoning-capable models like OpenAI's o1, RL was primarily used in RLHF for instruction-following and human preference alignment \citep{ouyang2022traininglanguagemodelsfollow}. Recently, Reinforcement Learning with Verifiable Rewards (RLVR) has proven effective for reasoning in verifiable domains such as mathematics and programming \citep{lambert2025tulu3pushingfrontiers}. OpenAI's o1 inspired a wave of open reasoning models, including DeepSeek-V4 \citep{xu2026deepseek}, Kimi-K3 \citep{team2026kimi}, Qwen3.8 \citep{qwen3.8-max}, and GLM-5.3 \citep{glm-5.3}. Beyond single-turn reasoning, RL is increasingly applied to agentic settings involving planning, tool invocation \citep{li2026discovery}, and multi-turn environment interaction \citep{ding2026empowering}, with notable efforts on search-augmented reasoning \citep{jin2025search}, multi-turn credit assignment \citep{zhou2025sweet}, and asynchronous large-scale agentic RL \citep{fu2026areal}.

\textbf{Emergent Behaviors of On-Policy Training.}
Recent studies have investigated intrinsic properties of on-policy RL training. \citet{yue2025doesreinforcementlearningreally} find that RL primarily improves pass@1 sampling efficiency without expanding the model's reasoning frontier. \citet{cui2025entropymechanismreinforcementlearning} identify early-stage entropy collapse that prematurely degrades exploration. \citet{tan2026scalingbehaviorsllmreinforcement} reveal a power-law relationship between test loss, computation, and data volume, showing larger models exhibit superior learning efficiency. From a parameter-dynamics perspective, \citet{cai2025predictability} uncover Rank-1 dominance and Rank-1 linear dynamics, enabling 2.5$\times$ acceleration via AlphaRL. \citet{cai2026learning} attribute on-policy distillation's efficiency to modular redundancy suppression and early directional stabilization, proposing EffOPD with 3$\times$ speedup \citep{song2026survey}.

\textbf{Steering Vectors and Representation Engineering.}
While most adaptation methods rely on parameter updates (full fine-tuning or LoRA), an alternative line of work intervenes directly on internal activations via vector steering. Representation Engineering (RepE) provides a unifying framework, treating population-level representations as the analytical unit for monitoring and manipulating cognitive phenomena \citep{zou2023representation}. High-level behaviors often exhibit low-dimensional structure: refusal is mediated by a single direction \citep{arditi2024refusal} or multi-dimensional concept cones \citep{wollschlager2025geometry}. Task vectors \citep{ilharco2022editing} and function vectors \citep{todd2024function} enable behavioral editing through vector arithmetic, and steering is increasingly viewed as a paradigm complementary to fine-tuning and prompting \citep{ostermann2026weights}.

\clearpage
\section{Preliminaries and Experimental Setup}
\label{sec:preliminaries-and-experimental-setup}

\subsection{Preliminaries}
\label{Preliminaries}
We consider three post-training paradigms: Reinforcement Learning (RL), On-Policy Distillation (OPD), and Off-Policy Distillation. Let $\pi_{\theta}$ denote the policy or student model, $\pi^*$ the fixed teacher model, and $\pi_{\mathrm{ref}}$ the fixed reference model used for regularization. The teacher provides distillation targets, whereas the reference model anchors RL regularization; these roles are distinct.

\paragraph{Reinforcement Learning (RL).}
The RL objective is:
\begin{equation}
\max_{\theta} J_{\mathrm{RL}}(\theta)
=
\mathbb{E}_{x \sim \mathcal{D}}
\left[
\mathbb{E}_{y \sim \pi_{\theta}(\cdot \mid x)}
\left[r(x,y)\right]
-
\beta
D_{\mathrm{KL}}
\left(
\pi_{\theta}(\cdot \mid x)
\parallel
\pi_{\mathrm{ref}}(\cdot \mid x)
\right)
\right],
\label{eq:rl_objective}
\end{equation}
where $y=(y_1,\ldots,y_T)$ is sampled from the current policy, $r(x,y)$ is the reward assigned to response $y$ for query $x$, and $\beta$ controls the strength of KL regularization.

In Reinforcement Learning from Verifiable Rewards (RLVR), $r(x,y)$ is determined by verifiable outcomes, such as answer correctness or unit-test execution, without requiring a learned reward model.

Separating the task-reward term from reference-policy regularization gives the policy-gradient form:
\begin{align}
\nabla_{\theta}J_{\mathrm{RL}}(\theta)
&=
\mathbb{E}_{x \sim \mathcal{D},\;
y \sim \pi_{\theta}(\cdot \mid x)}
\left[
\sum_{t=1}^{T}
A_t^{r}
\nabla_{\theta}
\log \pi_{\theta}(y_t \mid x,y_{<t})
\right]
\nonumber\\
&\quad -
\beta\nabla_{\theta}
\mathbb{E}_{x \sim \mathcal{D}}
\left[
D_{\mathrm{KL}}
\left(
\pi_{\theta}(\cdot \mid x)
\parallel
\pi_{\mathrm{ref}}(\cdot \mid x)
\right)
\right],
\label{eq:rl_gradient}
\end{align}
where $A_t^{r}$ denotes the token-level advantage associated with the task reward. In RLVR, this reward is typically observed only after the complete response has been generated. Equation~\ref{eq:rl_gradient} describes the underlying regularized objective; practical GRPO and DAPO updates use algorithm-specific advantage estimates and clipped surrogate objectives.

\paragraph{On-Policy Distillation (OPD).}
OPD samples trajectories from the student policy and minimizes the reverse KL divergence from the student to the teacher:
\begin{align}
\min_{\theta}J_{\mathrm{OPD}}(\theta)
&=
\mathbb{E}_{x \sim \mathcal{D}}
\left[
D_{\mathrm{KL}}
\left(
\pi_{\theta}(\cdot \mid x)
\parallel
\pi^*(\cdot \mid x)
\right)
\right]
\nonumber\\
&=
\mathbb{E}_{x \sim \mathcal{D},\;
y \sim \pi_{\theta}(\cdot \mid x)}
\left[
\sum_{t=1}^{T}
\left(
\log \pi_{\theta}(y_t \mid x,y_{<t})
-
\log \pi^*(y_t \mid x,y_{<t})
\right)
\right].
\label{eq:opd_objective}
\end{align}

The corresponding sequence-level gradient can be written as:
\begin{align}
\nabla_{\theta}J_{\mathrm{OPD}}(\theta)
=
\mathbb{E}_{x \sim \mathcal{D},\;
y \sim \pi_{\theta}(\cdot \mid x)}
\Bigg[
\sum_{t=1}^{T}
\sum_{t'=t}^{T}
&
\left(
\log \pi_{\theta}(y_{t'} \mid x,y_{<t'})
-
\log \pi^*(y_{t'} \mid x,y_{<t'})
\right)
\nonumber\\
&\cdot
\nabla_{\theta}
\log \pi_{\theta}(y_t \mid x,y_{<t})
\Bigg].
\label{eq:opd_gradient_full}
\end{align}

For optimization, we use an immediate token-level approximation that retains only the contribution at $t'=t$:
\begin{align}
\nabla_{\theta}J_{\mathrm{OPD}}(\theta)
\approx
\mathbb{E}_{x \sim \mathcal{D},\;
y \sim \pi_{\theta}(\cdot \mid x)}
\Bigg[
\sum_{t=1}^{T}
&
\left(
\log \pi_{\theta}(y_t \mid x,y_{<t})
-
\log \pi^*(y_t \mid x,y_{<t})
\right)
\nonumber\\
&\cdot
\nabla_{\theta}
\log \pi_{\theta}(y_t \mid x,y_{<t})
\Bigg].
\label{eq:opd_gradient_approx}
\end{align}

This approximation provides token-level supervision along student-generated trajectories, but omits the future-token contributions in Equation~\ref{eq:opd_gradient_full}. When implemented as an advantage-weighted surrogate, the log-probability difference is treated as a stop-gradient weight, with gradients taken through the student log-probability that it multiplies.

\paragraph{Off-Policy Distillation (Off-PD).}
Off-policy distillation instead samples trajectories from the teacher policy. Its objective is the forward KL divergence from the teacher to the student:
\begin{align}
\min_{\theta}J_{\mathrm{Off\text{-}PD}}(\theta)
&=
\mathbb{E}_{x \sim \mathcal{D}}
\left[
D_{\mathrm{KL}}
\left(
\pi^*(\cdot \mid x)
\parallel
\pi_{\theta}(\cdot \mid x)
\right)
\right]
\nonumber\\
&=
\mathbb{E}_{x \sim \mathcal{D},\;
y \sim \pi^*(\cdot \mid x)}
\left[
\sum_{t=1}^{T}
\left(
\log \pi^*(y_t \mid x,y_{<t})
-
\log \pi_{\theta}(y_t \mid x,y_{<t})
\right)
\right].
\label{eq:off_policy_distillation_objective}
\end{align}

Since the teacher-dependent term is constant with respect to $\theta$, this objective is equivalent to minimizing the negative log-likelihood of teacher-generated trajectories:
\begin{equation}
\mathcal{L}_{\mathrm{Off\text{-}PD}}(\theta)
=
-
\mathbb{E}_{x \sim \mathcal{D},\;
y \sim \pi^*(\cdot \mid x)}
\left[
\sum_{t=1}^{T}
\log \pi_{\theta}(y_t \mid x,y_{<t})
\right].
\label{eq:off_policy_distillation_nll}
\end{equation}

The corresponding gradient is:
\begin{equation}
\nabla_{\theta}\mathcal{L}_{\mathrm{Off\text{-}PD}}(\theta)
=
-
\mathbb{E}_{x \sim \mathcal{D},\;
y \sim \pi^*(\cdot \mid x)}
\left[
\sum_{t=1}^{T}
\nabla_{\theta}\log \pi_{\theta}(y_t \mid x,y_{<t})
\right].
\label{eq:off_policy_distillation_gradient}
\end{equation}

Thus, OPD and off-policy distillation differ in both trajectory sampling and the objective being optimized. OPD uses student-generated trajectories with a reverse-KL objective, whereas off-policy distillation uses teacher-generated trajectories with a forward-KL objective, implemented as sequence negative log-likelihood. The latter trains the student on teacher-generated prefixes, while the former evaluates the teacher on prefixes generated by the student.

\subsection{Experimental Setup}
\label{Experimental Setup}

\paragraph{Overview.}
Each experiment in this paper is defined by a training paradigm and an update configuration. The three paradigms are on-policy RL with verifiable rewards (Eqn.~\ref{eq:rl_objective}), on-policy distillation (Eqn.~\ref{eq:opd_objective}), and off-policy distillation (Eqn.~\ref{eq:off_policy_distillation_nll}), as formalized in Section~\ref{Preliminaries}. For RL, we consider full fine-tuning and LoRA, whereas for both on-policy and off-policy distillation, we additionally evaluate vector steering. These configurations differ in which parameters receive gradients and, consequently, in the dimensionality and expressiveness of the resulting updates. Within each matched comparison, we hold the base model, prompt set, teacher model when applicable, sampling recipe, and evaluation protocol fixed across the applicable update configurations. Parameterization-specific optimization settings are reported in Table~\ref{tab:hparams}, and execution differences are described below. In on-policy experiments, matching the sampling recipe does not imply identical trajectories, since the samples depend on the current policy. This setup allows us to compare update parameterizations under matched modeling and evaluation protocols while making their optimization differences explicit.

\paragraph{Models.}
To ensure the generality of our findings, we conduct experiments across model scales ranging from 1.5B to 14B parameters, spanning two model families: Qwen2.5/DeepSeek-R1-Distill (1.5B and 7B) and Qwen3 (4B, 8B, and 14B), and two RL algorithms (GRPO and DAPO). All RL models are trained locally using the \texttt{verl} framework \citep{Sheng_2025}. For all distillation students, the default teacher is the RL-tuned version of the student's own base model, ensuring that the teacher--student gap corresponds exactly to the RL update itself.

\paragraph{Datasets and Evaluation.}
The training data spans four topics, each with its own verifier and held-out evaluation
set (Table~\ref{tab:data}). Rewards are always \emph{verifiable} in the sense of
\S\ref{Preliminaries}: rule-based answer matching for mathematical and scientific
reasoning, unit-test execution for code, and programmatic constraint checking for
instruction following. For the off-policy distillation experiments in Fig.~\ref{fig2} (d), we additionally compare two sources of supervision: standard SFT on Math-CoT-44K \citep{ren2026rethinking_sft_generalization}, which contains approximately 44K mathematical queries and 32 responses sampled from Qwen3-32B for each query, yielding more than 1.4M reasoning traces; and off-policy distillation on OpenMathReasoning \citep{moshkov2025aimo2}, which contains 540K mathematical problems and 3.2M long chain-of-thought solutions generated by DeepSeek-R1 and QwQ-32B.

\begin{table}[h]
\centering
\small
\caption{Training and evaluation datasets and reward functions for each domain.}
\label{tab:data}
\setlength{\tabcolsep}{4pt}
\begin{tabular}{@{}llll@{}}
\toprule
\textbf{Domain}
& \textbf{Training Set}
& \textbf{Evaluation Set}
& \textbf{Reward} \\
\midrule

\multirow{2}{*}{Math}
& \multirow{2}{*}{DeepMath-103K
  \citep{he2025deepmath103klargescalechallengingdecontaminated}}
& AIME 2024 \citep{ye2025limoreasoning}
& \multirow{2}{*}{Answer matching} \\

&
& MATH500 \citep{lightman2023lets}
& \\

Science
& SciKnowEval-Training \citep{feng2024sciknoweval}
& SciKnowEval-Eval \citep{feng2024sciknoweval}
& Answer matching \\

Code
& Eurus \citep{cui2025process}
& LiveCodeBench-v5 \citep{jain2024livecodebench}
& Unit tests \\

Instruction following
& IFEval-Training \citep{zhao2024wildchat1mchatgptinteraction}
& IFEval-Eval \citep{zhao2024wildchat1mchatgptinteraction}
& Constraint checking \\

\bottomrule
\end{tabular}
\end{table}

\paragraph{Instruction-following split.}
For instruction following, we partition the dataset at the prompt level into an 80\% training split and a 20\% held-out evaluation split, denoted as IFEval-Training and IFEval-Eval, respectively. These names refer to the partitions used in this study. All responses associated with the same prompt remain in the same partition, and evaluation prompts are excluded from training. The split is kept fixed across all compared configurations.

\paragraph{Evaluation protocol.}
We report \texttt{mean@4} on the held-out evaluation sets. For each evaluation prompt, we independently sample four responses using the same decoding configuration. Each response is scored using the corresponding task-specific evaluation criterion. The reported result is the arithmetic mean of the four single-sample evaluation scores computed on the same held-out set:
\begin{equation}
\operatorname{mean@4}
=
\frac{1}{4}\sum_{j=1}^{4}
\operatorname{Score}
\left(\{(x_i,y_i^{(j)})\}_{i=1}^{N}\right),
\label{eq:mean_at_4}
\end{equation}
where $N$ is the number of evaluation prompts, $y_i^{(j)}$ is the $j$-th sampled response to prompt $x_i$, and $\operatorname{Score}$ is the task-specific evaluation metric. For mean per-response correctness, this is equivalently the average over all $4N$ scored responses. We do not select the best of the four responses or apply majority voting; \texttt{mean@4} is therefore distinct from \texttt{pass@4}.

\paragraph{Update configurations.}
We compare the following trainable parameterizations under the matched modeling and evaluation protocols described above, using the configuration-specific learning rates in Table~\ref{tab:hparams}.

\begin{itemize}[leftmargin=1.4em,itemsep=2pt]
\item \textbf{Full fine-tuning} updates all model parameters by default and serves as the unrestricted reference configuration. In principle, it can express a substantially broader class of parameter updates than the other configurations.

\item \textbf{LoRA} \citep{hu2021loralowrankadaptationlarge} injects low-rank adapters on \texttt{all-linear} target modules while keeping the backbone frozen. This restricts the update to a low-rank form, though each input can still produce a different correction (i.e., it remains context-dependent). The rank varies across experiments; Table~\ref{tab:hparams} reports the default rank-$8$ configuration.

\item \textbf{Vector steering} freezes the entire backbone and trains context-invariant bias vectors added to the residual stream. For a controlled layer set $\mathcal S$, the single-vector intervention is
\begin{equation}
\mathbf{h}_\ell \;\mapsto\; \mathbf{h}_\ell + \mathbf{v}_\ell,
\qquad \mathbf{v}_\ell \in \mathbb{R}^{d},\quad \ell \in \mathcal S,
\label{eq:vector_steering}
\end{equation}
with $|\mathcal S|d$ trainable parameters. For the default single-vector configuration on Qwen3-4B ($d=2560$) with $\mathcal S=\{5,\dots,20\}$ and $|\mathcal S|=16$, this amounts to $40{,}960$ scalars, approximately $4.1\times10^4$, or roughly $1.0\times10^{-5}$ of the backbone. Crucially, $\mathbf{v}_\ell$ does not vary with the input: the \emph{same} vector is added regardless of context or decoding position.
\end{itemize}

Unless otherwise stated, vector steering uses the controlled layer set $\mathcal S=\{5,\dots,20\}$ and a single vector per controlled layer ($K=1$). In the layer-wise experiments of Section~\ref{Layer-wise Compressibility of RL Updates}, all three configurations are restricted to the selected layer: full-parameter updating trains only that layer, LoRA adapts its linear modules, and vector steering adds an offset to its output. The remainder of the backbone is frozen. These selected-layer experiments are distinct from the default model-wide full-fine-tuning and LoRA configurations.

For multi-vector probes ($K>1$), we learn the ordered set $\{\mathbf{v}^{(0)}_\ell,\ldots,\mathbf{v}^{(K-1)}_\ell\}$ sequentially at each controlled layer. The first vector is trained to convergence and then retained as a fixed reference. At each subsequent stage $k$, \emph{only the current vector} $\mathbf{v}^{(k)}_\ell$ is injected and optimized; previously learned vectors remain inactive in the forward pass and are used solely to define the orthogonality constraint. After each optimizer step, the current vector is reprojected onto the orthogonal complement of the previously learned directions. Training advances to the next vector once the current-stage loss plateaus. Each direction is also evaluated independently, with only the corresponding vector active at each controlled layer. Consequently, each stage optimizes $|\mathcal S|d$ parameters, while the complete collection contains $K|\mathcal S|d$ scalars; increasing $K$ probes additional independently effective directions rather than simultaneously summing constant offsets.

Vector steering requires a substantially larger learning rate than the other two configurations; the values we use are reported in Table~\ref{tab:hparams}.

\begin{table}[h]
\centering\small
\caption{Learning rates and paradigm-specific hyperparameters.}
\label{tab:hparams}
\setlength{\tabcolsep}{5pt}
\begin{tabular}{@{}lccc@{}}
\toprule
& \textbf{RL} & \textbf{OPD} & \textbf{Off-PD} \\
\midrule
LR, full fine-tuning & $1\times10^{-6}$ & $1\times10^{-5}$ & $5\times10^{-5}$ \\
LR, LoRA ($r{=}8$)   & $3\times10^{-5}$ & $5\times10^{-6}$ & $1\times10^{-5}$ \\
LR, vector steering  & --- & $1\times10^{-1}$ & $5\times10^{-2}$ \\
\midrule
Max prompt / response & $2{,}048$ / $20{,}480$ & $3{,}072$ / $16{,}384$ & $3{,}072$ / $16{,}384$ \\
Prompt batch & $128$ & $1{,}024$ & $1{,}024$ \\
Mini-batch (per step) & $32$ & $1{,}024$ & $1{,}024$ \\
Training samples per prompt $n$ & $16$ & $1$ & $1$ \\
Evaluation samples per prompt & $4$ & $4$ & $4$ \\
Evaluation metric & \texttt{mean@4} & \texttt{mean@4} & \texttt{mean@4} \\
Rollout source & student & student & frozen teacher \\
Rollout IS correction & --- & token, threshold $5.0$ & off \\
Auxiliary reference-KL coeff. & $0.001$ (GRPO) / $0$ (DAPO) & $0$ & $0$ \\
Epochs & $1000$ & $1000$ & $1000$ \\
Precision & \texttt{bfloat16} & \texttt{bfloat16} & \texttt{bfloat16} \\
\bottomrule
\end{tabular}
\end{table}

The auxiliary reference-KL coefficient in Table~\ref{tab:hparams} refers to additional regularization toward $\pi_{\mathrm{ref}}$, not to the main distillation objective. In particular, a zero auxiliary coefficient does not remove the reverse-KL teacher supervision in OPD or the teacher-trajectory negative log-likelihood in Off-PD.

\paragraph{Prompts.}
All Qwen3 models are run and trained in \emph{non-thinking} mode
(\texttt{enable\_thinking=False} passed to the chat template) so that the response
distribution is directly comparable between the base model, the RL teacher, and every
steered student; DeepSeek-R1-Distill students keep thinking mode enabled, as their base
model is trained that way. We always use each model's own built-in chat template rather
than a hand-written wrapper, and do not manually prepend an additional system message.
Each domain uses its own answer-format instruction, which is also what the corresponding
verifier parses. The examples below show \emph{user-message content only}, with
\texttt{\{question\}} denoting the raw problem statement. Role delimiters and model-specific
thinking prefixes are supplied by the native chat template rather than manually inserted
into the user message.

For \textbf{mathematical reasoning}, the answer-extraction instruction is appended to the
question:
\begin{verbatim}
{question}
Please reason step by step, and put your final answer within \boxed{}.
\end{verbatim}

\noindent
For \textbf{scientific reasoning}, the instruction is placed first and the question is a
multiple-choice item whose options are enumerated A--J, so that the verifier can match a
single letter:
\begin{verbatim}
Solve the following problem. Make sure to put the answer (and only answer)
inside \boxed{}.

{question}

A: {option_A}
B: {option_B}
...
J: {option_J}
\end{verbatim}

\noindent
For \textbf{code generation}, the instruction requires a single fenced Python block, which
the sandbox extracts verbatim and executes against the hidden unit tests:
\begin{verbatim}
{question}

Write Python code to solve the problem. Present the code in
```python
Your code
```
at the end.
You need to think first then write the Python code.
\end{verbatim}

\noindent
For \textbf{instruction following}, the constraints are stated in natural language at the
head of the prompt and no separate reasoning instruction is added, since the verifiable
constraints concern the response's surface form:
\begin{verbatim}
Please follow these instructions in your response: {constraint_1} and
{constraint_2} and ... {constraint_n}. {question}
\end{verbatim}

\noindent
Teacher rollouts for off-policy distillation are generated with the \emph{same}
model-specific prompt template as the corresponding on-policy runs. Prompt construction
is therefore matched, while the two paradigms retain the distinct trajectory-sampling
distributions and distillation objectives defined in Section~\ref{Preliminaries}.

\paragraph{Infrastructure.}
All training runs use $8\times$ or $32\times$ H20 96GB GPUs with FSDP for the actor and
vLLM for generation. Vector-steering runs disable CUDA graph capture
(\texttt{enforce\_eager=true}) because the steering hook mutates hidden states inside the
vLLM forward pass; LoRA and full-parameter runs leave CUDA graphs enabled. Rollout tensor
parallelism is $4$ for RL and $2$ for distillation. When a fixed model is used for scoring,
its parameters are offloaded to CPU between forward passes; this model serves as the
reference policy for RL regularization or the teacher for distillation supervision,
depending on the paradigm. The two commands below give the on-policy and off-policy
distillation invocations for vector steering. The full-parameter and LoRA variants use
the corresponding update parameterizations and hyperparameters in Table~\ref{tab:hparams},
with the configuration-specific generation-engine settings described above.

\begin{tcolorbox}[
    colback=gray!5, colframe=gray!70, arc=3pt,
    left=2pt, right=2pt, top=2pt, bottom=2pt, boxrule=0.5pt,
    breakable, fontupper=\small, listing only,
    listing options={language=bash},
    title=On-Policy Distillation Command (vector steering), title style={color=black},
    label=cmd:opd
]
\begin{verbatim}
# examples/representation/4bopd_single.sh -> 4bopd.sh
python3 -m verl.trainer.main_ppo \
    algorithm.adv_estimator=grpo \
    algorithm.use_kl_in_reward=False \
    algorithm.rollout_correction.rollout_is=token \
    algorithm.rollout_correction.rollout_is_threshold=5.0 \
    algorithm.rollout_correction.rollout_rs=null \
    algorithm.rollout_correction.bypass_mode=false \
    data.train_files=/path/to/DeepMath-103K/train_filtered_level6.parquet \
    data.val_files="['/path/to/AIME2024/test.parquet']" \
    data.train_batch_size=1024 \
    data.max_prompt_length=3072 \
    data.max_response_length=16384 \
    data.filter_overlong_prompts=True \
    data.truncation='error' \
    data.shuffle=True \
    data.seed=42 \
    data.return_raw_chat=True \
    +data.apply_chat_template_kwargs.enable_thinking=False \
    actor_rollout_ref.model.path=$MODEL_PATH \
    +actor_rollout_ref.model.base_model_path=$MODEL_PATH \
    +actor_rollout_ref.ref.model.path=$TEACHER_MODEL_PATH \
    actor_rollout_ref.model.lora_rank=0 \
    actor_rollout_ref.model.enable_trainable_token_vector=true \
    actor_rollout_ref.model.trainable_token_vector_mode=single \
    actor_rollout_ref.model.trainable_token_vector_num=1 \
    actor_rollout_ref.model.trainable_token_vector_layer_start=5 \
    actor_rollout_ref.model.trainable_token_vector_layer_end=20 \
    actor_rollout_ref.model.trainable_token_vector_sampling_method=hypersphere \
    actor_rollout_ref.model.trainable_token_vector_scale=0.1 \
    actor_rollout_ref.model.trainable_token_vector_learnable_alpha=false \
    actor_rollout_ref.model.trainable_token_vector_alpha_init=0.0 \
    actor_rollout_ref.model.trainable_token_vector_curriculum=none \
    actor_rollout_ref.model.trainable_token_vector_force_all_tokens=true \
    actor_rollout_ref.model.use_remove_padding=True \
    actor_rollout_ref.model.enable_gradient_checkpointing=True \
    actor_rollout_ref.actor.optim.lr=1e-1 \
    actor_rollout_ref.actor.optim.lr_warmup_steps_ratio=0.0 \
    actor_rollout_ref.actor.policy_loss.only_reverse_kl_advantages=True \
    actor_rollout_ref.actor.ppo_mini_batch_size=1024 \
    actor_rollout_ref.actor.ppo_micro_batch_size_per_gpu=1 \
    actor_rollout_ref.actor.use_kl_loss=True \
    actor_rollout_ref.actor.kl_loss_coef=0 \
    actor_rollout_ref.actor.kl_loss_type=low_var_kl \
    actor_rollout_ref.actor.entropy_coeff=0 \
    actor_rollout_ref.actor.ppo_max_token_len_per_gpu=32768 \
    actor_rollout_ref.actor.fsdp_config.dtype=bfloat16 \
    actor_rollout_ref.actor.fsdp_config.model_dtype=fp32 \
    actor_rollout_ref.actor.fsdp_config.use_orig_params=True \
    actor_rollout_ref.rollout.name=vllm \
    actor_rollout_ref.rollout.calculate_log_probs=true \
    actor_rollout_ref.rollout.n=1 \
    actor_rollout_ref.rollout.temperature=1.0 \
    actor_rollout_ref.rollout.top_p=1.0 \
    actor_rollout_ref.rollout.tensor_model_parallel_size=2 \
    actor_rollout_ref.rollout.gpu_memory_utilization=0.8 \
    actor_rollout_ref.rollout.enforce_eager=true \
    actor_rollout_ref.rollout.free_cache_engine=True \
    actor_rollout_ref.rollout.max_num_batched_tokens=32768 \
    actor_rollout_ref.rollout.val_kwargs.do_sample=True \
    actor_rollout_ref.rollout.val_kwargs.temperature=1.0 \
    actor_rollout_ref.rollout.val_kwargs.top_p=1.0 \
    actor_rollout_ref.rollout.val_kwargs.n=4 \
    actor_rollout_ref.ref.log_prob_micro_batch_size_per_gpu=1 \
    actor_rollout_ref.ref.fsdp_config.param_offload=True \
    reward_model.reward_manager=naive \
    trainer.save_vector=true \
    trainer.save_vector_dir=/path/to/trainable_vectors \
    trainer.logger='["console","wandb"]' \
    trainer.n_gpus_per_node=8 \
    trainer.nnodes=1 \
    trainer.test_freq=10 \
    trainer.total_epochs=1000 $@
\end{verbatim}
\end{tcolorbox}

\begin{tcolorbox}[
    colback=gray!5, colframe=gray!70, arc=3pt,
    left=2pt, right=2pt, top=2pt, bottom=2pt, boxrule=0.5pt,
    breakable, fontupper=\small, listing only,
    listing options={language=bash},
    title=Off-Policy Distillation Command (vector steering), title style={color=black},
    label=cmd:offpd
]
\begin{verbatim}
# Stage 1: dump teacher rollouts once
python3 -m verl.trainer.main_generation \
    model.path=$TEACHER_MODEL_PATH \
    data.path=/path/to/DeepMath-103K/train_filtered_level6.parquet \
    data.prompt_key=prompt \
    data.n_samples=1 \
    data.batch_size=1024 \
    rollout.temperature=1.0 \
    rollout.top_p=1.0 \
    rollout.prompt_length=3072 \
    rollout.response_length=16384 \
    rollout.max_num_batched_tokens=32768 \
    rollout.tensor_model_parallel_size=2 \
    data.output_path=$OUTPUT_BASE/round_${i}.parquet

python3 examples/representation/merge_rollouts_to_sft.py \
    --inputs $OUTPUT_BASE/round_*.parquet \
    --output $OUTPUT_BASE/teacher_sft_all.parquet \
    --prompt_key prompt --response_key response

# Stage 2: distill off-policy on the frozen corpus
python3 -m verl.trainer.main_ppo \
    algorithm.adv_estimator=grpo \
    algorithm.rollout_correction.rollout_is=null \
    algorithm.use_kl_in_reward=False \
    data.train_files=$OUTPUT_BASE/teacher_sft_all.parquet \
    data.val_files="['/path/to/AIME2024/test.parquet']" \
    data.train_batch_size=1024 \
    data.max_prompt_length=3072 \
    data.max_response_length=16384 \
    data.return_raw_chat=True \
    +data.apply_chat_template_kwargs.enable_thinking=False \
    actor_rollout_ref.model.path=$MODEL_PATH \
    +actor_rollout_ref.ref.model.path=$TEACHER_MODEL_PATH \
    actor_rollout_ref.model.lora_rank=0 \
    actor_rollout_ref.model.enable_trainable_token_vector=true \
    actor_rollout_ref.model.trainable_token_vector_mode=single \
    actor_rollout_ref.model.trainable_token_vector_num=1 \
    actor_rollout_ref.model.trainable_token_vector_layer_start=5 \
    actor_rollout_ref.model.trainable_token_vector_layer_end=20 \
    actor_rollout_ref.model.trainable_token_vector_sampling_method=hypersphere \
    actor_rollout_ref.model.trainable_token_vector_scale=0.1 \
    actor_rollout_ref.model.trainable_token_vector_learnable_alpha=false \
    actor_rollout_ref.model.trainable_token_vector_alpha_init=0.0 \
    actor_rollout_ref.model.trainable_token_vector_curriculum=none \
    actor_rollout_ref.model.trainable_token_vector_force_all_tokens=true \
    actor_rollout_ref.actor.optim.lr=5e-2 \
    actor_rollout_ref.actor.policy_loss.only_reverse_kl_advantages=true \
    actor_rollout_ref.actor.ppo_mini_batch_size=1024 \
    actor_rollout_ref.actor.ppo_micro_batch_size_per_gpu=1 \
    actor_rollout_ref.actor.ppo_max_token_len_per_gpu=32768 \
    actor_rollout_ref.actor.use_kl_loss=True \
    actor_rollout_ref.actor.kl_loss_coef=0 \
    actor_rollout_ref.actor.entropy_coeff=0 \
    actor_rollout_ref.rollout.offline_teacher_rollout=true \
    actor_rollout_ref.rollout.offline_teacher_data_path=$./teacher_sft_all.parquet \
    actor_rollout_ref.rollout.offline_teacher_response_key=response \
    actor_rollout_ref.rollout.name=vllm \
    actor_rollout_ref.rollout.n=1 \
    actor_rollout_ref.rollout.tensor_model_parallel_size=2 \
    actor_rollout_ref.rollout.gpu_memory_utilization=0.5 \
    actor_rollout_ref.rollout.enforce_eager=true \
    actor_rollout_ref.ref.fsdp_config.param_offload=True \
    reward_model.reward_manager=naive \
    trainer.save_vector=true \
    trainer.logger='["console","wandb"]' \
    trainer.n_gpus_per_node=8 \
    trainer.nnodes=1 \
    trainer.test_freq=5 \
    trainer.total_epochs=1000 $@
\end{verbatim}
\end{tcolorbox}

\clearpage

\begin{figure}[t]
    \centering
    \includegraphics[width=\textwidth]{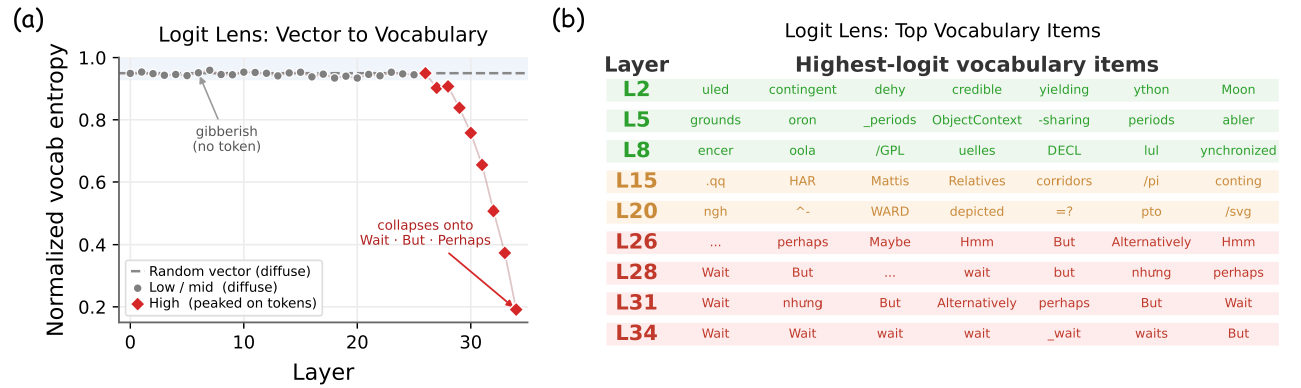}
    \caption{
    Vocabulary-space readout of steering vectors across network depth.
    (a) Normalized entropy of
    \(\operatorname{softmax}(
    \operatorname{RMSNorm}(\mathbf{v}_L)\mathbf{W}_U^\top)\).
    The dashed line at \(0.9498\) denotes the mean entropy of normalized
    Gaussian directions passed through the same readout pipeline.
    (b) Highest-logit vocabulary items for steering vectors learned at
    representative layers. All vectors are normalized before readout.
    }
    \label{fig:logit_lens}
\end{figure}

\section{Additional Experiments for Effective Manifold Capacity}
\label{Additional Experiments for Effective Manifold Capacity}

\subsection{Vocabulary-Space Readout}
\label{sec:vocabulary_readout}

The layer-wise experiments in the main text show that fixed-vector
steering recovers most of the RL-induced performance gain at lower and
intermediate layers, but becomes substantially less effective near the
output. To better understand this difference, we examine the
vocabulary-space readout of steering vectors learned at different
depths. This analysis characterizes how the learned directions align
with output vocabulary items and how this alignment changes with
injection depth.

For a steering vector \(\mathbf{v}_L \in \mathbb{R}^{H}\) learned at
layer \(L\), we bypass subsequent Transformer layers and directly apply
RMS normalization followed by the unembedding matrix \citep{geva2022transformerfeedforwardlayersbuild}:
\begin{equation}
\boldsymbol{\ell}(\mathbf{v}_L)
=
\mathbf{W}_U\operatorname{RMSNorm}(\mathbf{v}_L),
\qquad
\operatorname{RMSNorm}(\mathbf{v})
=
\frac{\mathbf{v}}
{\sqrt{\tfrac{1}{H}\sum_{j=1}^{H}v_j^2+\varepsilon}},
\qquad
\varepsilon=10^{-6},
\label{eq:vocabulary_readout}
\end{equation}
where \(\mathbf{W}_U \in \mathbb{R}^{V \times H}\) is the unembedding
matrix, and \(H\) and \(V\) denote the hidden dimension and vocabulary
size, respectively. The resulting readout distribution is
\begin{equation}
\mathbf{p}_L
=
\operatorname{softmax}
\left(
\boldsymbol{\ell}(\mathbf{v}_L)
\right).
\label{eq:vocabulary_distribution}
\end{equation}
The learned vector norms vary across injection depths.
RMS normalization reduces the influence of these magnitude differences,
allowing the readout to primarily characterize the vocabulary alignment
of each steering direction.

We quantify the concentration of the readout distribution using
normalized entropy:
\begin{equation}
\widetilde{\mathcal{H}}(\mathbf{p}_L)
=
-\frac{1}{\log V}
\sum_{t=1}^{V}p_{L,t}\log p_{L,t}.
\label{eq:vocabulary_entropy}
\end{equation}
Values closer to one indicate a more diffuse distribution, whereas
lower values indicate greater concentration on a smaller set of
vocabulary items.

To establish a random-direction baseline, we sample isotropic Gaussian
vectors with the same dimensionality as the steering vectors,
normalize them to unit \(L_2\) norm, and pass them through the same
RMS normalization, unembedding, and softmax operations:
\begin{equation}
\mathbf{z}^{(m)}
\sim
\mathcal{N}(\mathbf{0},\mathbf{I}_H),
\qquad
\bar{\mathbf{z}}^{(m)}
=
\frac{\mathbf{z}^{(m)}}{\|\mathbf{z}^{(m)}\|_2},
\end{equation}
\begin{equation}
\mathbf{p}_{\mathrm{null}}^{(m)}
=
\operatorname{softmax}
\left(
\mathbf{W}_U
\operatorname{RMSNorm}(\bar{\mathbf{z}}^{(m)})
\right).
\end{equation}
Averaging the normalized entropy over the \(M\) sampled directions
gives
\begin{equation}
\widetilde{\mathcal{H}}_{\mathrm{null}}
=
\frac{1}{M}
\sum_{m=1}^{M}
\widetilde{\mathcal{H}}
\left(
\mathbf{p}_{\mathrm{null}}^{(m)}
\right)
=
0.9498.
\label{eq:vocabulary_null}
\end{equation}

As shown in Fig.~\ref{fig:logit_lens} (a), the normalized entropy
remains close to the random baseline at lower and intermediate layers,
indicating a diffuse vocabulary readout.
The highest-logit items from L2 to L20 in
Fig.~\ref{fig:logit_lens}(b) likewise consist mainly of heterogeneous
subword fragments and symbols, without an apparent common semantic
theme. Combined with the strong performance of steering at these
depths, this observation shows that effective intermediate-layer
interventions need not exhibit a concentrated direct preference for
particular output tokens. This pattern is consistent with the role
of downstream computation: as a shared additive correction propagates
through the remaining layers, its interaction with different hidden
states can produce context-dependent output adjustments.

Near the output, the readout becomes increasingly concentrated.
Beginning around L26, normalized entropy falls below the random
baseline, while the highest-logit items increasingly include
\texttt{Wait}, \texttt{But}, \texttt{perhaps}, and
\texttt{Alternatively}.
These tokens are associated with contrast, qualification, and
reconsideration, suggesting that high-layer steering directions have
more explicit vocabulary-level associations.
The readout thus shifts from a diffuse pattern at lower and
intermediate layers toward a concentrated set of discourse-related
items near the output.

This increase in vocabulary interpretability does not coincide with
better steering performance. Instead, the main-text depth sweep shows
that fixed-vector steering deteriorates at these higher layers.
A direction can therefore have a clear vocabulary-level association
while remaining insufficient for effective task adaptation.
The issue is not simply whether a direction aligns with meaningful
output tokens, but how that direction is applied across token
positions. Expressions of contrast, qualification, or reconsideration
are useful in different contexts, whereas a fixed steering vector
applies the same additive correction at every position, without
explicitly adjusting its strength to the current representation.

\textbf{Summary.} The vocabulary readout provides a complementary view of the
layer-wise results.
Lower- and intermediate-layer steering remains effective despite
its diffuse direct readout, whereas high-layer directions exhibit
more concentrated vocabulary associations alongside weaker
fixed-vector performance.
This pattern suggests that direct vocabulary alignment alone does
not account for steering effectiveness.
One possible explanation is that the remaining layers allow earlier
interventions to acquire context-dependent effects, while later
interventions have less opportunity for such transformation.
The gains from high-layer gating are consistent with this account,
supporting input dependence as a useful factor in understanding
the depth sensitivity of fixed-vector steering.

\begin{figure*}[t]
    \centering
    \includegraphics[width=\textwidth]{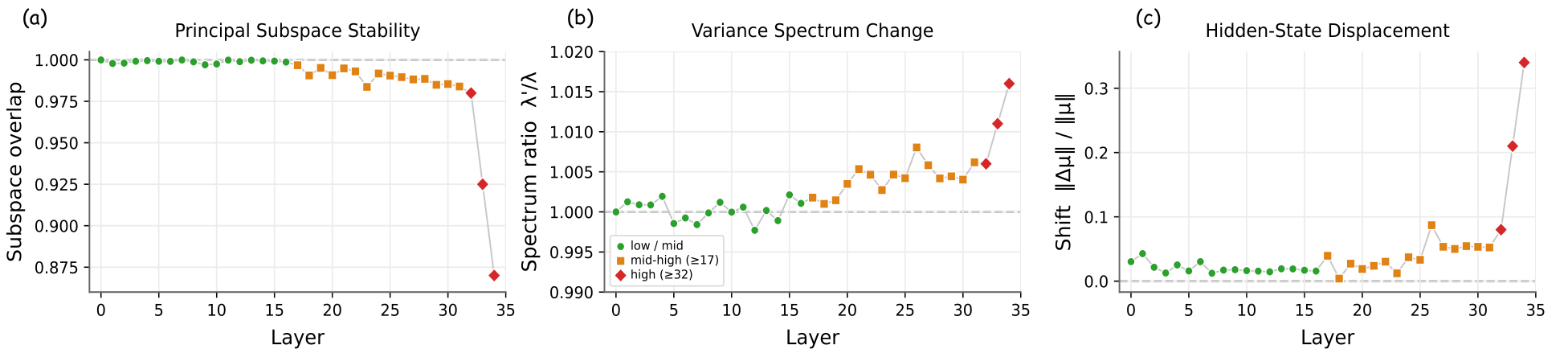}
    \caption{
Effect of single-vector steering on hidden-state geometry in Qwen3-4B
(\(k=20\), \(H=2560\)), measured at layer \(L+1\) for injection at layer \(L\).
(a) Mean principal-subspace overlap between clean and steered activations.
(b) Ratio of the top-\(k\) explained variance after and before steering.
(c) Relative displacement of the activation centroid.
Dashed lines denote the corresponding reference values.
}
    \label{fig:pca-null}
\end{figure*}

\subsection{Local Representation Geometry under Single-Vector Steering}
\label{app:local-representation-geometry}

Section~\ref{Layer-wise Compressibility of RL Updates} shows that
single-vector steering recovers most of the RL-induced performance
gain at lower and intermediate layers, but becomes less effective
near the output. The training dynamics and gating results suggest
that the expressiveness of a fixed additive correction contributes
to this depth dependence. Here, we complement the performance
analysis by examining changes in hidden-state geometry across
injection depths, focusing on the leading principal subspaces,
their associated variance, and the activation means.

\paragraph{Experimental Setup.}
For each injection layer \(L\), we collect hidden states from the
base and steered models on the same held-out inputs. The steering
vector is added to the output of layer \(L\), and representations
are measured at the output of layer \(L+1\). This captures the
local effect of steering after one subsequent Transformer layer.
Aggregating all valid token positions gives
\begin{equation}
\mathbf{X}_{L}\in\mathbb{R}^{N\times H},
\qquad
\mathbf{X}'_{L}\in\mathbb{R}^{N\times H},
\end{equation}
where \(N\) is the number of token positions and \(H\) is the hidden
dimension. The subscript \(L\) indexes the injection layer; both
matrices contain activations measured at layer \(L+1\).
Their respective means are
\begin{equation}
\boldsymbol{\mu}_{L}
=
\frac{1}{N}
\sum_{i=1}^{N}\mathbf{x}_{L,i},
\qquad
\boldsymbol{\mu}'_{L}
=
\frac{1}{N}
\sum_{i=1}^{N}\mathbf{x}'_{L,i}.
\end{equation}

\begin{figure}[t]
    \includegraphics[width=1\textwidth]{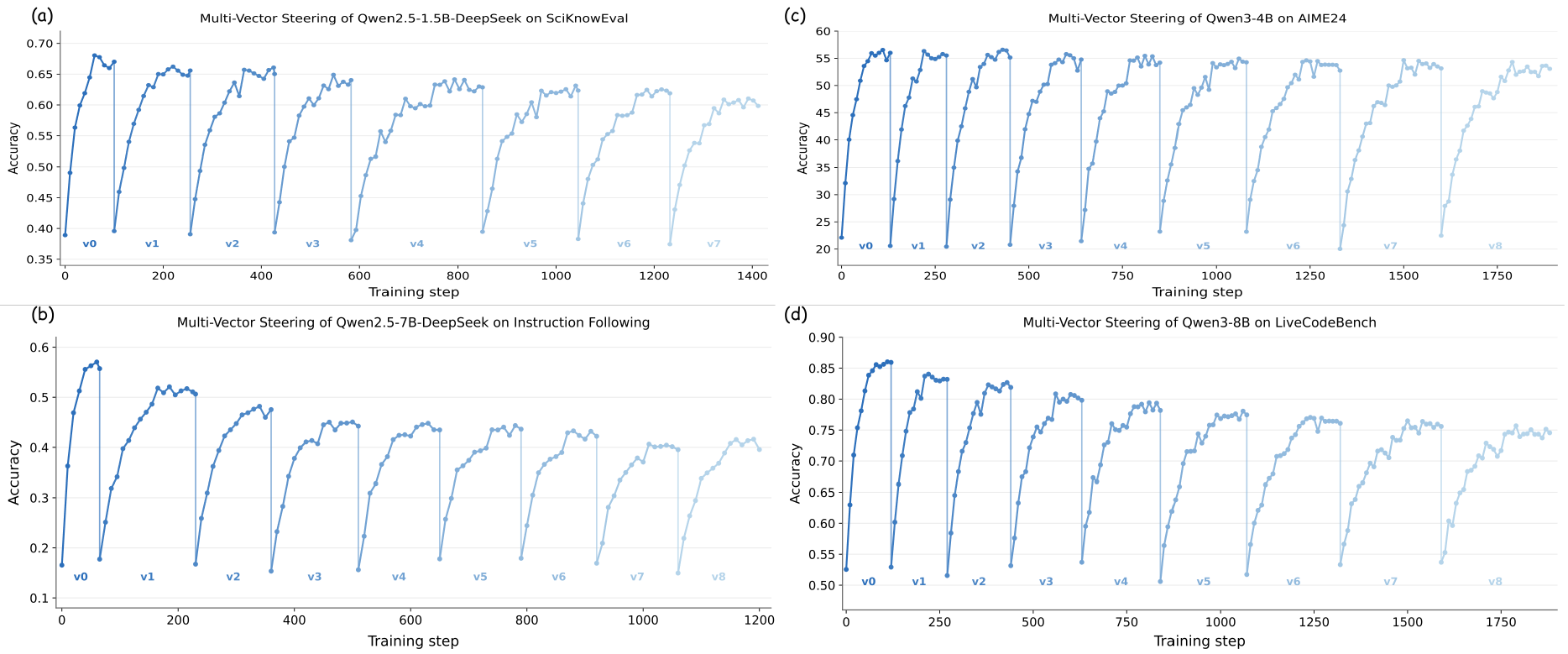}
    \caption{
Multi-vector steering across all four task domains and model scales. The first vector is consistently the most effective, with subsequent directions showing gradually decreasing standalone performance; the decay is slower for larger models.
}
    \label{appendix3}
\end{figure}

We center each activation matrix using its own mean and independently
compute low-rank singular value decompositions:
\begin{equation}
\mathbf{X}_{L,c}
\approx
\mathbf{U}_{L}
\boldsymbol{\Sigma}_{L}
\mathbf{V}_{L}^{\top},
\qquad
\mathbf{X}'_{L,c}
\approx
\mathbf{U}'_{L}
\boldsymbol{\Sigma}'_{L}
{\mathbf{V}'_{L}}^{\top}.
\end{equation}
The first \(k\) right singular vectors form the orthonormal
principal-component bases
\(\mathbf{V}_{L,k},\mathbf{V}'_{L,k}\in\mathbb{R}^{H\times k}\).
Writing \(\Sigma_{L,j}\) and \(\Sigma'_{L,j}\) for the corresponding
singular values, the covariance eigenvalues are
\begin{equation}
\lambda_{L,j}
=
\frac{\Sigma_{L,j}^{2}}{N-1},
\qquad
\lambda'_{L,j}
=
\frac{{\Sigma'_{L,j}}^{2}}{N-1}.
\end{equation}
For Qwen3-4B, we use \(k=20\) and \(H=2560\).

\paragraph{Principal-Subspace Alignment.}
We measure the alignment between the clean and steered top-\(k\)
principal subspaces as
\begin{equation}
\operatorname{Overlap}_{L,k}
=
\frac{1}{k}
\sum_{j=1}^{k}
\sigma_j
\left(
\mathbf{V}_{L,k}^{\top}
\mathbf{V}'_{L,k}
\right),
\end{equation}
where \(\sigma_j(\cdot)\) denotes the \(j\)-th singular value.
These singular values are the cosines of the principal angles
between the two subspaces, so an overlap close to one indicates
similar subspace orientations.

As shown in Fig.~\ref{fig:pca-null}(a), the overlap remains close
to one for lower- and intermediate-layer injections. At these
depths, the leading directions of activation variation therefore
remain largely aligned after one layer of propagation.
The overlap decreases toward the output, with the largest
departures occurring at the latest injection sites.
This indicates greater changes in the orientation of the leading
subspace for these high-layer interventions.

\paragraph{Variance Captured by the Leading Components.}
We next compare the total variance captured by the top-\(k\)
components of the steered and clean activations:
\begin{equation}
\rho_{L,k}
=
\frac{
\sum_{j=1}^{k}\lambda'_{L,j}
}{
\sum_{j=1}^{k}\lambda_{L,j}
}.
\end{equation}
The numerator and denominator use the leading eigenvalues of their
respective covariance matrices. A ratio close to one indicates
similar amounts of variance captured by the two principal subspaces.

Fig.~\ref{fig:pca-null}(b) shows that this ratio stays close to one
at lower and intermediate injection depths. Along with the high
subspace overlap, this suggests that steering at these depths
introduces relatively small changes in both the orientation and
the total variance of the leading components.
The ratio departs further from one near the output, although the
changes in captured variance remain modest.
Thus, the reduced subspace alignment at the latest injection sites
is accompanied by a comparatively small change in the variance
captured by the leading components.

\paragraph{Displacement of the Activation Centroid.}
To characterize changes in the mean representation, we measure
the relative centroid displacement:
\begin{equation}
D_{L}
=
\frac{
\left\|
\boldsymbol{\mu}'_{L}
-
\boldsymbol{\mu}_{L}
\right\|_2
}{
\left\|
\boldsymbol{\mu}_{L}
\right\|_2
+
10^{-9}
}.
\end{equation}
This statistic expresses the change in the activation mean relative
to the norm of the clean mean.

As shown in Fig.~\ref{fig:pca-null} (c), relative centroid
displacement is small for lower- and intermediate-layer injections
and increases near the output, particularly at the latest
injection sites.
The strong performance of earlier-layer steering therefore
coincides with only a small relative change in the mean
representation at the next layer.
High-layer steering produces a larger relative mean shift,
despite its weaker task performance.

\paragraph{Relation to the Layer-Wise Expressiveness Bottleneck.}
Effective steering at lower and intermediate layers is accompanied
by limited changes in the measured local geometry. Near the output,
reduced subspace alignment and larger centroid shifts coexist with
weaker task performance. These observations are consistent with the
expressiveness analysis in
Section~\ref{Layer-wise Compressibility of RL Updates}:
high-layer degradation occurs despite measurable changes in the
representations, and the gains from gating suggest that
input-dependent modulation helps make these interventions
more effective.

\clearpage
\begin{figure*}[t]
    \centering
    \includegraphics[width=\textwidth]{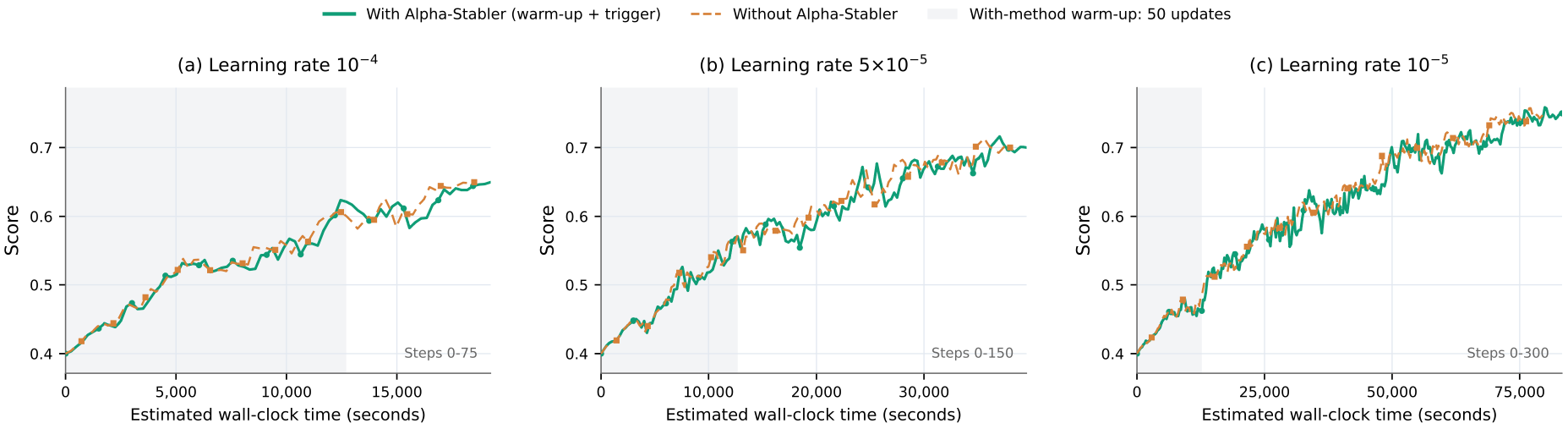}
    \caption{
    Comparison of pre-collapse score trajectories with and without Alpha-Stabler at learning rates of \textbf{(a)} \(10^{-4}\), \textbf{(b)} \(5\times10^{-5}\), and \textbf{(c)} \(10^{-5}\). Shaded regions denote the first \(50\) warm-up updates of Alpha-Stabler. Controlled and uncontrolled runs exhibit essentially the same score growth before collapse, indicating that Alpha-Stabler's stabilization introduces almost no additional computational overhead.
    }
    \label{appendix4}
\end{figure*}

\begin{algorithm}[t]
\caption{Alpha-Stabler: Warm-Up Calibration and PSI-Triggered Gradient Projection}
\label{alpha_stabler}
\begin{algorithmic}[1]
\Require Actor $\theta$; frozen base $\theta_0$; monitored layers $\mathcal S$ at $1/4$, $1/2$, $3/4$ depth
\Require Principal fraction $q=0.10$; warm-up $T_{\mathrm{warm}}=50$; interval $M=3$; EMA $\beta=0.95$; persistence $p=3$; thresholds $c_{\mathrm{on}}=3$, $c_{\mathrm{safe}}=2$
\Require RL optimizer and LR $\eta_\theta$; stabilizer $\epsilon>0$; min.\ avg.\ squared token shift $\eta_{\min}>0$
\Ensure Trained actor $\theta$ (inference uses the actor alone)
\State Run $T_{\mathrm{warm}}$ unprojected RL updates; accumulate base activation covariance on actor rollouts
\State Every $M$ updates, retain detached shifts in $\mathcal X_\ell$ if $N>0$ and $\|X_\ell\|_F^2/N>\eta_{\min}$
\State Estimate $U_{\ell,r}$ from the reference covariance, $r=\lceil qd\rceil$; freeze bases
\For{each $\ell\in\mathcal S$}
    \State $\mathcal W_\ell\gets\{\|XU_{\ell,r}\|_F^2/(\|X\|_F^2+\epsilon):X\in\mathcal X_\ell\}$
    \If{$\mathcal W_\ell$ is empty} \State \textbf{Stop and recalibrate} \EndIf
    \State $m_\ell\gets\operatorname{median}\mathcal W_\ell$; $s_\ell\gets1.4826\,\operatorname{MAD}(\mathcal W_\ell)+\epsilon$
    \State $\tau_\ell^{\mathrm{on}}\gets m_\ell+c_{\mathrm{on}}s_\ell$; $\tau_\ell^{\mathrm{safe}}\gets m_\ell+c_{\mathrm{safe}}s_\ell$
    \If{$0<\tau_\ell^{\mathrm{safe}}<\tau_\ell^{\mathrm{on}}<1$ fails} \State \textbf{Stop and recalibrate} \EndIf
    \State $a_\ell\gets0$; $n_\ell\gets0$; $\bar\psi_\ell\gets m_\ell$
\EndFor
\State Discard calibration records; $t\gets T_{\mathrm{warm}}$
\While{training not complete}
    \State Sample rollouts; compute rewards, advantages, and behavior log-probs
    \For{each minibatch $\mathcal B$ (one optimizer update)}
        \State $t\gets t+1$; clear parameter gradients
        \State Actor forward for $\mathcal L_{\mathrm{RL}}$; retain graph
        \If{$t\bmod M=0$}
            \State Frozen-base activations on the same tokens; $N\gets$ valid positions
            \For{each $\ell\in\mathcal S$}
                \State $X_\ell\gets\operatorname{sg}(H_\ell(\theta;\mathcal B)-H_\ell(\theta_0;\mathcal B))$
                \If{$N>0$ and $\|X_\ell\|_F^2/N>\eta_{\min}$}
                    \State $\psi_\ell\gets\|X_\ell U_{\ell,r}\|_F^2/(\|X_\ell\|_F^2+\epsilon)$
                    \State $\bar\psi_\ell\gets\beta\bar\psi_\ell+(1-\beta)\psi_\ell$
                    \State $n_\ell\gets n_\ell+1$ if $\bar\psi_\ell>\tau_\ell^{\mathrm{on}}$, else $0$
                    \State $a_\ell\gets1$ if $n_\ell\geq p$; $a_\ell\gets0$ if $\bar\psi_\ell<\tau_\ell^{\mathrm{safe}}$; else keep $a_\ell$
                \Else
                    \State $n_\ell\gets0$ \Comment{keep $\bar\psi_\ell$, $a_\ell$}
                \EndIf
            \EndFor
        \EndIf
        \State Fix flags during backward and optimizer step
        \State Backprop $\mathcal L_{\mathrm{RL}}$ with $\widetilde G_\ell^{(t)}=G_\ell^{(t)}-a_\ell^{(t)}(G_\ell^{(t)}U_{\ell,r})U_{\ell,r}^{\top}$
        \State Update $\theta$ with LR $\eta_\theta$
    \EndFor
\EndWhile
\State \Return $\theta$
\Statex \textit{Note:} $H_\ell(\theta;\mathcal B)$ stacks layer-$\ell$ activations at valid tokens in $\mathcal B$; $\operatorname{sg}$ is stop-gradient.
\end{algorithmic}
\end{algorithm}

\clearpage
% =====================================================================
\section{Theory of Low-Dimensional RLVR Steering}
\label{app:theory}

We develop a conditional theoretical account of when RLVR-induced output changes admit a compact activation-space intervention. We study recoverability within a specified intervention class, rather than the intrinsic dimension of the full RL parameter or activation trajectory.
\textbf{Part~I} formulates local distillation as Fisher-weighted least squares in Section~\ref{sec:reduction}, gives sufficient conditions for shared teacher-induced shifts in Section~\ref{sec:tilt}, and derives a recovery identity and a three-constant lower bound in Section~\ref{sec:recovery}.
\textbf{Part~II} characterizes sequentially constrained constant-vector fitting in Section~\ref{sec:multi} and distinguishes activation variance from output sensitivity in Section~\ref{sec:complement}.
\textbf{Part~III} analyzes Alpha-Stabler's fixed-reference PSI statistic and backward gradient projection in Section~\ref{sec:predicter}.
Section~\ref{sec:summary} summarizes the results and their theoretical or empirical scope.
Two relations organize the recovery analysis:
\begin{equation*}
\rho=\frac{\SNR}{1+\SNR}
\qquad\text{and}\qquad
\rho\ge\frac{\kappa}{(1+\eta)(1+\varepsilon)}.
\end{equation*}
Here $\varepsilon$ bounds sensitivity-weighted principal energy relative to complement energy, $\eta$ measures residual energy relative to a shared complement component, and $\kappa$ measures consistency of its amplitude. The quantity $\rho$ is the fraction of locally representable target energy explained by an optimal constant intervention, as defined in Definition~\ref{def:rho}; it is not a task-score gain recovery rate.

\subsection*{Informal preview}
Distillation into a frozen student locally amounts to fitting a hidden-state correction to the teacher's output distribution. Logit linearization and a quadratic KL expansion yield a weighted least-squares problem with context-dependent target $\bDelta(x)$ (Proposition~\ref{prop:wls}). An idealized KL-regularized RL optimum provides a separate account of how shared structure may arise: exponential reward tilting (Lemma~\ref{lem:tilt}) determines next-token probability ratios through continuation-success probabilities (Lemma~\ref{lem:next}). Shared success profiles, a small relative logit residual, and a consistent representation pullback then give a shared correction with controlled residual (Lemma~\ref{lem:profile}, Corollary~\ref{cor:fisherperp}, and Lemma~\ref{lem:pullback}). These are additional structural conditions, not consequences of reward binarity alone.
Under a constant metric, the optimal fixed vector is the mean target shift, and its recovery is high when that mean dominates context-dependent fluctuations (Theorem~\ref{thm:snr}). Table~\ref{tab:recovery_constants} summarizes the quantities entering the sufficient lower bound.

\begin{table}[h]
\centering\small
\caption{Geometric quantities entering the sufficient recovery bound. Their empirical counterparts require matching shift definitions, distributions, and metrics.}
\label{tab:recovery_constants}
\renewcommand{\arraystretch}{1.3}
\begin{tabularx}{\textwidth}{@{}clXX@{}}
\toprule
\textbf{Const.} & \textbf{Name} & \textbf{Meaning} & \textbf{Related analysis} \\
\midrule
$\varepsilon$ & principal intrusion & Upper bound on target principal-to-complement energy in the sensitivity metric & PC-energy measurements under compatible metrics and target definitions \\
$\eta$ & complement residual & Residual-to-shared energy within the complement & Context dependence and fixed-vector approximation \\
$\kappa$ & amplitude consistency & $1/(1+\CV(a)^2)$ & Consistency of the shared target amplitude \\
\bottomrule
\end{tabularx}
\end{table}

\noindent
Theorem~\ref{thm:bound} gives $\rho\ge\kappa/((1+\eta)(1+\varepsilon))$. The bound quantifies recovery under the stated localization and shared-component conditions. These conditions are motivated by the empirical findings; their validity and quantitative values for a particular teacher require a separate justification.

\subsection*{Part I.\quad Single-Vector Steering: Conditions for High Recovery}
\noindent\emph{Section~\ref{sec:setup} introduces the local setting. Sections~\ref{sec:reduction} and \ref{sec:tilt} derive the fitting problem and sufficient conditions for shared teacher structure. Section~\ref{sec:recovery} quantifies the optimal constant-vector fit.}

\subsection{Setup and notation}
\label{sec:setup}
We consider a frozen pretrained language model with decoder layers $\{f_\ell\}_{\ell=1}^{L}$, hidden width $d$, vocabulary size $V$, and unembedding matrix $U\in\R^{V\times d}$, which may be tied or untied. A \emph{context} $x\sim\mathcal D$ consists of a token position and its prefix under a specified distillation distribution $\mathcal D$. The base model has hidden states and logits
\begin{equation*}
\bh_\ell(x)\in\R^d,
\qquad
\bz_0(x)=U\bh_L(x)\in\R^V,
\end{equation*}
with $p_0(\cdot\mid x)=\softmax(\bz_0(x))$. Here $\bh_L(x)$ is the representation supplied to the unembedding, including final normalization. For an intervention before normalization, the downstream Jacobian includes that operation.
The teacher induces logits $\bz_T(x)$ and distribution $p_T(\cdot\mid x)$.
Vector steering adds a context-invariant offset $\bdelta_\ell\in\R^d$ at each controlled layer $\ell\in\mathcal S\subseteq\{1,\ldots,L\}$:
$\bh_\ell(x)\mapsto\bh_\ell(x)+\bdelta_\ell$, with all backbone weights frozen.
We first analyze a single injection site and extend to multiple sites in Remark~\ref{rem:sites}.

For a positive-semidefinite matrix $A$, write $\|v\|_A^2=v^\top Av$ and $\langle u,v\rangle_A=u^\top Av$. The notation $\|\cdot\|_2$ denotes the Euclidean vector norm or matrix spectral norm; $\lambda_i(\cdot)$ and $\sigma_i(\cdot)$ denote eigenvalues and singular values in descending order. All logarithms are natural and all vectors are columns. Bold symbols denote named vectors, while unbold subscripts such as $s_v$ denote scalar coordinates; matrices such as $J_x$, $F_x$, $G_x$, and $B_x$ use uppercase notation without requiring boldface. The symbol $\mathbf{1}_V$ denotes the all-ones vector in $\R^V$. Generic arguments such as $u$ and $v$ are understood from their stated spaces.
We analyze the expected forward-KL objective
\begin{equation}
\mathcal L(\bdelta)
=\E_{x\sim\mathcal D}
\KL\!\big(p_T(\cdot\mid x)\,\big\|\,p_{\bdelta}(\cdot\mid x)\big).
\label{eq:obj}
\end{equation}
The context distribution is fixed for each fitting problem. All expected quadratic energies used below are finite.

Let $\mathcal E\subseteq\R^d$ be a fixed working subspace with orthonormal basis $T\in\R^{d\times h}$. Interventions and target corrections lie in $\mathcal E$. Inverses are always taken for the compressed matrices, such as $T^\top G_xT$, assumed positive definite; positive definiteness on $\mathcal E$ does not require $G_x$ to be invertible on all of $\R^d$. Setting $T=I_d$ gives the full-space formulas when their inverses exist. For singular metrics, use a common identifiable subspace with a consistent identification of targets across contexts.
For results that decompose vectors using $P$ and $Q$, additionally assume that $\mathcal E$ is invariant under these projectors. This ensures that projected targets, means, and corrections remain in the space on which the stated metric bounds apply.

\paragraph{Objects and their dependence on context.}
\begin{itemize}[leftmargin=1.4em,itemsep=2pt]
\item \textbf{Context-dependent values and operators} include $\bh_\ell(x)$, $\bz_T(x)$, $\bt(x)$, $\bDelta(x)$, $a(x)$, $\bxi(x)$, and $J_x,F_x,G_x,B_x$.
\item \textbf{Context-invariant objects} include trainable vectors $\bdelta_\ell$ and $\bv_\ell^{(k)}$, shared directions $\bs$ and $\bw$, averaged metric $G$, mean target shift $\bmu$, projectors $P$, $Q$, and $\Pi_{\mathrm{vocab}}$, and the constants in the bounds. A trainable vector changes during optimization but is shared across contexts at each step.
\item \textbf{Distinct shift objects} are the learned intervention $\bdelta$, the teacher-induced fitting target $\bDelta(x)=B_x\bt(x)$, and the actual actor-to-base activation differences collected in $X_\ell^{(t)}$ in Part~III. They are not identified with one another.
\item \textbf{Sequence and training indices} are introduced separately: Section~\ref{sec:tilt} distinguishes prompts $q$ and complete responses $Y$ from contexts $x=(q,y_{<t})$; Part~III uses $t$ for a parameter-update index.
\end{itemize}
Thus, the local fitting problem approximates a context-dependent target family $\{\bDelta(x)\}$ by one constant vector in an output-sensitivity metric.

\begin{table}[h]
\caption{Recurring symbols for the distillation analysis. Context dependence is indicated by $(x)$ or a subscript $x$; Part~III separately introduces batched training quantities.}
\label{tab:notation}
\centering\small
\renewcommand{\arraystretch}{1.25}
\begin{tabularx}{\textwidth}{@{}llX@{}}
\toprule
\textbf{Symbol} & \textbf{Type} & \textbf{Meaning} \\
\midrule
$x\sim\mathcal D$ & context & Token position and prefix, $x=(q,y_{<t})$ \\
$q$, $Y$ & sequence level & Prompt and complete response, including termination \\
$\bh_\ell(x)\in\R^d$ & ctx.-dep. & Layer-$\ell$ hidden state \\
$\bz_0(x),\bz_T(x)\in\R^V$ & ctx.-dep. & Base and teacher logits \\
$\bt(x)=\bz_T(x)-\bz_0(x)$ & ctx.-dep. & Teacher logit shift modulo $\mathbf{1}_V$; Remark~\ref{rem:gauge} \\
$F_x\in\R^{V\times V}$ & ctx.-dep. & Fisher information at teacher logits; Equation~(\ref{eq:fisher}) \\
$J_x\in\R^{V\times d}$ & ctx.-dep. & Downstream logit Jacobian at the base injection state \\
$G_x=J_x^\top F_xJ_x$ & ctx.-dep. & Pullback sensitivity metric \\
$B_x=T(T^\top G_xT)^{-1}T^\top J_x^\top F_x$ & ctx.-dep. & Least-squares pullback into $\mathcal E$ \\
$\bDelta(x)=B_x\bt(x)$ & ctx.-dep. & Locally best target representation correction \\
$a(x)\in\R_{\ge0}$ & ctx.-dep. & Amplitude in the specified shared-component decomposition \\
$\bxi(x)\in\R^d$ & ctx.-dep. & Representation residual within the complement \\
$\nu_0(x)\in[0,1]$ & ctx.-dep. & Conditional success probability; Equation~(\ref{eq:ZS}) \\
$Z_\beta(x)$ & ctx.-dep. & Conditional reward-tilt normalizer \\
$b(x),a_0(x),\mathbf e(x)$ & ctx.-dep. & Baseline, modulation, and residual of a success profile \\
\midrule
$\bdelta_\ell,\bv_\ell^{(k)}\in\R^d$ & invariant & Learned constant steering vectors \\
$\bs\in\R^V$, $\bw\in\R^d$ & invariant & Shared logit and representation directions \\
$G=\E_x[G_x]$ & invariant & Averaged sensitivity metric \\
$\bmu=\E_x[\bDelta(x)]$ & invariant & Mean target; optimal constant vector when $G_x=G$ \\
$C,P,Q=I-P$ & invariant & Reference activation covariance and its spectral projectors \\
$\Pi_{\mathrm{vocab}}=I-\mathbf{1}_V\mathbf{1}_V^\top/V$ & invariant & Vocabulary-centering projector, distinct from $P$ \\
$c_\beta=e^{1/\beta}-1$ & invariant & Tilt strength for binary rewards \\
$\varepsilon,\eta,\kappa,\rho$ & invariant & Intrusion bound, residual ratio, amplitude consistency, and energy recovery \\
$\eta_{\mathrm{logit}},\theta,\zeta_B,\zeta_R$ & invariant & Relative logit-residual and pullback bounds \\
\bottomrule
\end{tabularx}
\end{table}

\paragraph{Assumptions.}
The local reduction uses Assumption~\ref{as:local}. Constant-metric identities use the exact case of Assumption~\ref{as:homog}. Assumption~\ref{as:reward} specifies shared teacher structure, while Assumptions~\ref{as:offprin}--\ref{as:shared} supply the additional geometry for the three-constant bound. Part~III's projection statements require only their own explicitly stated conditions.

\begin{assumption}[Local / trust-region regime]
\label{as:local}
Fix a teacher, base model, and context distribution. The KL approximation is used only for interventions in a specified neighborhood of the base state. Write
\begin{equation*}
J_x:=\left.\frac{\partial\bz(x)}{\partial\bh_\ell(x)}\right|_{\mathrm{base}}\in\R^{V\times d},
\qquad
\bz_{\bdelta}(x)=\bz_0(x)+J_x\bdelta+\mathbf e_z(x,\bdelta).
\end{equation*}
The Jacobian includes all downstream operations. Assume differentiability with
$\|\mathbf e_z(x,\bdelta)\|_2=o(\|\bdelta\|_2)$, and a bounded third derivative of the log-partition function along the segments between the teacher and intervened logits.
For $\mathbf d_\delta(x)=J_x\bdelta-\bt(x)$, also assume that the combined network-linearization and KL-expansion error is
$o(\|\mathbf d_\delta(x)\|_2^2)$ in the local regime under consideration.
A sufficient condition is $\mathbf d_\delta(x)\to0$ and
$\|\mathbf e_z(x,\bdelta)\|_2=o(\|\mathbf d_\delta(x)\|_2)$, with uniform or integrable bounds that justify taking expectations of the remainder.
Little-$o$ statements concern this joint local regime; differentiability in $\bdelta$ alone does not imply small relative error when $J_x\bdelta$ nearly cancels $\bt(x)$.
\end{assumption}

\begin{assumption}[Shared logit component]
\label{as:reward}
The teacher logit shift $\bt(x)=\bz_T(x)-\bz_0(x)\in\R^V$ admits
\begin{equation}
\bt(x)=\underbrace{a(x)\bs}_{\text{shared direction}}+\br_\perp(x),
\label{eq:as-reward}
\end{equation}
where $\bs$ is fixed, $\|\bs\|_2=1$, $\mathbf{1}_V^\top\bs=0$, and $a(x)\ge0$.
The residual obeys $\langle\bs,\br_\perp(x)\rangle_{F_x}=0$ at every context and
\begin{equation}
\E_x\|\br_\perp(x)\|_{F_x}^2
\le\eta_{\mathrm{logit}}\,\E_x\|a(x)\bs\|_{F_x}^2,
\qquad \eta_{\mathrm{logit}}\ll1.
\label{eq:etalogit}
\end{equation}
Remark~\ref{rem:gauge} justifies the centered gauge. Section~\ref{sec:tilt} provides sufficient success-profile conditions for this decomposition; it is not assumed to hold merely because rewards are binary.
\end{assumption}

\begin{assumption}[Metric homogeneity]
\label{as:homog}
Let $\bar G=\E_{x\sim\mathcal D}[G_x]$ and abbreviate $G=\bar G$, with $T^\top GT$ positive definite. Exact homogeneous-model identities assume $G_x=G$ as quadratic forms on $\mathcal E$.
When approximation is invoked, its quantitative form is
$(1-\zeta)G\preceq G_x\preceq(1+\zeta)G$ on $\mathcal E$, uniformly over the considered contexts, for $0\le\zeta<1$.
Remark~\ref{rem:approx} treats this perturbation for a fixed target family.
Consistency of $B_x\bs$ is a separate condition, quantified by $\zeta_B$ in Lemma~\ref{lem:pullback}; it does not follow from metric homogeneity.
\end{assumption}

\begin{assumption}[Off-principal localization]
\label{as:offprin}
Let $P$ be the Euclidean orthogonal projector onto the top-$m$ principal subspace of the reference activation covariance $C$ in Definition~\ref{def:CG}, and let $Q=I-P$. Here $m$ is a general retained rank; the monitoring construction in Section~\ref{sec:predicter} uses the particular rank $r=\lceil0.10d\rceil$.
Assume compatibility with the sensitivity metric, $P^\top GQ=0$ on the working subspace. Then
\begin{equation}
\E_x\Gnorm{\bDelta(x)}^2
=\E_x\Gnorm{P\bDelta(x)}^2+\E_x\Gnorm{Q\bDelta(x)}^2.
\label{eq:gsplit}
\end{equation}
Assume in addition the target-energy bound
\begin{equation*}
\E_x\Gnorm{P\bDelta(x)}^2
\le\varepsilon\,\E_x\Gnorm{Q\bDelta(x)}^2,
\qquad 0\le\varepsilon\ll1.
\end{equation*}
The operator condition gives a $G$-orthogonal split for every vector in $\mathcal E$; equality of averaged energies for one distribution is weaker. Neither covariance principal structure nor a low Euclidean PC fraction alone guarantees these sensitivity-weighted conditions. Remark~\ref{rem:approx} treats approximate compatibility.
\end{assumption}

\begin{assumption}[Shared direction in the complement]
\label{as:shared}
There is a fixed $\bw\in\mathcal E\cap\operatorname{range}(Q)$ with $\Gnorm{\bw}=1$ such that
\begin{equation}
Q\bDelta(x)=\underbrace{a(x)\bw}_{\text{shared complement component}}+\bxi(x),
\qquad a(x)\ge0,
\label{eq:as-shared}
\end{equation}
where $\bxi(x)\in\mathcal E\cap\operatorname{range}(Q)$ and
\begin{equation}
\E_x[\bxi(x)]=\mathbf0,
\qquad
\E_x\!\left[a(x)\langle\bw,\bxi(x)\rangle_G\right]=0.
\label{eq:GA1}
\end{equation}
Assume $0<\E_x[a(x)^2]<\infty$ and finite target energy, and define
\begin{equation*}
\kappa:=\frac{(\E_x[a(x)])^2}{\E_x[a(x)^2]}\in(0,1],
\qquad
\eta:=\frac{\E_x\Gnorm{\bxi(x)}^2}{\E_x[a(x)^2]}\ge0.
\end{equation*}
Then $\kappa=1$ exactly when $a(x)$ is constant almost surely, and
$\kappa=1/(1+\CV(a)^2)$ with $\CV(a)^2=\Var_x(a(x))/(\E_x[a(x)])^2$.
Nonnegativity fixes the shared amplitude's sign; $\eta$ is a residual-energy ratio, not a rank.
The two moment conditions in Equation~(\ref{eq:GA1}) have distinct roles: zero mean identifies the shared mean, while weighted orthogonality removes the mixed energy term. Pointwise orthogonality is sufficient for the latter but is not required.
The amplitude here belongs to the representation-space decomposition. It is identified with a fixed positive rescaling of the logit-space amplitude only under the compatible pullback conditions discussed in Remark~\ref{rem:pullbackcaveat}; otherwise the two decompositions use separately defined amplitudes.
\end{assumption}

\begin{remark}[Perturbations of the homogeneous model]
\label{rem:approx}
\emph{(i) Varying metrics.} Let
$\bar G_{\mathcal E}=T^\top\bar GT$ and $\mathbf b_{\mathcal E}=T^\top\E_x[G_x\bDelta(x)]$.
The quadratic optimum is $\bdelta^\star=T\bar G_{\mathcal E}^{-1}\mathbf b_{\mathcal E}$, with explained energy
$S=\mathbf b_{\mathcal E}^\top\bar G_{\mathcal E}^{-1}\mathbf b_{\mathcal E}$.
For $E_{\mathrm{var}}=\E_x\|\bDelta(x)\|_{G_x}^2>0$, its recovery is $\rho_{\mathrm{var}}=S/E_{\mathrm{var}}$.
For a fixed target family and the metric bounds in Assumption~\ref{as:homog}, comparison with constant-metric recovery $\rho_G$ yields
\begin{equation*}
\frac{1-\zeta}{1+\zeta}(1-\rho_G)
\le1-\rho_{\mathrm{var}}\le
\frac{1+\zeta}{1-\zeta}(1-\rho_G).
\end{equation*}
Indeed, the metric inequalities bound the residual objective for every admissible constant vector and therefore its infimum, as well as the total target energy. Comparing different target families requires an additional target-perturbation argument.

\emph{(ii) Approximately orthogonal splits.} Replace exact compatibility by
$|\langle u,v\rangle_G|\le\gamma\|u\|_G\|v\|_G$
for $u\in\mathcal E\cap\operatorname{range}(P)$ and $v\in\mathcal E\cap\operatorname{range}(Q)$, with $0\le\gamma<1$. Then
\begin{equation*}
(1-\gamma)(\|Pz\|_G^2+\|Qz\|_G^2)
\le\|z\|_G^2
\le(1+\gamma)(\|Pz\|_G^2+\|Qz\|_G^2).
\end{equation*}
This follows by expanding the cross term and using $2ab\le a^2+b^2$.
Under the remaining conditions of Theorem~\ref{thm:bound}, bounding both the mean energy and total target energy gives
$\rho\ge\frac{1-\gamma}{1+\gamma}\frac{\kappa}{(1+\eta)(1+\varepsilon)}$.

\emph{(iii) Distinct residual and metric conditions.} Corollary~\ref{cor:fisherperp} bounds the logit residual, and Lemma~\ref{lem:pullback} bounds a representation residual relative to a specified direction. Applying these estimates to $\eta$ additionally requires localization, zero residual mean, and weighted orthogonality as in Assumption~\ref{as:shared}, or explicit bounds on their deviations. Corollary~\ref{cor:psi} compares Euclidean and sensitivity-weighted intrusion only for the same collection of shifts.
\end{remark}

\subsection{Distillation reduces to weighted least squares}
\label{sec:reduction}
The categorical distribution has natural parameter $\bz\in\R^V$, log-partition function $A(\bz)=\log\sum_{v=1}^{V}e^{z_v}$, and mean map $\nabla A(\bz)=\softmax(\bz)$.
Writing $p=\softmax(\bz)$, its Hessian is
\begin{equation}
F(\bz)=\nabla^2A(\bz)=\diag(p)-pp^\top\succeq0.
\label{eq:fisher}
\end{equation}
Throughout, $F_x=F(\bz_T(x))$ is evaluated at teacher logits. It satisfies $F_x\mathbf{1}_V=\mathbf0$.

\begin{lemma}[KL as a Bregman divergence and its local quadratic form]
\label{lem:bregman}
For $\bz,\bz'\in\R^V$,
$\KL(\softmax(\bz')\,\|\,\softmax(\bz))=D_A(\bz\,\|\,\bz')$,
where $D_A$ is the Bregman divergence generated by $A$.
Under Assumption~\ref{as:local},
\begin{equation}
\KL\!\big(p_T(\cdot\mid x)\,\big\|\,p_{\bdelta}(\cdot\mid x)\big)
=\tfrac12\big(J_x\bdelta-\bt(x)\big)^\top F_x\big(J_x\bdelta-\bt(x)\big)
+o\big(\|J_x\bdelta-\bt(x)\|_2^2\big).
\label{eq:kl-quad}
\end{equation}
\end{lemma}
\begin{proof}
The identity follows by substituting $A$ and $\nabla A$ into $D_A$.
Set $d=J_x\bdelta-\bt(x)$ and $e=\mathbf e_z(x,\bdelta)$, so the actual logit difference from the teacher is $d+e$.
Taylor expansion at $\bz_T(x)$ gives
\begin{equation*}
\KL(p_T\|p_{\bdelta})
=\tfrac12 d^\top F_xd+d^\top F_xe+\tfrac12e^\top F_xe+R_A(d+e).
\end{equation*}
If $L_A$ bounds the third derivative of $A$ in trilinear operator norm along the logit segment, then
\begin{equation*}
\left|\KL(p_T\|p_{\bdelta})-\tfrac12d^\top F_xd\right|
\le\|F_x\|_2\|d\|_2\|e\|_2
+\tfrac12\|F_x\|_2\|e\|_2^2
+\tfrac{L_A}{6}\|d+e\|_2^3.
\end{equation*}
The combined remainder condition yields Equation~(\ref{eq:kl-quad}); the integrability condition permits the corresponding expected expansion.
\end{proof}

\begin{definition}[Representational metric and target shift]
\label{def:metric}
For each context, define
\begin{equation*}
\begin{gathered}
G_x:=J_x^\top F_xJ_x\in\R^{d\times d},\qquad
B_x:=T(T^\top G_xT)^{-1}T^\top J_x^\top F_x\in\R^{d\times V},\\[2pt]
\bDelta(x):=\argmin_{\bu\in\mathcal E}\|J_x\bu-\bt(x)\|_{F_x}^2=B_x\bt(x),
\end{gathered}
\end{equation*}
where $T^\top G_xT$ is positive definite. When $T=I_d$, $B_x=G_x^{-1}J_x^\top F_x$.
The metric measures local output sensitivity; $\bDelta(x)$ is the best linearized hidden-state correction for reproducing the teacher shift. This target is context-dependent even though the learned intervention is constant. It is a fitting target, not an observed teacher-minus-base hidden-state difference.
\end{definition}

\begin{proposition}[Distillation as local weighted least squares]
\label{prop:wls}
Under Assumption~\ref{as:local}, the local quadratic surrogate is
\begin{equation}
\mathcal L_{\mathrm{quad}}(\bdelta)
=\tfrac12\E_{x\sim\mathcal D}(\bdelta-\bDelta(x))^\top G_x(\bdelta-\bDelta(x))
+C_{\mathrm{irr}},
\label{eq:wls}
\end{equation}
where $C_{\mathrm{irr}}\ge0$ is independent of $\bdelta$ and is the residual teacher-logit energy not fitted by the local Jacobian on $\mathcal E$, measured in the Fisher seminorm. The KL objective differs by the expected remainder in Lemma~\ref{lem:bregman}.
\end{proposition}
\begin{proof}
Write $\bt(x)=J_x\bDelta(x)+\br_x$. The least-squares normal equations give $T^\top J_x^\top F_x\br_x=0$. Since $\bdelta-\bDelta(x)\in\mathcal E$,
\begin{equation*}
\|J_x\bdelta-\bt(x)\|_{F_x}^2
=\|J_x(\bdelta-\bDelta(x))\|_{F_x}^2+\|\br_x\|_{F_x}^2.
\end{equation*}
The first term is $\|\bdelta-\bDelta(x)\|_{G_x}^2$. Taking expectations proves Equation~(\ref{eq:wls}), with $C_{\mathrm{irr}}=\tfrac12\E_x\|\br_x\|_{F_x}^2$.
\end{proof}

\begin{remark}[Multiple injection sites]
\label{rem:sites}
Stack the steering vectors and use the joint Jacobian from all injection sites to the logits. Equation~(\ref{eq:wls}) holds with the full $G_x=J_x^\top F_xJ_x$, including cross-layer blocks; it is not in general a sum of independent layer-wise fitting problems. Injection at the unembedding input gives $J_x=U$, whereas earlier injection includes normalization and all other downstream operations.
\end{remark}

\paragraph{Local surrogate versus training dynamics.}
The reduction describes a fixed-distribution local objective. At a fixed context, forward and reverse KL have the same second-order Fisher term around the teacher, but generally differ beyond this local expansion. Changing on-policy prefix distributions or using an approximate token-level training gradient changes the optimization problem. Consequently, Equation~(\ref{eq:wls}) is not an identity for the complete OPD training trajectory or for every off-policy surrogate.
Moreover, the unconstrained optimum of the quadratic model need not remain in the neighborhood where that model is accurate. A small teacher--student output gap alone does not ensure this: the Jacobian, its conditioning, and the intervention size also matter. Section~\ref{sec:recovery} makes the distinction between unconstrained surrogate recovery and locally admissible recovery explicit.

\subsection{The RLVR teacher shift: exponential tilting and shared structure}
\label{sec:tilt}
We derive the complete-response optimum of an idealized KL-regularized RL objective and then its next-token shifts. The analysis separates sequence rewards, complete-response probability ratios, and token-level corrections. Shared structure requires conditions on continuation-success profiles and on the student's pullback.

\paragraph{Sequence notation and optimization setting.}
Let $q$ be a prompt, $Y$ a complete response including termination, and $x=(q,y_{<t})$ a next-token context. Responses terminate almost surely under the base policy. Rewards are bounded, $\beta>0$, and policies are supported on the base policy's support. The symbol $\pi^\star$ denotes the optimizer of this stated objective. Empirical teacher checkpoints instead enter the local fitting analysis through their actual logit shifts $\bt(x)$.

\begin{lemma}[KL-regularized RL optimum as an exponential tilt]
\label{lem:tilt}
For a fixed prompt $q$, the objective
\begin{equation*}
\max_\pi\ \E_{Y\sim\pi(\cdot\mid q)}[r(q,Y)]
-\beta\,\KL\big(\pi(\cdot\mid q)\,\big\|\,p_0(\cdot\mid q)\big),
\qquad\beta>0,
\end{equation*}
is strictly concave on its finite-KL domain and has the unique maximizer
\begin{equation}
\pi^\star(Y\mid q)=\frac{p_0(Y\mid q)e^{r(q,Y)/\beta}}{Z_\beta(q)},
\qquad
Z_\beta(q)=\E_{Y\sim p_0(\cdot\mid q)}[e^{r(q,Y)/\beta}].
\label{eq:pistar}
\end{equation}
Equivalently, $\log(\pi^\star(Y\mid q)/p_0(Y\mid q))=r(q,Y)/\beta-\log Z_\beta(q)$ on the support.
\end{lemma}
\begin{proof}
Bounded rewards make $Z_\beta(q)$ positive and finite, and the proposed optimizer has finite KL. Substitution rewrites the objective as
$\beta\log Z_\beta(q)-\beta\KL(\pi(\cdot\mid q)\|\pi^\star(\cdot\mid q))$.
Nonnegativity of KL gives the unique maximum. Strict concavity follows from strict convexity of KL and linearity of the reward term.
\end{proof}

\begin{lemma}[From complete responses to next-token probabilities]
\label{lem:next}
For $x=(q,y_{<t})$ with positive base prefix probability, define
\begin{equation}
Z_\beta(x):=\E_{Y\sim p_0(\cdot\mid x)}[e^{r(q,Y)/\beta}],
\qquad
\nu_0(x):=\Pr_{Y\sim p_0(\cdot\mid x)}(r(q,Y)=1),
\label{eq:ZS}
\end{equation}
where complete responses extend the prefix. For a supported next token $v$, write $xv$ for the extended context. Then
\begin{equation}
\pi^\star(v\mid x)=p_0(v\mid x)\frac{Z_\beta(xv)}{Z_\beta(x)}.
\label{eq:tiltnext}
\end{equation}
For binary rewards $r\in\{0,1\}$ and $c_\beta=e^{1/\beta}-1>0$,
\begin{equation}
\log\frac{\pi^\star(v\mid x)}{p_0(v\mid x)}
=\log[1+c_\beta\nu_0(xv)]-\log[1+c_\beta\nu_0(x)].
\label{eq:logratio}
\end{equation}
At termination the conditional quantities are evaluated on the resulting complete response.
\end{lemma}
\begin{proof}
Summing Equation~(\ref{eq:pistar}) over responses extending a prefix gives
\begin{equation*}
\pi^\star(y_{<t}\mid q)
=p_0(y_{<t}\mid q)\frac{Z_\beta(x)}{Z_\beta(q)}.
\end{equation*}
Taking the ratio for $y_{<t}v$ and $y_{<t}$ proves Equation~(\ref{eq:tiltnext}). For binary rewards,
$Z_\beta(x)=1-\nu_0(x)+e^{1/\beta}\nu_0(x)=1+c_\beta\nu_0(x)$,
which gives Equation~(\ref{eq:logratio}).
\end{proof}

Thus, relative next-token changes depend on candidate-token continuation-success probabilities, not just the terminal reward of an already observed response. The tower property gives
$\nu_0(x)=\sum_vp_0(v\mid x)\nu_0(xv)$.
Neither this identity nor binary rewards imposes a common direction across different contexts.

\begin{remark}[Softmax gauge and centered logits]
\label{rem:gauge}
Since $\softmax(\bz+c\mathbf{1}_V)=\softmax(\bz)$, logits are defined modulo the all-ones direction. Choose the centered gauge $\Pi_{\mathrm{vocab}}=I-\mathbf{1}_V\mathbf{1}_V^\top/V$. Because $F_x\mathbf{1}_V=\mathbf0$,
\begin{equation}
F_x\Pi_{\mathrm{vocab}}=F_x,
\qquad
\|\Pi_{\mathrm{vocab}}\bu\|_{F_x}=\|\bu\|_{F_x},
\qquad
B_x\Pi_{\mathrm{vocab}}\bu=B_x\bu
\quad\text{for all }\bu\in\R^V.
\label{eq:gauge}
\end{equation}
Centering preserves Fisher seminorms, Fisher inner products, and $\bDelta(x)$. For the idealized teacher with finite logits, we may use
$\bt(x)=\Pi_{\mathrm{vocab}}\log(\pi^\star(\cdot\mid x)/p_0(\cdot\mid x))$
and choose $\mathbf{1}_V^\top\bs=0$. The vocabulary-centering projector $\Pi_{\mathrm{vocab}}$ is distinct from the activation-principal projector $P$.
\end{remark}

\begin{lemma}[A sufficient condition for a shared logit component]
\label{lem:profile}
Let $\bs\in\R^V$ be fixed, $\|\bs\|_2=1$, and $\mathbf{1}_V^\top\bs=0$. Suppose the continuation-success profiles obey
\begin{equation}
\nu_0(xv)=b(x)+a_0(x)s_v+e_v(x),
\qquad b(x)\in[0,1],\quad a_0(x)\ge0,\quad\nu_0(xv)\in[0,1],
\label{eq:profile}
\end{equation}
where $b(x)$ is token-independent, $s_v$ is the $v$-th scalar coordinate of $\bs$, and $\mathbf e(x)$ is a residual profile with coordinates $e_v(x)$. Define $\mathbf d(x)=a_0(x)\bs+\mathbf e(x)$. Then
\begin{equation}
\Pi_{\mathrm{vocab}}\log\frac{\pi^\star(\cdot\mid x)}{p_0(\cdot\mid x)}
=\widehat a(x)\bs+\widehat{\br}(x),
\qquad
\widehat a(x)=\frac{c_\beta a_0(x)}{1+c_\beta b(x)}\ge0,
\label{eq:sharedlogit}
\end{equation}
with coordinatewise logarithms and ratios, and
\begin{equation}
\|\widehat{\br}(x)\|_2
\le c_\beta\|\mathbf e(x)\|_2
+\frac{c_\beta^2}{2}\|\mathbf d(x)\|_4^2,
\qquad
\|\mathbf d(x)\|_4^2=\Big(\textstyle\sum_v|d_v(x)|^4\Big)^{1/2}.
\label{eq:rembound}
\end{equation}
\end{lemma}
\begin{proof}
For $f(z)=\log(1+c_\beta z)$ on $[0,1]$, $|f'(z)|\le c_\beta$ and $|f''(z)|\le c_\beta^2$.
Taylor expansion between $b(x)$ and $b(x)+d_v(x)$, both in that interval, yields
\begin{equation*}
f(b(x)+d_v(x))=f(b(x))+f'(b(x))d_v(x)+R_v(x),
\qquad |R_v(x)|\le\tfrac12c_\beta^2|d_v(x)|^2.
\end{equation*}
Centering Equation~(\ref{eq:logratio}) removes the token-independent terms. Since $\Pi_{\mathrm{vocab}}\bs=\bs$,
$\widehat{\br}(x)=\Pi_{\mathrm{vocab}}(f'(b(x))\mathbf e(x)+R(x))$.
The projector is a Euclidean contraction, and
$\|R(x)\|_2\le\tfrac12c_\beta^2\|\mathbf d(x)\|_4^2$.
These inequalities prove Equation~(\ref{eq:rembound}).
\end{proof}

The bound is absolute. Obtaining high relative alignment additionally requires a residual small compared with $\widehat a(x)\|\bs\|_{F_x}$, as specified next; a small absolute residual alone is insufficient when the shared amplitude is small.

\begin{corollary}[Fisher-orthogonal shared-component decomposition]
\label{cor:fisherperp}
Assume Lemma~\ref{lem:profile}, $\|\bs\|_{F_x}>0$, and $0\le\theta<1$ such that
\begin{equation}
\|\widehat{\br}(x)\|_{F_x}
\le\theta\widehat a(x)\|\bs\|_{F_x}
\quad\text{at every considered context}.
\label{eq:theta}
\end{equation}
For $\bt(x)=\widehat a(x)\bs+\widehat{\br}(x)$, define
\begin{equation}
a(x):=\frac{\langle\bs,\bt(x)\rangle_{F_x}}{\|\bs\|_{F_x}^2},
\qquad
\br_\perp(x):=\bt(x)-a(x)\bs.
\label{eq:aperp}
\end{equation}
Then $a(x)\ge(1-\theta)\widehat a(x)\ge0$, $\langle\bs,\br_\perp(x)\rangle_{F_x}=0$, and
\begin{equation}
\E_x\|\br_\perp(x)\|_{F_x}^2
\le\frac{\theta^2}{(1-\theta)^2}\,\E_x\|a(x)\bs\|_{F_x}^2.
\label{eq:relenergy}
\end{equation}
This meets Assumption~\ref{as:reward}'s small-residual requirement when $\theta$ is sufficiently small.
\end{corollary}
\begin{proof}
By Equation~(\ref{eq:aperp}),
$a(x)=\widehat a(x)+\langle\bs,\widehat{\br}(x)\rangle_{F_x}/\|\bs\|_{F_x}^2$.
Cauchy--Schwarz bounds the second term in absolute value by $\theta\widehat a(x)$.
The vector $a(x)\bs$ minimizes the Fisher squared distance to $\bt(x)$ over the one-dimensional span, so
\begin{equation*}
\|\br_\perp(x)\|_{F_x}
\le\|\bt(x)-\widehat a(x)\bs\|_{F_x}
\le\theta\widehat a(x)\|\bs\|_{F_x}
\le\frac{\theta}{1-\theta}a(x)\|\bs\|_{F_x}.
\end{equation*}
Squaring and averaging proves the bound. Orthogonality follows directly from the definition, and Remark~\ref{rem:gauge} ensures gauge invariance.
\end{proof}

\begin{remark}[Small positive $\beta$ and uniform residual bounds]
\label{rem:smallbeta}
A multiplicative success profile provides a bound without diverging constants as $\beta\downarrow0$, under an additional uniform success condition.
Suppose $\nu_0(xv)\ge\nu_{\min}>0$, $c_\beta\nu_{\min}\ge1/\varrho$ for $0<\varrho\le1$, and
\begin{equation}
\log\nu_0(xv)=\tilde b(x)+a_0(x)s_v+\tilde e_v(x),
\qquad a_0(x)\ge0.
\label{eq:logprofile}
\end{equation}
Using $\log(1+c_\beta z)=\log c_\beta+\log z+\log(1+1/(c_\beta z))$ gives
\begin{equation*}
\Pi_{\mathrm{vocab}}\log\frac{\pi^\star(\cdot\mid x)}{p_0(\cdot\mid x)}
=a_0(x)\bs+\Pi_{\mathrm{vocab}}\tilde{\mathbf e}(x)+\boldsymbol\epsilon(x),
\qquad
\|\boldsymbol\epsilon(x)\|_2\le\sqrt V\,\varrho.
\end{equation*}
Indeed, each last logarithm lies in $[0,\varrho]$, and centering is a contraction.
Corollary~\ref{cor:fisherperp} applies only if its relative Fisher-residual condition also holds. The positive lower bound on success probabilities excludes zero-success continuations and is not automatic for binary rewards.
This regime concerns a small positive penalty coefficient in an idealized optimum, not a small realized KL divergence or a zero-penalty checkpoint.
\end{remark}

\begin{lemma}[Consistency of the representation pullback]
\label{lem:pullback}
Retain Definition~\ref{def:metric} and Assumption~\ref{as:reward}. For fixed nonzero $\bw_0\in\mathcal E$, define
$\mathbf e_h(x)=\bDelta(x)-a(x)\bw_0$. Then
\begin{equation}
\mathbf e_h(x)=a(x)(B_x\bs-\bw_0)+B_x\br_\perp(x).
\label{eq:eh}
\end{equation}
Let $G$ be a fixed positive definite metric on $\mathcal E$, set $E_0=\E_x[a(x)^2]\Gnorm{\bw_0}^2>0$, and suppose
\begin{equation}
\E_x\!\left[a(x)^2\Gnorm{B_x\bs-\bw_0}^2\right]\le\zeta_BE_0,
\qquad
\E_x\Gnorm{B_x\br_\perp(x)}^2\le\zeta_RE_0
\label{eq:zeta}
\end{equation}
for $\zeta_B,\zeta_R\ge0$. Then
\begin{equation}
\E_x\Gnorm{\mathbf e_h(x)}^2\le2(\zeta_B+\zeta_R)E_0.
\label{eq:ehbound}
\end{equation}
The constants separately control variation of the shared pullback and the pulled-back residual energy.
\end{lemma}
\begin{proof}
Substitute $\bt(x)=a(x)\bs+\br_\perp(x)$ into $\bDelta(x)=B_x\bt(x)$ and subtract $a(x)\bw_0$ to obtain Equation~(\ref{eq:eh}). Applying
$\Gnorm{u+v}^2\le2\Gnorm{u}^2+2\Gnorm{v}^2$
and taking expectations proves Equation~(\ref{eq:ehbound}).
\end{proof}

Under exact homogeneity, define $A_x=F_x^{1/2}J_xT$ and
$\mathcal P_x=A_x(A_x^\top A_x)^{-1}A_x^\top$.
This is an orthogonal projector, so
\begin{equation}
\Gnorm{B_x\br_\perp(x)}^2
=(F_x^{1/2}\br_\perp(x))^\top\mathcal P_x(F_x^{1/2}\br_\perp(x))
\le\|\br_\perp(x)\|_{F_x}^2.
\label{eq:proj}
\end{equation}
Consequently, the smallest admissible residual ratio in Equation~(\ref{eq:zeta}) is at most
$\eta_{\mathrm{logit}}\E_x\|a(x)\bs\|_{F_x}^2/E_0$.

\begin{remark}[Shared-direction consistency under homogeneous metrics]
\label{rem:pullbackcaveat}
Equation~(\ref{eq:proj}) controls the pulled-back residual, not variation of $B_x\bs$. For example, replacing $J_x$ by $-J_x$ preserves $G_x$ but reverses the pullback, so equal metrics need not yield a common correction direction.
For a consistent nonzero $\bw_0$, normalize
$\bw=\bw_0/\Gnorm{\bw_0}$ and $a_h(x)=a(x)\Gnorm{\bw_0}$.
To use this component in Assumption~\ref{as:shared}, one must additionally verify $Q\bw=\bw$, zero residual mean, and weighted orthogonality. The bound on $\mathbf e_h(x)$ alone does not supply these moment conditions.
\end{remark}

\paragraph{From reward structure to recovery.}
The argument gives a sequence of sufficient conditions: an exact regularized optimum gives a next-token tilt relation; a shared success profile and a small relative residual give a shared logit component; a consistent pullback gives a shared representation component. The final lower bound separately requires off-principal localization and the moment conditions of Assumptions~\ref{as:offprin}--\ref{as:shared}. Each condition concerns a different object and must be checked independently when applying the analysis.

\paragraph{Connection to practical RL teachers.}
The tilt identities concern the exact optimum with $\beta>0$. For trained checkpoints, including DAPO checkpoints with zero explicit KL coefficient, the fitting and recovery results take $\bt(x)$ as given and apply under their stated local and structural conditions. Clipping and small updates motivate inspecting a local regime; they do not imply an exact tilt identity, metric homogeneity, or a shared direction for a trained checkpoint.

\subsection{Single-vector recovery: an identity and a lower bound}
\label{sec:recovery}
The quadratic optimum is generally metric-weighted. Under exact homogeneity, its explained target energy has the signal-to-fluctuation form in Section~\ref{sub:snr} and the sufficient geometric bound in Section~\ref{sub:bound}.

\begin{proposition}[Optimal steering vector]
\label{prop:opt}
Assume finite target energy and finite $C_{\mathrm{irr}}$. If $T^\top\bar GT$ is positive definite, where $\bar G=\E_x[G_x]$, the unique minimizer of the quadratic surrogate over $\mathcal E$ is
$\bdelta^\star=T(T^\top\bar GT)^{-1}T^\top\E_x[G_x\bDelta(x)]$.
For $T=I_d$, it is $\bar G^{-1}\E_x[G_x\bDelta(x)]$.
Under exact homogeneity,
\begin{equation*}
\bdelta^\star=\bmu:=\E_{x\sim\mathcal D}[\bDelta(x)].
\end{equation*}
This optimizes the quadratic surrogate without a trust-region constraint, not necessarily the original nonlinear KL objective.
\end{proposition}
\begin{proof}
Writing $\bdelta=Tz$, the normal equation is
$T^\top\bar GTz=T^\top\E_x[G_x\bDelta(x)]$.
Positive definiteness gives the unique solution. If $G_x=G$ on $\mathcal E$, $\E_x[\bDelta(x)]\in\mathcal E$ solves the equation.
\end{proof}

\begin{definition}[Recovery ratio]
\label{def:rho}
In the constant-metric quadratic model, let
$E_{\mathrm{tar}}=\E_x\Gnorm{\bDelta(x)}^2\in(0,\infty)$ and define
\begin{equation*}
\rho:=1-\frac{\E_x\Gnorm{\bDelta(x)-\bdelta^\star}^2}{E_{\mathrm{tar}}}\in[0,1].
\end{equation*}
This is the fraction of locally representable target energy explained by the optimal constant vector, excluding $C_{\mathrm{irr}}$. In particular,
\begin{equation*}
\frac{\mathcal L_{\mathrm{quad}}(0)-\mathcal L_{\mathrm{quad}}(\bdelta^\star)}
{\mathcal L_{\mathrm{quad}}(0)}
=\frac{\rho E_{\mathrm{tar}}}{E_{\mathrm{tar}}+2C_{\mathrm{irr}}}.
\end{equation*}
Thus, even perfect target-energy recovery need not remove irreducible logit error. Exact teacher reproduction additionally requires local reachability and negligible approximation error.
Neither $\rho$ nor the displayed surrogate-loss ratio equals an empirical task-score gain recovery rate without further assumptions connecting distributions to task performance. Remark~\ref{rem:approx}(i) gives the varying-metric counterpart.
\end{definition}

\paragraph{Locally admissible interventions.}
Let $\mathcal T\subseteq\mathcal E$ be a specified trust region containing zero. Its best constant-metric target-energy recovery is
\begin{equation*}
\rho_{\mathcal T}
:=1-\frac{\inf_{\bdelta\in\mathcal T}\E_x\Gnorm{\bdelta-\bDelta(x)}^2}{E_{\mathrm{tar}}},
\qquad 0\le\rho_{\mathcal T}\le\rho.
\end{equation*}
The inequalities follow because zero is feasible and $\mathcal T\subseteq\mathcal E$.
If $\bmu\in\mathcal T$, the two recovery values coincide. Otherwise, unconstrained surrogate recovery is only an upper bound on what this trust region permits. Relating either value to nonlinear KL additionally requires the remainder bounds to hold at the evaluated interventions.

\subsubsection{Signal-to-fluctuation form}
\label{sub:snr}
\begin{theorem}[Recovery as a signal-to-fluctuation ratio]
\label{thm:snr}
In the constant-metric quadratic model, let $\tilde{\bDelta}(x)=\bDelta(x)-\bmu$, with $\E_x[\tilde{\bDelta}(x)]=\mathbf0$ and positive finite total target energy. Then
\begin{equation}
\boxed{\begin{aligned}
\rho&=\frac{\Gnorm{\bmu}^2}
{\Gnorm{\bmu}^2+\E_x\Gnorm{\tilde{\bDelta}(x)}^2}
=\frac{\SNR}{1+\SNR},\\[2pt]
\SNR&:=\frac{\Gnorm{\bmu}^2}{\E_x\Gnorm{\tilde{\bDelta}(x)}^2}.
\end{aligned}}
\label{eq:snr}
\end{equation}
If the fluctuation energy vanishes, set $\SNR=+\infty$ and $\rho=1$. For any target level $u\in(0,1)$, $\rho\ge u$ if and only if $\SNR\ge u/(1-u)$, using this convention.
\end{theorem}
\begin{proof}
Proposition~\ref{prop:opt} gives $\bdelta^\star=\bmu$. Since $G$ is constant and the fluctuation has zero mean,
$\E_x\Gnorm{\bDelta(x)}^2=\Gnorm{\bmu}^2+\E_x\Gnorm{\tilde{\bDelta}(x)}^2$.
Substitution into Definition~\ref{def:rho} proves Equation~(\ref{eq:snr}).
If the fluctuation energy is zero, positive total energy makes the mean energy positive. The threshold relation follows by rearranging $\rho=\SNR/(1+\SNR)$.
\end{proof}

\begin{corollary}[Spectral upper bound]
\label{cor:spectral}
In the same model, let $\bu(x)=G^{1/2}\bDelta(x)$,
$\mathbf m=\E_x[\bu(x)]=G^{1/2}\bmu$, and $\Sigma=\E_x[\bu(x)\bu(x)^\top]$.
For a weighted data operator or factor $M$ with $\Sigma=MM^\top$,
\begin{equation*}
\rho=\frac{\|\mathbf m\|_2^2}{\tr\Sigma}
\le\frac{\lambda_1(\Sigma)}{\tr\Sigma}
=\frac{\sigma_1^2(M)}{\sum_i\sigma_i^2(M)}.
\end{equation*}
For $\mathbf m\ne0$, equality holds exactly when $\mathbf m$ lies in the leading eigenspace of $\Sigma$ and
$\mathbf m^\top\operatorname{Cov}(\bu(x))\mathbf m=0$.
For $\mathbf m=0$ and $\tr\Sigma>0$, the inequality is strict.
\end{corollary}
\begin{proof}
Since $\Sigma=\mathbf m\mathbf m^\top+\operatorname{Cov}(\bu(x))$, one has $\mathbf m\mathbf m^\top\preceq\Sigma$.
For $v=\mathbf m/\|\mathbf m\|_2$,
\begin{equation*}
\lambda_1(\Sigma)\ge v^\top\Sigma v
=\|\mathbf m\|_2^2+v^\top\operatorname{Cov}(\bu(x))v
\ge\|\mathbf m\|_2^2.
\end{equation*}
Equality requires equality at both steps. When $\mathbf m=0$, positive trace implies $\lambda_1(\Sigma)>0$. Finally, $\tr\Sigma=E_{\mathrm{tar}}$ and the factorization gives the singular-value expression.
\end{proof}

\paragraph{Implications for capacity, gap, and depth.}
The leading spectral share is an upper bound, not generally the recovery of a constant vector. Collinear targets can have unit leading spectral share but low mean-energy recovery when their amplitudes vary or their signs cancel.
Uniformly scaling a fixed target family by a nonzero scalar leaves $\rho$ unchanged in a fixed metric. Thus, target magnitude or teacher--student gap alone does not impose a monotone capacity requirement. Injection depth can change the Jacobian, reachability, sensitivity metric, and target variation; the theorem does not assert a universal monotone relationship with depth. The observed gap and depth trends are empirical findings consistent with changes in these quantities, not direct consequences of the recovery identity.

\subsubsection{Explicit three-constant bound}
\label{sub:bound}
Assumptions~\ref{as:offprin}--\ref{as:shared} specify a sufficient favorable regime through principal energy, complement residual energy, and amplitude consistency.

\begin{theorem}[Off-principal single-direction bound]
\label{thm:bound}
In the constant-metric quadratic model, under Assumptions~\ref{as:offprin}--\ref{as:shared} and their working-subspace compatibility conditions,
\begin{equation}
\boxed{\rho\ge\frac{\kappa}{(1+\eta)(1+\varepsilon)}.}
\label{eq:bound}
\end{equation}
\end{theorem}

\paragraph{Energy decomposition.}
Both moment conditions in Equation~(\ref{eq:GA1}) are used. In particular, weighted orthogonality removes the mixed term in
\begin{equation}
\E_x\Gnorm{a(x)\bw+\bxi(x)}^2
=\E_x[a(x)^2]\Gnorm{\bw}^2+\E_x\Gnorm{\bxi(x)}^2
+2\E_x[a(x)\langle\bw,\bxi(x)\rangle_G].
\label{eq:GA2}
\end{equation}
Zero residual mean identifies $Q\bmu=\E_x[a(x)]\bw$. Remark~\ref{rem:approx} supplies a perturbation bound when the principal/complement split is only approximately $G$-orthogonal.

\begin{proof}
By Theorem~\ref{thm:snr}, $\rho=\Gnorm{\bmu}^2/E_{\mathrm{tar}}$.
The zero-mean residual gives $Q\bmu=\E_x[a(x)]\bw$. Metric compatibility and $\Gnorm{\bw}=1$ imply
\begin{equation*}
\Gnorm{\bmu}^2\ge\Gnorm{Q\bmu}^2
=(\E_x[a(x)])^2=\kappa\E_x[a(x)^2].
\end{equation*}
The energy split, localization bound, and weighted orthogonality give
\begin{equation*}
E_{\mathrm{tar}}
\le(1+\varepsilon)\E_x\Gnorm{Q\bDelta(x)}^2
=(1+\varepsilon)(1+\eta)\E_x[a(x)^2].
\end{equation*}
The last factor is positive; division proves the result.
\end{proof}

\paragraph{What the bound establishes.}
Write $E_P=\E_x\Gnorm{P\bDelta(x)}^2$ and $E_Q=\E_x\Gnorm{Q\bDelta(x)}^2>0$.
The proof gives complement recovery
$\rho_Q=\Gnorm{Q\bmu}^2/E_Q=\kappa/(1+\eta)$.
If $E_P>0$, define $\rho_P=\Gnorm{P\bmu}^2/E_P$ and $e=E_P/E_Q$. Then the exact relation is
\begin{equation*}
\rho=\frac{\rho_Q+e\rho_P}{1+e}.
\end{equation*}
If $E_P=0$, $\rho=\rho_Q$. The theorem discards a nonnegative principal contribution and uses $e\le\varepsilon$.
Accordingly, low principal energy is a sufficient condition within this decomposition, not necessary for high constant-vector recovery: a constant nonzero target entirely in the principal subspace also has $\rho=1$. The theorem quantifies a specified localization pattern; it does not derive that pattern from the RLVR objective or establish its prevalence.

\begin{corollary}[Coefficient-of-variation form]
\label{cor:cv}
Under the same conditions,
\begin{equation*}
\rho\ge\frac{1}{(1+\CV(a)^2)(1+\eta)(1+\varepsilon)},
\end{equation*}
since $\kappa=1/(1+\CV(a)^2)$.
\end{corollary}

\subsection*{Part II.\quad Multi-Vector Steering: The Geometry of the Residual Subspace}
% The legacy Part II label is attached to the numbered subsection below.
\noindent\emph{Section~\ref{sec:multi} extends constant-vector fitting to subspace constraints and distinguishes it from adaptive subspace approximation. Section~\ref{sec:complement} relates principal structure to sensitivity and weight-induced activation changes.}

\subsection{Sequential-orthogonal steering and constrained fitting}
\label{sec:multi}
\label{Multi-Vector Steering: The Geometry of the Residual Subspace}
The empirical protocol learns $\bv^{(0)},\ldots,\bv^{(K-1)}$ sequentially. At each stage only the current multi-layer direction is injected; earlier directions define layer-wise Euclidean orthogonality constraints. Reprojection enforces these constraints after each optimizer step, and each direction is evaluated independently. The result below concerns exact fitting for a fixed target distribution and metric, not cumulative injection of the learned vectors.

\begin{theorem}[Constant-vector recovery under a subspace constraint]
\label{thm:svd}
% Legacy label retained for cross-reference compatibility.
Let $\mathcal V_k\subseteq\mathcal E$ be a feasible subspace in the constant-metric quadratic model, and let $R_k$ have full-column-rank columns spanning it. The exact constrained minimizer is
\begin{equation*}
\bv^{(k)}=R_k(R_k^\top GR_k)^{-1}R_k^\top G\bmu.
\end{equation*}
For the single-site sequential protocol,
$\mathcal V_k=\mathcal E\cap\operatorname{span}\{\bv^{(j)}:j<k\}^{\perp_2}$.
With $E_{\mathrm{tar}}>0$, its standalone target-energy recovery satisfies
\begin{equation}
\rho^{(k)}=\frac{\Gnorm{\bv^{(k)}}^2}{E_{\mathrm{tar}}},
\qquad
\mathcal V_{k+1}\subseteq\mathcal V_k\ \Longrightarrow\ \rho^{(k+1)}\le\rho^{(k)}.
\label{eq:rhoK}
\end{equation}
If $G=cI$ on $\mathcal E$, $c>0$, $\mathcal V_0=\mathcal E$, and $\bmu\ne0$, exact first-stage fitting gives $\bv^{(0)}=\bmu$, and fitting in its Euclidean orthogonal complement gives $\bv^{(1)}=\mathbf0$.
\end{theorem}
\begin{proof}
The parallel-axis identity splits the fitting objective into a constant fluctuation term and $\|\bv-\bmu\|_G^2$.
Writing $\bv=R_kz$ gives $R_k^\top G(R_kz-\bmu)=0$ and the stated solution.
The residual is $G$-orthogonal to $\mathcal V_k$, hence
$\langle\bv^{(k)},\bmu\rangle_G=\|\bv^{(k)}\|_G^2$.
Expanding the loss reduction yields Equation~(\ref{eq:rhoK}). Nesting the feasible subspaces cannot improve the optimum. In the isotropic case the second space is orthogonal to $\bmu$, whose projection is zero.
If the feasible space has dimension zero, its unique minimizer is zero without applying the matrix-inverse formula. A zero vector introduces no further orthogonality constraint.
\end{proof}

For multiple sites, intersect the product of the layer-wise feasible spaces with the working subspace. The theorem applies to the stacked intervention and the full sensitivity metric, including cross-layer blocks.

\paragraph{Adaptive subspace approximation.}
For comparison, consider a shared subspace with unrestricted context-dependent coefficients. Let $1\le K\le\dim\mathcal E$ and $\bu(x)=G^{1/2}\bDelta(x)$, with $\Sigma=\E_x[\bu(x)\bu(x)^\top]$. Then
\begin{equation}
\min_{\substack{Q_K^\top Q_K=I_K\\\operatorname{range}(Q_K)\subseteq G^{1/2}\mathcal E}}
\E_x\|\bu(x)-Q_KQ_K^\top\bu(x)\|_2^2
=\tr\Sigma-\sum_{i=1}^{K}\lambda_i(\Sigma).
\label{eq:adaptive_subspace_recovery}
\end{equation}
Indeed, the captured energy is $\tr(Q_K^\top\Sigma Q_K)$, maximized by a leading eigenspace. The fitted coefficients $Q_K^\top\bu(x)$ depend on context. Thus, the optimal explained fraction is $\sum_{i=1}^{K}\lambda_i(\Sigma)/\tr\Sigma$ when the trace is positive.
These are oracle coefficients for the quadratic target. A particular gated network must also represent and learn the required coefficient map, and layer-wise restrictions may further reduce its attainable fit. The spectral expression is therefore a characterization of a more flexible approximation class, not a guarantee for a trained gating architecture.

\paragraph{Interpretation of later steering directions.}
Nonzero standalone performance of later Euclidean-orthogonal vectors establishes the existence of alternative useful interventions. It does not establish that they recover distinct spectral components of the teacher target.
Even a deterministic target can admit more than one useful Euclidean-orthogonal intervention under an anisotropic metric. For example, with $\mathcal E=\R^2$, $\bDelta(x)=(1,0)^\top$, and
$G=\left(\begin{smallmatrix}2&1\\1&1\end{smallmatrix}\right)$,
Theorem~\ref{thm:svd} gives $\bv^{(0)}=(1,0)^\top$ and $\bv^{(1)}=(0,1)^\top$, with standalone recovery $1$ and $1/2$, respectively, although the target second moment has rank one.
In the empirical setting, anisotropy, nonlinear responses, evolving sampled contexts, and approximate optimization can all affect the learned directions; task scores are also distinct from quadratic energy recovery. The observed decline with extraction order and its model-scale dependence are consequently descriptions of independently effective interventions, not estimates of target spectral rank or intrinsic manifold dimension.

\subsection{Effective directions and the principal subspace}
\label{sec:complement}
The main-text steering measurements place approximately $1\%$--$2\%$ of the learned vectors' Euclidean energy in a top-$10\%$ activation principal subspace, below the approximately $10\%$ isotropic reference. This observation concerns the learned vectors. The recovery bound instead uses sensitivity-weighted energy of the context-dependent target family. We now state the metric comparison needed to connect such quantities when the shift objects are matched.

\begin{definition}[Variance and sensitivity metrics]
\label{def:CG}
For a specified base-model reference context distribution $\mathcal D_{\mathrm{ref}}$, let
$\bar{\bh}_\ell=\E_{x\sim\mathcal D_{\mathrm{ref}}}[\bh_\ell(x)]$ and define
\begin{equation*}
C:=\E_{x\sim\mathcal D_{\mathrm{ref}}}[(\bh_\ell(x)-\bar{\bh}_\ell)(\bh_\ell(x)-\bar{\bh}_\ell)^\top],
\qquad G:=\E_{x\sim\mathcal D}[G_x].
\end{equation*}
The leading eigenspace of $C$ defines $P$, with $Q=I-P$; fix a choice of basis if the cutoff eigenvalue is degenerate.
The matrix $G$ measures output sensitivity, and at the unembedding input equals $U^\top\E_x[F_x]U$.
The reference distribution used to estimate covariance and the distribution used for fitting must be specified; they need not be identical.
An uncentered implementation uses $C+\bar{\bh}_\ell\bar{\bh}_\ell^\top$ instead and must report its basis accordingly. High activation variance is not equivalent to high or low output sensitivity.
\end{definition}

\begin{proposition}[Sensitivity-weighted principal-energy bound]
\label{prop:mismatch}
On the working subspace, define
\begin{equation*}
g_P=\sup_{v\in\mathcal E\cap\operatorname{range}(P)\setminus\{0\}}
\frac{v^\top Gv}{\|v\|_2^2},
\qquad
g_Q=\inf_{v\in\operatorname{span}\{Q\bDelta(x)\}\setminus\{0\}}
\frac{v^\top Gv}{\|v\|_2^2}>0,
\end{equation*}
with $g_P=0$ if $P\mathcal E=\{0\}$.
If $\E_x\|Q\bDelta(x)\|_2^2>0$, then
\begin{equation*}
e_G:=\frac{\E_x\Gnorm{P\bDelta(x)}^2}{\E_x\Gnorm{Q\bDelta(x)}^2}
\le\frac{g_P}{g_Q}
\frac{\E_x\|P\bDelta(x)\|_2^2}{\E_x\|Q\bDelta(x)\|_2^2}.
\end{equation*}
If the right-hand side is small, it supplies a small admissible $\varepsilon\ge e_G$ for the localization part of Assumption~\ref{as:offprin}. Its metric-orthogonality condition remains separate.
\end{proposition}
\begin{proof}
At each context,
$\Gnorm{P\bDelta(x)}^2\le g_P\|P\bDelta(x)\|_2^2$
and
$\Gnorm{Q\bDelta(x)}^2\ge g_Q\|Q\bDelta(x)\|_2^2$.
Averaging and dividing by the positive denominator proves the claim; the zero-principal-space case is immediate.
\end{proof}

\paragraph{Sensitivity and localization are distinct.}
If $J_xv=c(x)\mathbf{1}_V+e(x)$, softmax gauge invariance gives $v^\top G_xv=\|e(x)\|_{F_x}^2$.
A small $g_P$ requires an appropriate uniform sensitivity bound over the principal subspace, not merely this identity for one vector. The ratio $g_P/g_Q$ can otherwise be large, so a small Euclidean PC fraction need not imply small sensitivity-weighted intrusion.
Nor does localization of a learned mean imply localization of its target family. With $G=I$, $P=e_1e_1^\top$, and equiprobable targets $\bDelta_\pm=(\pm c,1)^\top$, the optimal constant vector is $(0,1)^\top$ and has zero principal energy, whereas target principal-to-complement energy is $c^2$ and recovery is $1/(1+c^2)$. Measurements of learned-vector energy therefore do not alone verify Assumption~\ref{as:offprin}.

\subsubsection{A weight-space account of activation leakage}
We relate a layer's weight displacement to its first-order activation change. This is distinct from both the fitted intervention and the pulled-back teacher target.

\paragraph{Weight-principal components.}
For $W\in\R^{d\times d'}$ with singular value decomposition $W=\sum_i s_i\bp_i\bq_i^\top$, define
\begin{equation*}
\PW:=\sum_{i\le k}\bp_i\bp_i^\top.
\end{equation*}
The split $\Delta W=\PW\Delta W+(I-\PW)\Delta W$ separates principal weight components from their complement.
The bound $|s_i(W+\Delta W)-s_i(W)|\le\|\Delta W\|_2$ controls absolute spectral displacement but does not determine activation-principal localization.

\begin{lemma}[Weight components and first-order activation leakage]
\label{lem:bridge}
Consider $\bh_\ell(x)=\phi(W\bc(x))$ with elementwise differentiable $\phi$, holding $\bc(x)$ fixed while perturbing $W$. Let $D(x)=\diag(\phi'(W\bc(x)))$ have bounded norm at the evaluated points. For the first-order shift $\Dh(x)=D(x)\Delta W\bc(x)$,
\begin{equation}
\|P\Dh(x)\|_2
\le\underbrace{\|D(x)\|_2\|\PW\Delta W\|_2\|\bc(x)\|_2}_{\text{on-principal weight mass}}
+\underbrace{\|PD(x)(I-\PW)\|_2\|\Delta W\|_2\|\bc(x)\|_2}_{\text{cross-subspace leakage}}.
\label{eq:bridge}
\end{equation}
If $\|\Dh(x)\|_2>0$ and the right-hand side is at most $\xi\|\Dh(x)\|_2$, then $\|P\Dh(x)\|_2/\|\Dh(x)\|_2\le\xi$.
\end{lemma}
\begin{proof}
Decompose
\begin{equation*}
P\Dh(x)=PD(x)\PW\Delta W\bc(x)+PD(x)(I-\PW)\Delta W\bc(x).
\end{equation*}
Apply the triangle inequality, submultiplicativity, and $\|P\|_2\le1$. The relative bound follows by division when the total shift is nonzero.
\end{proof}

The leakage factor separates as
$PD(x)(I-\PW)=(P-\PW)D(x)(I-\PW)+\PW D(x)(I-\PW)$,
exposing basis mismatch and Jacobian mixing. Small weight-update norm or zero principal weight mass alone does not make this factor small relative to the total activation change.
Extending the result beyond first order requires a nonlinear remainder bound; changing earlier layers additionally perturbs $\bc(x)$. The lemma is a conditional first-order decomposition, not a deduction of activation localization from small weight updates.

\subsection*{Part III.\quad Geometric Analysis of Alpha-Stabler}
% Attach the legacy Part III label to a numbered heading for valid \ref output.
\noindent\emph{Parts~I--II analyze frozen-model intervention recovery. We now study a distinct parameter-training procedure: a fixed-reference activation diagnostic and local backward gradient projection. The diagnostic's energy identity and the projection's local properties are exact; the link to parameter-induced activation changes is conditional, and overall warning and stabilization performance is empirical.}

\subsection{Principal-subspace monitoring and gradient control}
\label{sec:predicter}
\label{Geometric Analysis of Alpha-Stabler}

\paragraph{Warm-up estimation and fixed reference geometry.}
During the first $T_{\mathrm{warm}}=50$ parameter updates, the actor trains without projection. Its generated sequences are also evaluated by the frozen initial base model to estimate reference token-activation covariance. After warm-up, let $\mathcal U_\ell\in\mathbb R^{d\times r}$ contain the selected leading orthonormal eigenvectors, with $r=\lceil0.10d\rceil$.
This is the main-text basis $U_{\ell,r}$; the retained fraction refers to hidden dimensions, not explained variance. Set $P_\ell=\mathcal U_\ell\mathcal U_\ell^\top$ and $Q_\ell=I-P_\ell$ and hold them fixed thereafter.
The analysis conditions on the estimated projectors. It neither assumes they equal population projectors nor supplies a finite-sample PCA guarantee. During each controlled backward pass and optimizer step, both projectors and flags are fixed, and their estimation and selection are detached from differentiation.

\paragraph{Token-level principal-subspace intrusion.}
For identical actor/base token sequences, stack the analyzed valid token-level activation differences as rows of $X_\ell^{(t)}\in\mathbb R^{N\times d}$. Consistent with Equation~(\ref{eq6}),
\begin{equation}
\mathrm{PSI}_\ell^{(t)}
=\frac{\|X_\ell^{(t)}\mathcal U_\ell\|_F^2}
{\|X_\ell^{(t)}\|_F^2+\epsilon_{\mathrm{num}}},
\qquad \epsilon_{\mathrm{num}}>0.
\label{eq:psi}
\end{equation}
The stabilizer is the main-text numerical $\epsilon$, not the theoretical intrusion bound $\varepsilon$. Energies are summed over tokens before normalization; averaging vectors first would permit cancellation. Aggregate PSI sums principal and total shift energies across monitored layers before taking their ratio.
The identities below hold for the analyzed shift collection. If monitoring uses a token subset, they are exact for that subset, not automatically for all tokens or the population. The sampling and distributed aggregation protocol must be specified with the implementation.

\begin{corollary}[Energy interpretation and metric comparison]
\label{cor:psi}
Fix a layer. Let $A=\|XP\|_F^2$, $B=\|XQ\|_F^2$, and $A+B>0$, and set $\psi_0=A/(A+B)$.
For $B>0$, define $e_E=A/B$. Then
\begin{equation}
\psi_0=\frac{e_E}{1+e_E},
\qquad
0\le\psi_0-\mathrm{PSI}
=\psi_0\frac{\epsilon_{\mathrm{num}}}{A+B+\epsilon_{\mathrm{num}}}.
\label{eq:psi_energy_identity}
\end{equation}
Write each row of $X$ as $\boldsymbol\Delta_i^\top$. For the same shifts with $B>0$, define
$e_G=\sum_i\|P\boldsymbol\Delta_i\|_G^2/\sum_i\|Q\boldsymbol\Delta_i\|_G^2$.
Suppose $0<m_G\le M_G$ bound the metric on the span of the projected shifts:
$m_G\|v\|_2^2\le\|v\|_G^2\le M_G\|v\|_2^2$.
With $\chi=M_G/m_G$,
\begin{equation}
\frac{e_G}{\chi}\le e_E\le\chi e_G.
\label{eq:psi_metric_comparison}
\end{equation}
If this same shift collection satisfies $e_G\le\varepsilon$, then
$\mathrm{PSI}\le\psi_0\le\chi\varepsilon/(1+\chi\varepsilon)$.
\end{corollary}
\begin{proof}
Orthogonality gives $\|X\|_F^2=A+B$ and $\|X\mathcal U\|_F^2=A$, proving the energy identities by substitution.
The metric bounds give
$m_GA\le\sum_i\|P\boldsymbol\Delta_i\|_G^2\le M_GA$
and the corresponding bounds for $B$. Division proves Equation~(\ref{eq:psi_metric_comparison}); the last claim follows from monotonicity of $u/(1+u)$.
If $B=0$ and $A>0$, $\psi_0=1$ and the finite-ratio comparison is inapplicable. If $A+B=0$, the stabilized statistic is zero but contains no directional information.
\end{proof}

The shifts in this corollary are observed actor-to-base differences. To apply the bound from Parts~I--II requires a justified relationship to $\bDelta(x)$, matched distributions, and appropriate metric bounds; the notation does not assert that relationship. A small PSI is a small directional energy fraction and need not imply small absolute activation displacement.

\paragraph{Predictor and threshold calibration.}
During an initially stable warm-up, retain detached actor-to-base shift matrices at checks every $M=3$ updates, subject to $N>0$ and
$\|X_\ell^{(t)}\|_F^2/N>\eta_{\min}$. The latter is a floor on the average squared token-level shift norm, not the energy of an averaged shift.
After the reference bases are frozen, evaluate the retained matrices using Equation~(\ref{eq:psi}); let $\mathcal W_\ell$ be the valid warm-up PSI values, as in Algorithm~\ref{alpha_stabler}. Set
\begin{equation}
\begin{aligned}
m_\ell&=\operatorname{median}_{\psi\in\mathcal W_\ell}\psi,&
s_\ell&=1.4826\operatorname{median}_{\psi\in\mathcal W_\ell}|\psi-m_\ell|+\epsilon_{\mathrm{num}},\\
\tau_\ell^{\mathrm{on}}&=m_\ell+3s_\ell,&
\tau_\ell^{\mathrm{safe}}&=m_\ell+2s_\ell.
\end{aligned}
\label{eq:appendix_detector_calibration}
\end{equation}
Calibration is invalid if any $\mathcal W_\ell$ is empty or
$0<\tau_\ell^{\mathrm{safe}}<\tau_\ell^{\mathrm{on}}<1$ fails; training stops for recalibration rather than using such thresholds. After calibration, bases and thresholds are fixed and retained records can be released.

Every $M$ updates, valid observations update
$\bar\psi_\ell\leftarrow0.95\bar\psi_\ell+0.05\mathrm{PSI}_\ell$, initialized at $m_\ell$.
The flag $a_\ell\in\{0,1\}$ starts at zero. It activates after $p=3$ consecutive valid checks with $\bar\psi_\ell>\tau_\ell^{\mathrm{on}}$ and releases when a valid check gives $\bar\psi_\ell<\tau_\ell^{\mathrm{safe}}$.
A valid check with $\bar\psi_\ell\le\tau_\ell^{\mathrm{on}}$, or an invalid check, resets the consecutive-warning count; an invalid check leaves the moving average and flag unchanged. Flags are held between checks and fixed throughout the associated gradient accumulation, backward pass, and optimizer update.
These rules define a detector with smoothing, persistence, and hysteresis. The MAD multiplier does not provide a false-alarm probability for dependent training observations. An initially stable calibration window is an operational requirement, and useful lead times or error rates require empirical evaluation rather than following from Equation~(\ref{eq:appendix_detector_calibration}).

\paragraph{Controller: backward-only gradient projection.}
Let $G_\ell^{(t)}\in\mathbb R^{N\times d}$ be the incoming activation gradient at a monitored output before its local control operation. This is distinct from the sensitivity metric $G$ in Parts~I--II. When several sites are controlled, the incoming gradient already includes any downstream modifications. With a fixed basis and detached flag, pass upstream
\begin{equation}
\widetilde G_\ell^{(t)}
=G_\ell^{(t)}-a_\ell^{(t)}(G_\ell^{(t)}\mathcal U_\ell)\mathcal U_\ell^\top.
\label{eq:controller}
\end{equation}
This hook leaves the forward mapping unchanged at fixed model parameters. At an active site it applies $G_\ell^{(t)}Q_\ell$; at an inactive site it leaves the incoming gradient unchanged. Future forward activations can change through the resulting parameter updates.

\begin{proposition}[Minimal-change removal of principal gradients]
\label{prop:gradient_projection_optimality}
% Compatibility alias for the shortened main-text paragraph.
\label{prop:minimal-change}
For a Euclidean orthogonal projector $P$ and $Q=I-P$, the unique solution of
\begin{equation}
\underset{Z:\,ZP=0}{\operatorname{minimize}}\;\frac12\|Z-G_\ell\|_F^2
\label{eq:gradient_projection_problem}
\end{equation}
is $\widetilde G_\ell=G_\ell Q$, and
\begin{equation}
\widetilde G_\ell P=0,\qquad
\widetilde G_\ell Q=G_\ell Q,\qquad
\|\widetilde G_\ell\|_F^2=\|G_\ell\|_F^2-\|G_\ell P\|_F^2.
\label{eq:gradient_projection_properties}
\end{equation}
\end{proposition}
\begin{proof}
Feasibility gives $Z=ZQ$. Principal and complement components are Frobenius-orthogonal, hence
$\|Z-G_\ell\|_F^2=\|Z-G_\ell Q\|_F^2+\|G_\ell P\|_F^2$.
The unique minimizer is $Z=G_\ell Q$. The remaining identities follow from $P^2=P$, $Q^2=Q$, and $PQ=0$.
\end{proof}

The proposition concerns the incoming gradient at one hook. It preserves that gradient's complement and cannot increase its local norm. It does not imply a smaller parameter-gradient norm, preservation of the uncontrolled network's complete gradient, or an analogous projection of future activation changes. Those quantities also depend on network Jacobians, other control sites, and the optimizer.

\paragraph{From projected gradients to activation changes.}
We analyze one post-warm-up SGD step on a fixed batch. This provides a local relation between a projected activation gradient and the forward displacement it induces, rather than an optimizer-independent stability guarantee.

\begin{lemma}[Local principal leakage under projected SGD]
\label{lem:projected_sgd_leakage}
% Compatibility alias for the shortened main-text paragraph.
\label{lem:activation-leakage}
Fix a token batch, one active control site, and a parameter block $\vartheta$. Suppose the loss depends on this block only through the selected activations, all other parameters are held fixed, and $h(\vartheta)$ is twice continuously differentiable locally. Stack token activations as $h=\operatorname{vec}(H^\top)$, and define
$J=\partial h/\partial\vartheta$, $\Pi=I_N\otimes P$, $\mathcal Q=I-\Pi$, and $g=\nabla_h\mathcal L$.
For a plain SGD step with $\alpha\ge0$ whose update segment stays within a neighborhood of bounded second derivative,
\begin{equation}
\delta\vartheta=-\alpha J^\top\mathcal Qg,\qquad
\delta h=-\alpha K\mathcal Qg+R,\qquad K=JJ^\top,
\label{eq:projected_sgd_activation}
\end{equation}
where $\|R\|_2\le c_h\alpha^2\|J^\top\mathcal Qg\|_2^2$ for a local constant $c_h$. Consequently,
\begin{equation}
\|\Pi\delta h\|_2
\le\alpha\|\Pi K\mathcal Q\|_2\|\mathcal Qg\|_2
+c_h\alpha^2\|J^\top\mathcal Qg\|_2^2.
\label{eq:principal_leakage_bound}
\end{equation}
\end{lemma}
\begin{proof}
The single-path dependency gives the modified parameter gradient $J^\top\mathcal Qg$.
Taylor's theorem for the vector-valued map gives
$h(\vartheta+\delta\vartheta)-h(\vartheta)=J\delta\vartheta+R$,
with $\|R\|_2\le c_h\|\delta\vartheta\|_2^2$; one may take $c_h$ as half a uniform bilinear operator-norm bound on its second derivative along the segment.
Substitute the SGD step, apply $\Pi$, and use $\|\Pi\|_2\le1$ and submultiplicativity to obtain Equation~(\ref{eq:principal_leakage_bound}).
\end{proof}

\paragraph{What projection does and does not preserve.}
First-order principal displacement is controlled by $\Pi K\mathcal Q$, not by $\Pi\mathcal Q=0$ alone. Relative to an unprojected SGD step at the same state, batch, and gradient,
\begin{equation*}
\delta h_{\mathrm{proj}}-\delta h_{\mathrm{plain}}
=\alpha K\Pi g+(R_{\mathrm{proj}}-R_{\mathrm{plain}}).
\end{equation*}
Thus the first-order difference in complement displacement is $\alpha\mathcal QK\Pi g$, which need not vanish. If $\Pi K\mathcal Q=0$, symmetry of $K$ also gives $\mathcal QK\Pi=0$: the projected update then has no first-order principal displacement and preserves the plain step's complement displacement to first order. Without this additional decoupling, the lemma is a bound, not a guarantee of either preservation property or a comparison showing less leakage than the unprojected step.

\paragraph{Preconditioning and additional update terms.}
For a fixed realized parameter-space preconditioner $D$ and an additional direction $b$, write a particular update as
$\delta\vartheta=-\alpha(DJ^\top\mathcal Qg+b)$.
The same Taylor argument gives, provided its update segment satisfies the same local regularity condition,
\begin{equation*}
\begin{aligned}
\|\Pi\delta h\|_2
\le{}&\alpha\|\Pi JDJ^\top\mathcal Q\|_2\|\mathcal Qg\|_2
+\alpha\|\Pi Jb\|_2\\
&+c_h\alpha^2\|DJ^\top\mathcal Qg+b\|_2^2.
\end{aligned}
\end{equation*}
The first-order coupling is now $JDJ^\top$; $b$ may contain optimizer-state or other update contributions not represented by the current projected gradient. This is a conditional estimate for a specified step, not a bound on how an adaptive preconditioner or its state evolves. Applying it to a concrete optimizer requires its actual update decomposition. Multiple control sites likewise require the full modified gradient flow; the single-site SGD lemma does not supply a joint guarantee for the complete training algorithm.

\paragraph{Local leakage versus cumulative PSI and stability.}
The lemma bounds one update on a fixed batch, whereas PSI measures actor-to-base differences accumulated over training and is a ratio of energies. Even zero new principal displacement need not decrease PSI: for a one-token difference $(1,1)$ with the first coordinate principal, a complement-only change $(0,-\alpha)$, $0<\alpha<1$, raises the unstabilized PSI from $1/2$ to $1/(1+(1-\alpha)^2)$.
Long-run guarantees would additionally require uniform control of couplings, update magnitudes, remainders, optimizer state, and changing observation distributions. We therefore treat warning utility and training stabilization as empirical outcomes in Figures~\ref{fig7}(c) and \ref{fig8}, not consequences of a monotone-PSI or no-collapse theorem.

\paragraph{Implementation and overhead.}
The mathematical properties require aligned actor/base observations and flags fixed before the controlled backward pass; they do not require a particular graph-retention strategy. Reusing a compatible actor forward and using a separate detached monitoring pass have different runtime and memory costs. Neither the projection identity nor the absence of additional trainable parameters bounds end-to-end overhead. Such claims require measurements that include reference evaluation, any additional actor pass, covariance estimation, calibration storage, and communication.

\subsection{Summary of theoretical and empirical results}
\label{sec:summary}
Table~\ref{tab:theory_summary} records the object and scope of each conclusion. Parts~I--II characterize conditional target-energy recovery and alternative intervention classes. Part~III establishes finite-collection energy identities, exact local gradient-projection properties, and a conditional single-step relationship to activation displacement.

\begin{table}[t]
\caption{Summary of recovery, steering geometry, and gradient-control results, with their assumptions and empirical scope.}
\label{tab:theory_summary}
\centering\small
\renewcommand{\arraystretch}{1.18}
\begin{tabularx}{\textwidth}{@{}p{0.38\textwidth}p{0.26\textwidth}X@{}}
\toprule
\textbf{Statement} & \textbf{Reference} & \textbf{Scope / status}\\
\midrule
\multicolumn{3}{@{}l}{\emph{Part I --- fixed-vector recovery}}\\
Distillation admits a weighted least-squares approximation & Proposition~\ref{prop:wls} & Fixed contexts; controlled network and KL remainders\\
The regularized optimum is an exponential reward tilt & Lemma~\ref{lem:tilt} & Exact optimum of the stated objective, $\beta>0$\\
Next-token ratios follow continuation normalizers & Lemma~\ref{lem:next} & Exact under the tilt model\\
Shared profiles and consistent pullbacks give shared shifts & Lemmas~\ref{lem:profile}, \ref{lem:pullback} & Additional sufficient structural conditions\\
$\rho=\mathrm{SNR}/(1+\mathrm{SNR})$ & Theorem~\ref{thm:snr} & Constant-metric surrogate target energy\\
Leading spectral energy share upper-bounds $\rho$ & Corollary~\ref{cor:spectral} & Upper bound, not a general equality\\
$\rho\ge\kappa/[(1+\eta)(1+\varepsilon)]$ & Theorem~\ref{thm:bound} & Assumed localization, moments, and metric compatibility\\
\midrule
\multicolumn{3}{@{}l}{\emph{Part II --- geometry of steering interventions}}\\
Constrained constant-vector fitting is a metric projection & Theorem~\ref{thm:svd} & Fixed target family and metric; exact optimization\\
Adaptive subspace recovery is a spectral partial sum & Equation~(\ref{eq:adaptive_subspace_recovery}) & Unrestricted context-dependent coefficients\\
Later orthogonal directions can remain useful & Section~\ref{sec:multi} & Empirical standalone interventions; not a rank estimate\\
Euclidean and sensitivity-weighted intrusion are related & Proposition~\ref{prop:mismatch} & Same targets; explicit metric constants\\
Weight components bound first-order activation leakage & Lemma~\ref{lem:bridge} & Fixed input; basis-mismatch and mixing factors\\
\midrule
\multicolumn{3}{@{}l}{\emph{Part III --- monitoring and gradient control}}\\
PSI measures Euclidean principal energy of analyzed shifts & Corollary~\ref{cor:psi} & Exact finite-collection identity; stabilizer accounted for\\
Projection minimally removes incoming principal gradients & Proposition~\ref{prop:gradient_projection_optimality} & Exact at each active hook with fixed projector\\
Principal activation displacement obeys a coupling bound & Lemma~\ref{lem:projected_sgd_leakage} & Single-site, fixed-batch, local plain SGD\\
Warm-up specifies a fixed reference and thresholds & Algorithm~\ref{alpha_stabler} & Calibration procedure; no detection-error guarantee\\
PSI warns and control improves training stability & Figures~\ref{fig7}(c) and \ref{fig8} & Empirical findings in the evaluated settings\\
\bottomrule
\end{tabularx}
\end{table}

\paragraph{Summary.}
The analysis links three distinct questions while keeping their mathematical objects distinct. First, shared teacher targets can permit strong constant-vector recovery under local approximation, moment, and metric conditions. Second, constrained and adaptive fitting describe different intervention capacities; neither standalone orthogonality nor low learned-vector PC energy alone determines the teacher target's rank or localization. Third, Alpha-Stabler uses a fixed activation reference for monitoring and local gradient control. Its projection has exact minimal-change properties, while its effect on parameter-induced activation displacement depends on coupling and the optimizer. These conditional results organize the empirical capacity and localization observations and motivate the controller; they do not replace empirical evidence for task-score recovery, early warning, stability, or runtime overhead.

\end{document}